\documentclass[twoside]{article}

\usepackage[]{amsmath,amssymb,epsfig}
\usepackage{amsthm}
\usepackage{amsmath}
\usepackage{amssymb}
\usepackage{graphicx}
\usepackage{epstopdf}
\usepackage{comment}
\usepackage{array}
\usepackage{algorithm}
\usepackage{url}

\usepackage{lscape}
\usepackage{algpseudocode}
\usepackage{setspace}
\usepackage{multicol}
\usepackage{multirow}
 \usepackage{color}
\usepackage{colortbl}
\usepackage{xcolor}

\usepackage[%
    font={small,sf},
    labelfont=bf,
    format=hang,    
    format=plain,
    margin=0pt,
    width=0.99\textwidth,
]{caption}
\usepackage[list=true]{subcaption}

\newtheorem{theorem}{Theorem}

\newtheorem{lemma}{Lemma}

\newtheorem{remark}{Remark}

\numberwithin{equation}{section}
\newcommand{\calR}{\ensuremath{\mathcal{R}}}

\newcommand{\tri}{\ensuremath{\bigtriangleup}}
\newcommand{\norm}[1]{\parallel{#1}\parallel}
\newcommand{\abs}[1]{|{#1}|}

\newcommand{\set}[1]{\left\{{#1}\right\}}

\newcommand{\expec}{\ensuremath{\mathbb{E}}}
\newcommand{\matR}{\ensuremath{\mathbb{R}}}

\newcommand{\prob}{\ensuremath{\mathbb{P}}}

\usepackage{rotating}

\usepackage{caption}
\usepackage{float}
\usepackage{url}
\usepackage{subcaption}
\usepackage{caption}
 \usepackage[accepted]{aistats2019}

\newcommand{\noteMC}[1]{\textbf{\textcolor{red}{{\small{[Mihai: }}#1]}}}

\def\norm#1{\left\|#1\right\|}

\renewcommand{\bar}{\overline} %

\DeclareMathOperator{\Tr}{Tr}%
\DeclareMathOperator{\vol}{vol} %

\DeclareMathOperator{\Cut}{cut}

\usepackage{xcolor}% http://ctan.org/pkg/xcolor
\usepackage{colortbl}% http://ctan.org/pkg/colortbl
\usepackage[utf8]{inputenc}
\usepackage{amsfonts}
\newcommand{\Ln}{L^{-}}
\newcommand{\Lp}{L^{+}}
\newcommand{\Dn}{D^{-}}
\newcommand{\Dp}{D^{+}}
 
\newcommand{\Tbar }{ \bar{T} }
 
\newcommand{\ELn }{   \expec{[\Ln]}  }
\newcommand{\ELp }{   \expec{[\Lp]}  }

\newcommand{\EDp }{   \expec{[\Dp]}  }
\newcommand{\EDn }{   \expec{[\Dn]}  }

\newcommand{\An}{A^{-}}
\newcommand{\Ap}{A^{+}}
\newcommand{\EAp }{   \expec{[\Ap]}  }
\newcommand{\EAn }{   \expec{[\An]}  }

\usepackage{multirow,bigdelim}

\usepackage{tikz}
\usetikzlibrary{matrix,decorations.pathreplacing,calc}

\pgfkeys{tikz/mymatrixenv/.style={decoration=brace,every left delimiter/.style={xshift=3pt},every right delimiter/.style={xshift=-3pt}}}
\pgfkeys{tikz/mymatrix/.style={matrix of math nodes,left delimiter=[,right delimiter={]},inner sep=2pt,column sep=1em,row sep=0.5em,nodes={inner sep=0pt}}}
\pgfkeys{tikz/mymatrixbrace/.style={decorate,thick}}
\newcommand\mymatrixbraceoffseth{0.5em}
\newcommand\mymatrixbraceoffsetv{0.2em}

\newcommand*\mymatrixbraceright[4][m]{
    \draw[mymatrixbrace] ($(#1.north west)!(#1-#3-1.south west)!(#1.south west)-(\mymatrixbraceoffseth,0)$)
        -- node[left=2pt] {#4} 
        ($(#1.north west)!(#1-#2-1.north west)!(#1.south west)-(\mymatrixbraceoffseth,0)$);
}

\newcommand*\mymatrixbracetop[4][m]{
    \draw[mymatrixbrace] ($(#1.north west)!(#1-1-#2.north west)!(#1.north east)+(0,\mymatrixbraceoffsetv)$)
        -- node[above=2pt] {#4} 
        ($(#1.north west)!(#1-1-#3.north east)!(#1.north east)+(0,\mymatrixbraceoffsetv)$);
}

\newcommand{\tce}{   \tilde{c}_{\varepsilon}  }

\newcommand{\EA }{   \expec{[A]}  }
\newcommand{\EDbar }{   \expec{[\Dbar]}  }
 \newcommand{\Dbar }{ \bar{D} }

\newcommand{\Lbar }{ \bar{L} }
\newcommand{\ELbar }{   \expec{[\Lbar]}  }

\newcommand{\mb}[1]{\mbox{\boldmath$#1$}}
\newcommand{\ones }{ \mb{1} }

\newcommand{\infovec }{w}
\newcommand{\pert }{\tri}

\newcommand{\Gn }{ G^{-} }
\newcommand{\Gp }{ G^{+} }

\newcommand{\taup }{ \tau^{+} }
\newcommand{\taun }{ \tau^{-} }

\usepackage[font={small}]{caption}

\usepackage[skip=-5pt, font={small}]{subcaption}
\usepackage[T1]{fontenc}
\usepackage{eurosym}
\def\yen{{\setbox0=\hbox{Y}Y\kern-.97\wd0\vbox{\hrule height.1ex
width.98\wd0\kern.33ex\hrule height.1ex width.98\wd0\kern.45ex}}}

\usepackage{hyperref}
\hypersetup{colorlinks=true, citecolor=red,linkcolor=blue,urlcolor=blue}

\usepackage{wrapfig}

\usepackage{enumitem}

\begin{document}

\twocolumn[

% \aistatstitle{SPONGE: A Laplacian generalized eigenproblem formulation for clustering signed networks}
% \aistatstitle{SPONGE: A Laplacian generalized eigenproblem for signed clustering}
% 
\aistatstitle{SPONGE: A generalized eigenproblem for clustering signed networks}
% graphs

% SPONGE: A Laplacian generalized eigenproblem for signed clustering \\ 
% A volume regularized cut ratio for clustering signed networks

 \aistatsauthor{Mihai Cucuringu \And Peter Davies \And  Aldo Glielmo \And Hemant Tyagi}
 \aistatsaddress{University of Oxford \\ The Alan Turing Institute \\ mihai.cucuringu@stats.ox.ac.uk \And  University of Warwick \\ P.W.Davies@warwick.ac.uk  \And  King's College London \\aldo.glielmo@kcl.ac.uk \And INRIA Lille - Nord Europe\\hemant.tyagi@inria.fr} 
]

\begin{abstract}
We introduce a principled and theoretically sound spectral method for 
% the general 
$k$-way clustering % problem 
in signed graphs, where the affinity measure between  %  the nodes % can take 
nodes 
takes either positive or negative values. Our approach is motivated by social balance theory, where the task of clustering aims to decompose the network into disjoint groups, such that individuals within the same group are connected by as many positive edges as possible, while individuals from different groups are connected by as many negative edges as possible. Our algorithm relies on a generalized eigenproblem formulation inspired by recent work on 
 % the constrained clustering problem. 
constrained clustering. 
We provide theoretical guarantees for our approach in the setting of a signed stochastic block model, by leveraging tools from matrix perturbation theory and random matrix theory. An extensive set of numerical experiments on both synthetic and real data shows that our approach compares favorably with state-of-the-art methods for signed clustering, especially % when the number of clusters is large.
                        for large number of clusters and sparse measurement graphs. 
% and the measurements graphs are sparse.
%
%
% We study the task of clustering signed networks, in which edge weights can take negative, as well as positive, values. The problem arose from the area of social balance theory, and has recently seen attention due to its applications to social network and time series data. We present a spectral method for clustering signed graphs based on a generalized eigenvalue problem formulation. Under a signed extension of the stochastic block model we show theoretical guarantees on the algorithm's performance. We then demonstrate experimentally improvement of previous approaches on examples of synthetic and real data.
\end{abstract}

% \noteAG{Exampled comment Aldo}
% \notePD{Exampled comment Peter}
% \noteHT{Exampled comment Hemant}
% \noteMC{Exampled comment Mihai}
% \notePM{Exampled comment Pedro}

%%% \section{Introduction}

\vspace{-7mm}
\section{Introduction}
\vspace{-2mm}

Clustering is a %  widely used 
       popular
unsupervised learning  
            task  % technique % that considers the task of 
aimed at 
extracting groups of nodes in a weighted graph in such a way that the average connectivity or similarity between pairs of nodes within the same group is larger than that of pairs of nodes from different groups.  While most of the literature has focused on clustering graphs where the edge weights  % take non-negative values, 
are  non-negative, 
the task of clustering signed graphs (whose edge weights can take negative values as well)
% (that contain edges whose weights take both positive and negative values) 
%
remained relatively unexplored,  % The analysis of signed networks
and has recently become an increasingly important research topic \cite{Leskovec_2010_SNS}. % This interest was fueled by 

The motivation for recent studies arose from a variety of examples from social networks, where users  express relationships of trust-distrust or friendship-enmity, online news and review websites such as Epinions \cite{epinions}  and Slashdot \cite{slashdot} that allow users to approve or denounce others \cite{Leskovec_2010_PPN}, and shopping bipartite networks  encoding 
% like and dislike
like-dislike 
%relationships 
preferences between users and products \cite{banerjee2012partitioning}.

\vspace{-1mm}
Another application  % of this line of work 
stems from time series analysis, in particular clustering time series  % data
\cite{aghabozorgi2015time}, a task broadly used for analyzing gene expression data in biology \cite{fujita2012functional}, economic time series that capture macroeconomic variables \cite{focardi2001clustering}, and financial time series corresponding to large baskets of instruments in the stock market \cite{ziegler2010visual, pavlidis2006financial}.
% involving a large number of financial instruments, massive amount of financial time series data that originates from the stock market 
% 
In such contexts, a popular similarity measure     % approach 
in the literature is given by the Pearson correlation coefficient % (and related distances),  
that measures % the 
linear dependence between variables and takes values in $[-1,1]$. By interpreting the correlation matrix as a weighted network whose (signed) edge weights capture the pairwise correlations, we cluster the multivariate time series by clustering the underlying signed network. To increase robustness, tests of statistical significance are often applied to individual pairwise correlations,  leading to sparse networks after thresholding on the p-value associated to each  individual sample correlation \cite{Ha2015_Network_Threshold_pValue}. 
% tests of statistical significance  applied to individual correlations - indicating the probability of obtaining a correlation as large as (or larger than) the sample correlation, assuming the null hypothesis is true.
% 
We refer the reader to the popular work of Smith et al. \cite{Smith_2011_FMRI_networks_TimeSeries}  for a detailed survey and comparison of various  % methodologies
        methods
for turning time series % data 
into networks. %, where the authors explore the interplay between fMRI time series and the network generation process. %  and most 
Importantly, they conclude that in general correlation-based approaches can be quite successful at estimating the connectivity of brain networks from fMRI time series. %  data.

\vspace{-3mm}
% \paragraph{Our contribution is as follows. \noteMC{todo} }
\paragraph{Contributions.} Our contributions are as follows.  
\vspace{-1mm}

% \hspace{-2mm}
\noindent  $\bullet$ We propose a regularized spectral algorithm for clustering signed graphs that is based on solving a generalized eigenproblem. Our approach is scalable and compares favorably to state-of-the-art methods.

% \hspace{-2mm}
\vspace{-1mm} 
\noindent  $\bullet$ We provide a detailed theoretical analysis of our algorithm 
%  in terms of its
   with respect to 
its robustness against sampling sparsity 
% of the measurement graph 
and noise level, under a Signed Stochastic Block Model (SSBM). 

% \hspace{-2mm}
\vspace{-1mm} 
\noindent   $\bullet$ %Third, we provide a similar robustness analysis of signed clustering 
% algorithm 
To our knowledge, we provide the first theoretical guarantees -- in the SSBM framework -- for the
Signed Laplacian method introduced in the popular work of Kunegis et al. \cite{kunegis2010spectral} for clustering signed graphs.

% \hspace{-2mm}
\vspace{-1mm} 
\noindent  $\bullet$ Finally, we provide extensive numerical experiments on both synthetic and real data, showing that our % proposed
algorithm compares favourably to 
% state-of-the-art methods. 
state-of-art methods. In particular, it is able to recover 
partitions  in the regime
% clusters 
% in regimes 
where the graph is sparse and the number of clusters $k$ is large, where existing methods completely fail.

% In this paper, we present a novel spectral clustering method for signed graphs. Contrary to the most recent contributions on spectral clustering for signed graphs~\cite{Chiang:2012:Scalable, Kunegis:2010:spectral, Mercado:2016:Geometric} where either the addition or geometric mean of positive and negative interactions are proposed, we consider a discrete optimization problem based on the \textit{ratio} of graph cut criteria, which essentially aims at identifying clusters where most of positive edges are inside clusters, and most of negative edges are between clusters. We consider a continuous relaxation of this problem and present a theoretical  analysis under the stochastic block model for signed graphs, showing consistency in expectation and presenting the first probabilistic concentration bounds of its kind for signed networks. Further, we show that the continuous relaxation leads to a scalable generalized eigenvalue problem, and augment our findings with extensive numerical experiments on both real and synthetic data, that confirm its numerical superiority to  state-of-the-art methods from the literature. 

% \vspace{3mm}

\vspace{-3mm}
%------------------------
% Outline of the paper
%-----------------------
% \paragraph{Outline of the paper.}
\paragraph{Paper outline.} 
The remainder of this paper is organized as follows.  
Section \ref{sec:relatedWork} is a summary of related work from the signed clustering literature. 
Section \ref{sec:formulation_SPONGE} formulates our \textsc{SPONGE} (Signed Positive Over Negative Generalized Eigenproblem) algorithm for clustering signed graphs. 
% problem that leads to a
% mention \textsc{SPONGE} here??  todo
%
%Section \ref{sec:SSBM_model} introduces the Signed Stochastic Block Model (SSBM).
% 
%Section \ref{sec:SPONGE_Theory} is a theoretical analysis of \textsc{SPONGE} under % a perturbed  
%the SSBM model.
%
Section \ref{sec:SPONGE_Theory} introduces the Signed Stochastic Block Model (SSBM) and 
contains our theoretical analysis of \textsc{SPONGE} in the SSBM.
Section \ref{sec:Signed_Laplacian_theory} contains a similar theoretical analysis for signed spectral clustering via the Signed Laplacian.
Section \ref{sec:numericalExperiments}  contains numerical experiments on various synthetic and real data sets.
Finally, Section \ref{sec:conclusion} summarizes our findings along with future research directions.

\vspace{-3mm}
%------------
% Notation
%-----------
\paragraph{Notation.}
% \subsection{Notation} 
For a matrix $A \in \matR^{n \times n}$, we 
% will 
denote its eigenvalues and eigenvectors by 
$\lambda_i(A)$ and $v_i(A)$ respectively, $ \forall i = 1,\dots,n$. For symmetric $A$, we assume the ordering $\lambda_1(A) \geq \cdots \geq \lambda_n(A)$. For $A \in \matR^{m \times n}$, 
$\norm{A}_2$ denotes its spectral norm, i.e., the largest singular value of $A$. We denote $ \ones $ to be the all one's column vector. For a matrix $U$,
$\calR(U)$ denotes the range space of its columns.
Throughout, $G=(V,E)$ denotes the signed graph with vertex set $V$, edge set $E$, and adjacency matrix $A \in \{0,\pm 1 \}^{n \times n} $. 
We let $\Gp = (V, E^{+})$ (resp. $\Gn=(V, E^{-})$) 
%
% ($ E^{+}  \cap  E^{-} \hspace{-2mm} =\emptyset$),
%
denote the unsigned subgraphs of positive (resp. negative) edges with  % underlying
adjacency matrices $\Ap$ (resp. $\An$), such that $A = \Ap \hspace{-1mm} - \An$. 
Here, $A_{ij}^+ = \max\set{A_{ij},0}$ and $A_{ij}^- = \max\set{-A_{ij},0}$. Moreover 
$ E^{+}  \hspace{0mm}  \cap  E^{-} \hspace{-1mm} =\emptyset$, 
and 
$ E^{+}  \hspace{0mm}  \cup  E^{-} \hspace{-1mm} = E$.

\vspace{-2mm}
%%% \section{Related literature}

\vspace{-2mm}
\section{Related literature}  \label{sec:relatedWork}
% Problem setup: RelatedLit_input

% Kim:2018:SRL:3178876.3186117 ---> Kim_2018_SIDE_WWW

\vspace{-2mm}
The problem of clustering signed graphs can be traced back to the work of Cartwright and Harary from the 1950s on social balance theory \cite{HararySigned,cartwright1956structural}, who explored the concept of \textit{balance}
in signed graphs. A signed graph is said to be balanced iff (i) all the edges are positive, or (ii)  the nodes can be partitioned into two disjoint sets such that positive edges exist only within clusters, and negative edges are only present across clusters.
The ``weak balance theory'' of Davis \cite{DavisWeakBalance} relaxed the balanced relationship -- a signed graph is weakly balanced iff (i) all the edges are positive, or (ii)  the nodes can be partitioned into $k \in \mathbb{N}$ disjoint sets such that positive edges exist only within clusters, and negative edges are only present across clusters.

Motivated by this theory, the $k$-way clustering  problem in signed graphs amounts to finding a partition into $k$ clusters such that most edges within clusters are positive, and most edges across clusters are negative. 
\begin{wrapfigure}{l}{0.5\columnwidth}
\vspace{-4mm}
\centering
\includegraphics[width=0.55\columnwidth]{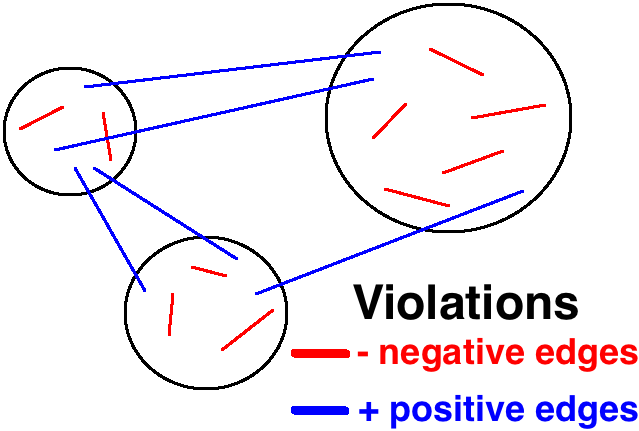} 
% \subcaption{  $\eta=0$ }
\vspace{-10mm}
% \label{fig:PolarHists}
\end{wrapfigure}
Alternatively, one may seek a partition such that the number of \textbf{violations} is minimized, i.e., negative edges within the cluster and positive edges across clusters, as depicted in the figure above.
In order to avoid partitions where clusters contain only a few nodes, one often wishes to  also 
incentivize clusters of large size or volume. 

A number of algorithms have been proposed for clustering signed graphs. Doieran and Mrvar \cite{DoreianMrvar20091} proposed a local search approach in the spirit of the  Kernighan-Lin algorithm \cite{KernighanLin}. Yang et al. \cite{YangCheungLiu} introduced an agent-based approach by  considering a certain random walk on the graph. 
In recent years, several efforts for the analysis of signed graphs have lead to novel extensions for  various tasks, including 
edge prediction~\cite{kumar2016wsn,Leskovec_2010_PPN},
node classification~\cite{Bosch_2018_nodeclassification,tang2016nodeClassification}, 
node embeddings~\cite{Chiang_2011_ELC,derr_2018_signedGraphConvolutionalNetwork,Kim_2018_SIDE_WWW, Wang_2017_signedNetworkEmbedding}, 
node ranking~\cite{Chung_2013,Shahriari_2014}, 
and clustering~\cite{Chiang_2012_Scalable, kunegis2010spectral, Mercado_2016_Geometric}. We refer the reader to~\cite{tang2015survey} for a recent survey on the topic.

\vspace{-1mm}
\textbf{Spectral methods} on signed networks began with Anchuri et al. \cite{anchuri2012communities}; %  who proposed an approach to optimize
they proposed optimizing modularity and other objective functions in signed 
graphs. Kunegis et al. \cite{kunegis2010spectral} proposed 
%new  spectral tools for clustering, link prediction, and visualization of signed graphs, by 
solving a 2-way ``signed'' ratio-cut problem via the (combinatorial) Signed Laplacian \cite{HouSignedLap} $\bar{L} = \bar{D} - A$, 
where $\bar{D}$ is a diagonal matrix with $\bar{D}_{ii} = \sum_{i=1}^{n} |A_{ij}|$. 
%The authors proposed a unified approach to handle negative weights, motivated by a number of tools and applications  such as  random walks, link predicton, graph clustering and 
% graphs 
%drawing, and electrical networks.
% 
Similar signed extensions also exist for the random-walk Laplacian $\bar{L}_{\text{rw}} = I - \bar{D}^{-1} A$, and the symmetric graph Laplacian $\bar{L}_{\text{sym}} = I - \bar{D}^{-1/2} A \bar{D}^{-1/2}$, 
the latter of which is particularly suitable for skewed degree distributions.
Chiang et al. \cite{DhillonBalNormCut} put forth the claim that the Signed Laplacian $ \Lbar$  faces a fundamental weakness when directly extended to $k$-way clustering for $k>2$. 
% suffers a theoretical flaw when $k>2$, 
%
%%%%%%%%------------------------------------------------------------------------------
%%%%%%%%-----------------------------------------  Leave OUT BRC - and reformualte 
\iffalse  
They proposed a formulation based on the  \textit{Balanced Ratio Cut} (BRC) objective

\vspace{-4mm}
\begin{equation}
\operatorname{min}_{ \{x_1,\ldots,x_k\} \in I } \left(  \sum_{c=1}^{k} \frac{x_c^T(D^{+} - A)x_c }{x_c^T x_c}  \right),
\label{obj_BRatioC}   
\end{equation}
\vspace{-4mm}

and the closely related \textit{Balanced Normalized Cut} (BNC) objective, normalized by volume

\vspace{-4mm}
\begin{equation}
\operatorname{min}_{ \{x_1,\ldots,x_k\} \in I } \left(  \sum_{c=1}^{k} \frac{x_c^T(D^{+} - A)x_c }{x_c^T \bar{D} x_c}  \right).
\label{obj_BNormC}  
\end{equation}
\vspace{-4mm}

%
Here, $\Dp $ denotes the diagonal matrix with degrees  $\Dp_{ii} = \sum_{j=1}^{n} A^{+}_{ij}$
on the diagonal, and 
$\{x_1,\ldots,x_k\} \in I$ denotes a $k$-cluster indicator set, 
%  Each of the $k$ vectors is 
% defined as

\vspace{-5mm}
\begin{equation}
x_t(i) = \left\{
 \begin{array}{rl}
 1 & \; \text{if node } i \in C_t\\
 0 & \; \text{otherwise } 	\\
     \end{array}
   \right.
\label{def:indicatorFcn}
\end{equation}
\vspace{-5mm}

with $C_1, \ldots, C_k$ denoting the $k$ clusters. 
\fi 
%%%%%%%%------------------------------------------------------------------------------
%%%%%%%%------------------------------------------------------------------------------
%
%
They proposed a formulation based on the   \textit{Balanced Normalized Cut} (BNC) objective

\vspace{-4mm}
\begin{equation*}
\operatorname{min}_{ \{x_1,\ldots,x_k\} \in \mathcal{I} } \left(  \sum_{c=1}^{k} \frac{x_c^T(D^{+} - A)x_c }{x_c^T \bar{D} x_c}  \right).
%\label{obj_BNormC}  
\end{equation*}
\vspace{-4mm}

Here, $\Dp $ denotes the diagonal matrix with degrees  $\Dp_{ii} = \sum_{j=1}^{n} A^{+}_{ij}$; $C_1, \ldots, C_k$ denote the $k$ clusters, and 
% on the diagonal.  $\{x_1,\ldots,x_k\} \in I$ 
$\mathcal{I}$ denotes a $k$-cluster indicator set, 
%  Each of the $k$ vectors is 
% defined as
% 
\iffalse
\vspace{-5mm}
\begin{equation}
x_t(i) = \left\{
 \begin{array}{rl}
 1 & \; \text{if node } i \in C_t\\
 0 & \; \text{otherwise } 	\\
     \end{array}
   \right.
\label{def:indicatorFcn}
\end{equation}
\vspace{-5mm}
\fi 
where $ (x_t)_i = 1, \text{if node } i \in C_t $, and $0$ otherwise. % , and  
% 
% with 
% $C_1, \ldots, C_k$ denoting the $k$ clusters.
The same authors also consider the closely related  \textit{Balanced Ratio Cut}, which replaces $\bar{D}$ in the denominator by $I$. 
We remark that spectral clustering algorithms (for signed/unsigned graphs) typically have a common pipeline, wherein a suitable graph operator is considered (for eg. Laplacian), its $k$ (or $k-1$) extremal eigenvectors are found, and the resulting points in $\matR^k$ (or $\matR^{k-1}$) are then clustered using $k$-means.

Hsieh et al. \cite{DhillonLowRank} propose performing matrix completion as a preprocessing step 
before clustering using the top $k$ eigenvectors of the completed matrix. Mercado et al. \cite{Mercado_2016_Geometric} present an extended spectral method based on the geometric mean of Laplacians.
For $k=2$, Cucuringu \cite{sync_congress} showed that signed clustering can be cast as an instance of the group synchronization \cite{sync} problem over $\mathbb{Z}_2$. 
%, for which spectral,  semidefinite programming relaxations, and message passing algorithms have been studied.
% 
Finally, we refer the reader to \cite{JeanGallierSurvey} for a recent survey on clustering 
% methods for 
signed and unsigned graphs. 

%%% Formulation - Optimization
%%% \section{SPONGE: a generalized eigenproblem formulation}

% \section{Modeling}

\vspace{-2mm}
\section{ \textsc{SPONGE}: a signed generalized eigenproblem formulation}  \label{sec:formulation_SPONGE}
\vspace{-3mm}

% \section{SPONGE: Signed Positive - Over - Negative Generalized Eigenproblem}

% Description of the optimization problem considered, and its relaxation.

% G -- $\Gp$  _G  --->  _{\Gp}
%  H -- $\Gn$   _H  --->  _{\Gn}
% \textbf{Notation}: 
% Moved to the notation section --- > 
\iffalse
Let $\Gp$, respectively $\Gn$, denote the subgraph of positive, respectively negative, edges, with  % underlying
adjacency matrices $\Ap$, respectively $\An$.
% \noteMC{Decide on this - since in the theory part we use $A^+$ and $A^-$ + rewrite} 
\fi 
% 
Given an unsigned graph $H$ with adjacency matrix $W$ with non-negative entries, for any cluster $C\subset V$ % we
define 
%
%\vspace{-4mm}
%\begin{equation*}
  $\Cut_{H}(C,\overline{C}) := \sum_{i\in C, j \in \overline{C} } W_{ij}$
%\end{equation*}
%\vspace{-4mm}
%
as the total weight of  edges crossing from $C$ to $\bar{C}$. %  in $H$.
% We consider the following two measures of
Also define the volume of $C$, $\vol_H(C) := \sum_{i\in C}\sum_{j = 1}^n W_{ij}$ as 
the sum of degrees of nodes in $C$. Motivated by the approach of \cite{consClust} in the context of constrained clustering, we aim to minimize the following  two 
measures of \textit{badness}  %to % be minimized
%minimize

\vspace{-5mm}
\begin{equation}
\label{eq:Cut_{\Gp}}
 \frac{\Cut_{\Gp}(C,\overline{C})}{\vol_{\Gp}(C)},
\end{equation}
\vspace{-3mm}
\begin{equation}\label{eq:Cut_{\Gn}}
 \Big( \frac{\Cut_{\Gn}(C,\overline{C})}{\vol_{\Gn}(C)} \Big)^{-1} = \frac{\vol_{\Gn}(C)}{\Cut_{\Gn}(C,\overline{C})}. 
\end{equation}
\vspace{-4mm}

Ideally, $C$ % should be 
is 
such that both~\eqref{eq:Cut_{\Gp}} and~\eqref{eq:Cut_{\Gn}} are small. To this end, we first consider ``merging'' 
the objectives ~\eqref{eq:Cut_{\Gp}} and~\eqref{eq:Cut_{\Gn}}, and would like to solve

\vspace{-4mm}
\begin{equation*} %\label{eq:objective_function}
\min_{C\subset V}\,\, \frac{\Cut_{\Gp}(C,\overline{C}) +  \taun \vol_{\Gn}(C)}{ \Cut_{\Gn}(C,\overline{C}) +   \taup \vol_{\Gp}(C) }, 
\end{equation*}
\vspace{-4mm}

with $ \taup, \taun > 0$ denoting  trade-off or regularization parameters.   
While at first sight this may seem rather ad-hoc in nature, we provide a sound theoretical justification for our approach in later sections.
A natural extension to  % the case 
$k > 2$ disjoint clusters $C_1,\hdots,C_k$ 
% of the set of nodes $V$ 
% amounts to % considering   
leads to the following discrete optimization problem

\vspace{-9mm}
\begin{equation}
\label{eq:discrete_optimization_problem}
\min_{C_1,\hdots,C_k} \sum_{i=1}^k \frac{\Cut_{\Gp}(C_i,\overline{C_i}) + \taun  \vol_{\Gn}(C_i)}{\Cut_{\Gn}(C_i,\overline{C_i}) + \taup \vol_{\Gp}(C_i)}.
\end{equation}
\vspace{-6mm}

% Given a subset $C_i\subset V$ we can define a normalized indicator vector $x_{C_i}$
For a subset $C_i\subset V$, the normalized indicator vector % is % $x_{C_i}$

\vspace{-7mm}
\begin{equation*} 
\hspace{-3mm}
(x_{C_i})_j = 
 \begin{cases} 
(\Cut_{\Gn}(C_i,\overline{C_i}) + \taup \vol_{\Gp}(C_i))^{-1/2};& \hspace{-1mm} j \in C_i \\
0 ; & \hspace{-1mm} j \notin C_i
\end{cases}
\end{equation*}
\vspace{-6mm}

renders~\eqref{eq:discrete_optimization_problem} % to be rewritten 
as the  discrete optimization problem

\vspace{-5mm}
\begin{equation}
\label{eq:discrete_optimization_problem-2}
\min_{C_1,\hdots,C_k} \sum_{i=1}^k \frac{ x_{C_i}^T (\Lp +  \taun \Dn) x_{C_i} }{ x_{C_i}^T (\Ln +  \taup \Dp) x_{C_i} }, 
\end{equation}
\vspace{-4mm}

% \Lun_{\Gp} --->  \Lp 
% \Lun_{\Gn} --->  \Ln
% \D_{\Gn} --->   \Dn 
% \D_{\Gp} --->   \Dp 

which is % clearly 
NP-hard.  
Here $\Lp$ (resp. $\Ln$) denotes the Laplacian of $\Gp$ (resp. $\Gn$), and  
$\Dp$ (resp. $\Dn$) denotes a diagonal matrix with the degrees of $\Gp$ (resp. $\Gn$). 
A common approach in this situation is to drop the discreteness constraint and allow each $x_{C_i}$ to take values in $\mathbb{R}^n$. To this end, we introduce a new set of vectors $z_1,\ldots,z_k\in\mathbb{R}^n$, such that they are orthonormal with respect to $\Ln+ \taup \Dp$, i.e.,  

\vspace{-4mm}
\begin{itemize}
  \setlength\itemsep{0em}
 \item $z_i^T  (\Ln  +   \taup  \Dp) z_i =1$, and 
 \item $z_i^T(\Ln + \taup   \Dp)z_j= 0$, for $i \neq j$.
\end{itemize}
\vspace{-4mm}

% leading to the following modified version of
This leads to the following modified version of \eqref{eq:discrete_optimization_problem-2}  

\vspace{-8mm}
\begin{equation}\label{eq:discrete_optimization_problem-z}
 \min_{z_i^T(\Ln+\Dp)z_j=\delta_{ij}} \sum_{i=1}^k \frac{ z_i^T (\Lp + \taun  \Dn) z_i }{ z_i^T (\Ln + \taup  \Dp) z_i }.
\end{equation}
\vspace{-5mm}

% Note  % that 
The above choice of $(\Ln  +   \taup  \Dp) $-orthonormality of vectors $z_1,\ldots,z_k$ 
% with respect to $\Ln+\Dp$ 
is not -- strictly speaking -- a relaxation of \eqref{eq:discrete_optimization_problem-2}. 
But it leads to a suitable eigenvalue problem. Indeed, assuming $\Ln +  \taup  \Dp$ is full rank, consider the change of variables $y_i=(\Ln +  \taup  \Dp)^{1/2}z_i$ 
which changes the orthonormality constraints of~\eqref{eq:discrete_optimization_problem-2}
to  $y_i^T y_j = \delta_{ij}$. Furthermore, denoting $Y=[y_1,\ldots,y_k ]\in\mathbb{R}^{n\times k}$, one can rewrite~\eqref{eq:discrete_optimization_problem-z} as 
\iffalse
%%% JMLR single column formatting 
\vspace{-4mm}
\begin{equation}
\hspace{-5mm}
\footnotesize
 \min_{Y^T Y = I} \Tr \Big( Y^T (\Ln + \taup \Dp)^{-1/2} (\Lp + \taun \Dn) (\Ln + \taup \Dp)^{-1/2} Y \Big)
 \label{eq:discrete_optimization_problem-y}
\end{equation}
\vspace{-4mm}

\noteMC{Decide if to keep it as above (it already has footnotesize) - OR - break it into lines as:}

\fi

\vspace{-7mm}
% \hspace{-3mm}
\begin{align}
\label{eq:discrete_optimization_problem_y__}
\hspace{-3mm}
 \min_{Y^T Y = I} &  \Tr \Big( Y^T (\Ln + \taup \Dp)^{-1/2}  \\ \nonumber
              &  (\Lp + \taun \Dn) (\Ln + \taup \Dp)^{-1/2} Y \Big).    \nonumber
\end{align}
\vspace{-8mm}

The solution to~\eqref{eq:discrete_optimization_problem_y__} is given by the eigenvectors corresponding to the $k$-smallest eigenvalues of $(\Ln +  \taup \Dp)^{-1/2} (\Lp +  \taun \Dn) (\Ln +  \taup \Dp)^{-1/2}$ (see for eg. \cite[Theorem 2.1]{sameh2000}). One can also verify\footnote{\iffalse
\begin{lemma}\label{lemma:eigpair-generalizedeigpair}
 $(\lambda,v)$ is an eigenpair of $(\Ln + \taup \Dp)^{-1/2} (\Lp + \taun \Dn) (\Ln + \taup \Dp)^{-1/2}$ if and only if $(\lambda, (\Ln + \taup \Dp)^{-1/2}v)$ is an eigenpair of $(\Ln + \taup \Dp, \Lp + \taun \Dn)$.
\end{lemma}
\begin{proof}[Proof of Lemma~\ref{lemma:eigpair-generalizedeigpair}]
Let $A = \Ln + \taup \Dp$ and $B = \Lp + \taun \Dn$. Then by Lemma~\ref{lemma:aux-1} we can see that
 $(\lambda,v)$ is an eigenpair of $A^{-1/2}BA^{-1/2}$ if and only if 
 $(\lambda, A^{-1/2}v)$ is an eigenpair of $(B,A)$.
\end{proof}
\fi 
%
Let $A,B$ be symmetric matrices with $A \succ 0$.  
Then $(\lambda,v)$ is an eigenpair of $A^{-1/2}BA^{-1/2}$ iff $(\lambda, A^{-1/2}v)$ is a 
generalized eigenpair of $(B,A)$.
Indeed, for $w=A^{-1/2}v$, 
$ A^{-1/2}BA^{-1/2}v = \lambda v   
%\Leftrightarrow BA^{-1/2}v = \lambda A^{1/2}v  
\Leftrightarrow Bw = \lambda Aw.$ 
%     
%
%    
\iffalse
\begin{lemma}\label{lemma:aux-1}
 Let $A,B$ be symmetric positive definite matrices. 
 Then $(\lambda,v)$ is an eigenpair of $A^{-1/2}BA^{-1/2}$ if and only if 
 $(\lambda, A^{-1/2}v)$ is an eigenpair of $(B,A)$.
 \vspace{-3mm}
 \begin{proof}
    \begin{align*}
    A^{-1/2}BA^{-1/2}v = \lambda v  &  \iff BA^{-1/2}v = \lambda A^{1/2}v  \\
                     & \iff Bw = \lambda Aw 
    \end{align*}
    where $w=A^{-1/2}v$.
    \end{proof}
\end{lemma}
\fi 
}
that $(\lambda,v)$ is an eigenpair of the previous matrix if and only 
if $(\lambda, (\Ln +  \taup \Dp)^{-1/2}v)$ is a \emph{generalized} eigenpair of 
$( \Lp +  \taun \Dn, \Ln +  \taup \Dp)$. 

Our complete algorithm \textsc{SPONGE} first finds the smallest $k$ generalized eigenvectors of $( \Lp +  \taun \Dn, \Ln +  \taup \Dp)$ for suitably chosen $\taup,\taun > 0$. We then cluster the resulting embedding of the vertices in $\matR^k$ using $k$-means$++$. We also consider a variant of \textsc{SPONGE}, namely \textsc{SPONGE}$_{sym}$, where the embedding is generated using the smallest $k$ generalized eigenvectors of $( \Lp_{sym} +  \taun I, \Ln_{sym} +  \taup I)$, wherein $\Lp_{sym} = (\Dp)^{-1/2} \Lp (\Dp)^{-1/2}$ is the so-called symmetric Laplacian of $G^{+}$ (similarly for $\Ln_{sym}$). %\noteHT{Aldo + Peter: please check if this is how \textsc{SPONGE}$_{sym}$ is formulated.} \notePD{Yes, it is.}

\begin{remark}
Solving \eqref{eq:discrete_optimization_problem_y__} is computationally expensive in practice as it involves % first 
computing a matrix-inverse. This is not the case if we solve the generalized eigenproblem version of \eqref{eq:discrete_optimization_problem_y__}. In our experiments, we use LOBPCG \cite{lobpcg_method}, a preconditioned eigensolver\footnote{Locally Optimal Block Preconditioned Conjugate Gradient method.} for solving large positive definite generalized eigenproblems.
\end{remark}

%\noteHT{Explain the SPONGE algorithm pipeline: obtain embedding of vertices in $\matR^k$ via bottom k-eigenvectors as above, apply k-means clustering to the embedding. Also explain that we consider \textsc{SPONGE}$_{sym}$ as well}
%%%%-----------------------------------------
%%%% \pedro{verify this: I think its worthwhile putting this as a theorem/lemma. I am not sure if the last statement has the wrong order, i.e. $(\Ln +  \taup \Dp, \Lp +  \taun \Dn)$ instead of $(\Lp +  \taun \Dn, \Ln +  \taup \Dp)$... does the ordering matter? }.
%%%%-----------------------------------------

%%%  \section{Problem setup and SSBM}
%\input{SSBM_Model_input}

% \onecolumn

\vspace{-3mm}
%---------------------------------
% Analysis for SPONGE under SSBM
%---------------------------------
\section{Analysis of \textsc{SPONGE} under SSBM}  \label{sec:SPONGE_Theory}
\vspace{-3mm}

We begin by introducing the signed stochastic block model (SSBM) 
in Section \ref{subsec:SSBM_model} and then theoretically analyze the 
performance of \textsc{SPONGE} in Section \ref{subsec:theory_core_sponge}.

\vspace{-2mm}
%--------
% SSBM
%--------
\subsection{Signed stochastic block model}  \label{subsec:SSBM_model}
\vspace{-2mm}
%

\iffalse 
OLD TEXT:

\section{Problem setup}

\subsection{A signed stochastic block model (SSBM) for signed networks}

To compare algorithms in an environment with a known ground truth, we use a signed variant of the stochastic block model to generate test networks. The \emph{signed stochastic block model (SSBM)} we used is defined as follows.

We fix four parameters with which to generate graphs:
\begin{itemize}
    \item $n$ - number of nodes in network
    \item $k$ - number of planted communities
    \item $p$ - edge presence probability
    \item $\eta$ - edge sign flip probability
\end{itemize}

To generate a graph, we take $n$ nodes and place them arbitrarily into $k$ equally-sized (up to rounding) clusters. Each edge in the graph is present independently with probability $p$. Initially, edges within communities are given sign $+1$, and edges between communities are given sign $-1$. Then, the sign of each edge is independently flipped with probability $\eta$. Note that the planted communities in SSBM graphs can only be recoverable by signed clustering algorithms if $\eta<0.5$, since otherwise communities have an equal or lower proportion of positive to negative edges than the graph as a whole, and would not constitute a good clustering output.

\fi 
%%%%%%%
%%%%%%%
%%%%%%%

For ease of exposition, we assume $n$ is a multiple of $k$, and partition the vertices of $G$ into 
$k$-equally sized clusters 
% $\calC_1, \dots, \calC_k$. 
$C_1, \dots, C_k$. 
In particular, we assume w.l.o.g 
that 
% $\calC_l = \set{\frac{(l-1)n}{k}+1,\dots,\frac{l n}{k}}$ 
$C_l = \set{\frac{(l-1)n}{k}+1,\dots,\frac{l n}{k}}$ 
for $l=1,\dots,k$. 
The graph $G$ follows the Erd\H{o}s-R\'enyi random graph model $G(n,p)$ wherein each edge takes value $+1$ if both its endpoints are contained in the same cluster, and $-1$ otherwise. To model noise, %in the graph, 
we flip the sign of each edge independently with probability $\eta \in [0,1/2)$. 

Let $A \in \set{0, \pm 1}^{n \times n}$ denote the adjacency matrix of $G$, then $(A_{ij})_{i \leq j}$ are independent random variables. Recall that $A = A^+ - A^-$,
where $ A^+, A^- \in \set{0,1}^{n \times n}$ are the adjacency matrices of the unsigned graphs 
$\Gp,\Gn$ respectively. Then, $ (A_{ij}^+)_{i \leq j} $ are independent,  
and similarly $ (A_{ij}^-)_{i \leq j}$ are also independent. 
But clearly, for given $i,j \in [n]$ with $i \neq j$, %the entries 
$A_{ij}^+$ and  $A_{ij}^-$ are dependent.

\begin{remark} \label{rem:ssbm}% \hspace{-2mm}
Contrary to stochastic block models for unsigned graphs, we do not require the % edge probabilities 
%of  % vertices 
intra-cluster edge probabilities to be different from those of
%those lying % in different
%across 
inter-cluster edges. While this is 
necessary in the unsigned case for detecting clusters (eg. \cite{Mossel2015a,Mossel2015b}), 
it is not the case for signed networks since the sign of the edge already achieves this purpose implicitly. 
In fact, as one would expect, it is the parameter $\eta$ that is crucial for identifiability, 
% which
as shown  % next
formally in our analysis.
\end{remark}

\vspace{-2mm}
%--------------------------------------------
% Theoretical results for \textsc{SPONGE}
%--------------------------------------------
\subsection{Theoretical results for \textsc{SPONGE}} \label{subsec:theory_core_sponge}
\vspace{-2mm}

We now theoretically analyze the performance of \textsc{SPONGE} under the SSBM. 
In particular, we analyze the embedding % obtained from
given by  the smallest $k$ eigenvectors of 
% the following matrix operator

\vspace{-8mm}
\begin{equation*}
  T = (L^{-} + \tau^+ D^{+} )^{-1/2} (L^{+}  + \tau^{-} D^{-} ) (L^{-} + \tau^+ D^{+} )^{-1/2}, 
  %\label{eq:defT}
\end{equation*}
\vspace{-8mm}

for parameters $\tau^-, \tau^+ > 0$. 
%
% (already defined) Here $ L^{+}$ denotes the Laplacian corresponding to the subgraph of positive edges with adjacency matrix $A^+$, which is defined as $ L^{+} =   D^{+}  -  A^{+} $, where $  D^{+} $ is a diagonal matrix with the positive degrees on the diagonal, i.e., $D^+_{ii} = \sum_{j=1}^n A_{ij}$. Analogous considerations are made for $ L^{-}$ and $ D^{-}$. 
%
Recall that $(\lambda, v)$ is an eigenpair of $T$ 
if and only if $ (\lambda, (L^{-} + \tau^+ D^{+} )^{-1/2}  v ) $ is a generalized eigenpair 
for the matrix pencil  $(L^{+}  + \tau^{-} D^{-}, L^{-} + \tau^+ D^{+})$. 
We % make the assumption
assume throughout that both $L^{+}  + \tau^{-} D^{-}$  and   $ L^{-} + \tau^+ D^{+}$ are full rank. 
For ease of exposition, we focus on the case $k=2$ but the results can be extended to the general $k \geq 2$ setting (work in progress) using the same proof outline. Denote 

\vspace{-9mm}
\begin{align*} %\label{eq:TbarDef}
  \Tbar & = ( \ELn + \tau^+ \EDp )^{-1/2}   \\
        &  \quad   (\ELp  + \tau^{-}  \EDn )  ( \ELn + \tau^+ \EDp )^{-1/2},    \nonumber
\end{align*}
\vspace{-8mm}

and also denote

\vspace{-8mm}
\begin{equation*} %\label{eq:k_2_small_eigmats}
V_2(T) = [v_{n}(T) \ v_{n-1}(T)],  
% \in \matR^{n \times 2}, 
\quad 
V_2(\Tbar) = [v_{n}(\Tbar) \ v_{n-1}(\Tbar)], 
% \in \matR^{n \times 2}
\end{equation*}
\vspace{-8mm}

to be $n \times 2$  matrices consisting of the smallest two (unit $\ell_2$ norm) eigenvectors of $T, \Tbar$ respectively. Let 

\vspace{-5mm}
\begin{equation} \label{eq:k_2_inform_vect}
\infovec = \frac{1}{\sqrt{n}}( \underbrace{ 1, \ldots, 1}_{n/2}, \;\; \underbrace{-1, \ldots, -1}_{n/2})^T \in \matR^n
\end{equation}
\vspace{-5mm}

correspond to the ``ground truth'' or  %  the 
``planted clusters'' % that 
we seek to recover.
Our main result is the following. 
% the following theorem.
%
%---------------------------------
% Main theorem : bottom k eigvecs
%---------------------------------
\begin{theorem} \label{thm:sponge_k_2_bot_k_short}
For $\eta \in [0, 1/2)$ let $ \tau^+, \tau^- > 0$ satisfy 

\vspace{-5mm}
\begin{equation} \label{eq:tau_cond_sponge_k2bot_k_short}
\tau^- <  \tau^+  \Big( \frac{ \frac{n}{2} -1 + \eta  }{ \frac{n}{2} - \eta } \Big). 
\end{equation}
\vspace{-5mm}

Then it holds that $\set{v_{n-1}(\Tbar), v_{n}(\Tbar)} = \set{\frac{1}{\sqrt{n}}\ones,\infovec}$ where $\infovec$ is defined 
in \eqref{eq:k_2_inform_vect}. 
Moreover, assuming $n \geq 6 $, for given $0 < \varepsilon \leq 1/2$, $ \epsilon \in (0,1) $ and $\varepsilon_{\tau} \in (0,1)$ let $\tau^- \leq \varepsilon_{\tau}  \tau^+  \Big( \frac{ \frac{n}{2} -1 + \eta  }{ \frac{n}{2} - \eta } \Big)$ and 
$p \geq c^{\prime}_1(\varepsilon, \tau^+, \tau^-,\varepsilon_{\tau}, \eta,\epsilon) \frac{\log n}{n}$ 
where $c^{\prime}_1(\varepsilon, \tau^+, \tau^-,\varepsilon_{\tau}, \eta,\epsilon) > 0$ depends only on 
$\varepsilon, \tau^+, \tau^-,\varepsilon_{\tau}, \eta,\epsilon$. 
%
%\begin{align*}
%p & \geq   
% \max  \bigg\{ 
%24,  
%\frac{36 \tce^2}{ (\tau^+)^2},   
%\frac{36 \tce^2}{ (\tau^-)^2},  \\
%  & \quad 
%\Big(  \frac{ \bar{c}(\varepsilon, \tau^+, \tau^-)    }{  \epsilon   \min  \Big\{   \frac{2}{3} \frac{ %(1-\varepsilon_{\tau}) } { (1 + \tau^+  ) },  \frac{(1-2 \eta)}{3}   \frac{ (3 + \tau^+ + \tau^- ) }{ %(1+\tau^+)^2 }    \Big\}   }    \Big)^4 
%\bigg\}     
%\Big( \frac{\log n}{n} \Big), 
%\end{align*}
%
%where $ \tce	= (1 + \varepsilon) 2 \sqrt{2} +1  +\sqrt{3}$, and 
%
%
%\begin{align*}
%\bar{c}(\varepsilon, \tau^+, \tau^-)
%& = \frac{ 3^{3/2} \sqrt{2} \; \tce^{1/2} (1+\tau^-) } {  (\tau^+)^{3/2}   } + 
%\frac{3 \tce }{ \tau^+ }  \\
%& +  \frac{6^{3/2} \; \tce^{3/2} }{ (\tau^+)^{3/2} }  
%  + \frac{18 \; \tce^{2} }{ (\tau^+)^{2} } + \frac{9 \; \tce (1+\tau^-) }{ (\tau^+)^{2} }.   
%\end{align*}
%
%
Then there exists a constant $c_{\varepsilon} > 0$  depending only on $\varepsilon$ such that with probability at least 
$1 - \frac{4}{n} - 2n \exp{ \big( \frac{ - p  n }{c_{\varepsilon}} \big) }$, it holds that

\vspace{-4mm}
\begin{equation*}
	\norm{( I - V_2(T) V_2(T)^T) V_2(\bar{T})}_2  \leq  \frac{ \epsilon }{ 1 - \epsilon}. 
\end{equation*}
\end{theorem}

\vspace{-4mm}
The theorem states that $\calR(V_2(T))$ is close to $\calR(V_2(\Tbar))$ with high probability provided $n,p$ are suitably large, and $\taun$ is sufficiently small compared to $\taup$. The latter condition is required to ensure that the smallest two eigenvectors of $\Tbar$ are $\set{\frac{1}{\sqrt{n}}\ones, w}$. Also note that the embedding generated by any orthonormal basis for $\calR(V_2(T))$ leads to the same clustering performance\footnote{For a $2 \times 2$ orthogonal matrix $O$, the rows of the matrix $V_2(T) O$ are obtained via the same orthogonal transformation on the corresponding rows of $V_2(T)$.}. Since the embedding corresponding to $V_2(\Tbar)$ leads to perfectly separated (ground truth) clusters, hence the closer $\calR(V_2(T))$ is to $\calR(V_2(\Tbar))$, the better is the clustering performance. Using standard tools, one can actually use bounds 
on subspace recovery to bound the misclustering rate of $k$-means (see for eg. \cite{QinRohe2013}).
%---------------------
% Proof Sketch/steps
%---------------------

\vspace{-3mm}
\paragraph{Proof sketch.} The proof is deferred to the appendix, but the main steps involved are as follows.
%\vspace{-3mm}
%\begin{enumerate}
%\item 
We first compute the spectra of $ \expec[\Ln],\expec[\Lp], \expec[\Dn]$, and $\expec[\Dp] $ by finding the eigenvalues and the 
corresponding \textit{relevant} eigenvectors (i.e., associated to 
the smallest two eigenvalues).
%, as described in Lemma \ref{lem:expecs_posneg_mats}. 
We then identify conditions on the parameters $\taup,\taun$ under which the smallest two eigenvectors of 
$\Tbar$ are $\set{\infovec, \frac{1}{\sqrt{n}}\ones}$.  As shown in the proof, $\infovec$ is always one of the smallest two eigenvectors (since $\taup,\taun > 0$). The condition \eqref{eq:tau_cond_sponge_k2bot_k_short} leads 
to $\frac{1}{\sqrt{n}}\ones \in \set{v_{n-1}(\Tbar),v_n(\Tbar)}$. 
%(outlined in Lemma \ref{lem:sponge_Tbar_spectra}). 
%
%\item In Lemma \ref{lem:sponge_Tbar_pert}, we derive deterministic perturbation bounds on $ || T - \Tbar  ||_2 $. 
%
%\item In Lemma \ref{lem:sponge_conc_k_2}, we derive concentration bounds for $A^-, A^+, D^-, D^+$.
%
%\item 
Next, we derive concentration bounds  
using tools from random matrix theory for $A^-, A^+, D^-, D^+$ holding with high probability. This in turn leads to a bound on $|| T - \Tbar  ||_2$. 
%
%\item 
Combining the above results and by controlling the perturbation term $|| T - \Tbar  ||_2$, we obtain via the Davis-Kahan theorem \cite{daviskahan}, a bound on 
%
%\begin{align*}
$\norm{\sin( \Theta(\calR(V_2(T)), \calR(V_2(\Tbar))))}_2$ which equals $\norm{( I - V_2(T) V_2(T)^T) V_2(\bar{T})}_2$.
%
%\end{align*}
%
Here, $\Theta(\calR(V_2(T)), \calR(V_2(\Tbar)))$ is the diagonal matrix of canonical angles between $\calR(V_2(T))$ and $\calR(V_2(\Tbar))$.
%\end{enumerate}
%

\vspace{-2mm}
\paragraph{Selecting only $v_n(T)$.} Alternately, one could consider taking just the smallest eigenvector of $T$, i.e. $v_n(T)$, leading to a one-dimensional embedding. 
The following theorem states that, provided $\taun$ is sufficiently larger than $\taup$, and if $n,p$ are suitably large, then $\calR(v_n(T))$ is close to $\calR(\infovec)$ with high probability.
%
%--------------------------------------
% Main theorem k=2: bottom k-1 eigvecs
%--------------------------------------
\begin{theorem} \label{thm:sponge_k_2_bot_k1_short}
For $\eta \in [0, 1/2)$ let $ \tau^+, \tau^- > 0$ satisfy 

\vspace{-7mm}
\begin{align} \label{eq:tau_cond_sponge_k2bot_k1_short}
\tau^- >  \left(\frac{\eta}{1-\eta}\right)   
\Big( \frac{ \frac{n}{2} -1 + \eta  }{ \frac{n}{2} - \eta } \Big) \tau^+.
\end{align}
\vspace{-5mm}

Then it holds that $v_{n}(\Tbar) = \infovec$ with $\infovec$ defined in \eqref{eq:k_2_inform_vect}. 

Moreover, assuming $n \geq 6 $, for given $0 < \varepsilon \leq 1/2,  \epsilon \in (0,1) $ and $\varepsilon_{\tau} \in (0,1)$ let 
$\tau^- \geq  \frac{1}{\varepsilon_{\tau}} \left(\frac{\eta}{1-\eta}\right)   
\Big( \frac{ \frac{n}{2} -1 + \eta  }{ \frac{n}{2} - \eta } \Big) \tau^+,$  

\vspace{-4mm}
$$p \geq c^{\prime}_2(\varepsilon, \tau^+, \tau^-,\varepsilon_{\tau}, \eta,\epsilon) \frac{\log n}{n},$$ 
\vspace{-6mm}

where $c^{\prime}_2(\varepsilon, \tau^+, \tau^-,\varepsilon_{\tau}, \eta,\epsilon) > 0$ 
depends only on the indicated parameters. 
%
%\begin{align*}
%p & \geq   
%\max  \bigg\{ 
%24,  
%\frac{36 \tce^2}{ (\tau^+)^2},   
%\frac{36 \tce^2}{ (\tau^-)^2},  \\
%& \quad \Big(  \frac{ \bar{c}(\varepsilon, \tau^+, \tau^-) }{ \epsilon   
%\min  \Big\{ \frac{ \eta (\frac{1}{\varepsilon_{\tau}} - 1) } { 1 - \eta + \tau^+   },  \frac{(1-2 \eta)}{3} %  \frac{ (3 + \tau^+ + \tau^- ) }{ (1+\tau^+)^2 }    \Big\}   }    \Big)^4 
%\bigg\}     
%\Big( \frac{\log n}{n} \Big), 
%\end{align*}
%
%where $\tce$ and $\bar{c}(\varepsilon, \tau^+, \tau^-)$ are as defined in Theorem \ref{thm:sponge_k_2_bot_k}. 
Then there exists a constant $c_{\varepsilon} > 0$  depending only on $\varepsilon$ such that with probability at least 
$\Big( 1 - \frac{4}{n} - 2n \exp{ \big( \frac{ - p  n }{c_{\varepsilon}} \big) }  \Big),$ it holds that

\vspace{-4mm}
\begin{equation*}
	\norm{( I - v_n(T) v_n(T)^T) w}_2  \leq  \frac{ \epsilon }{ 1 - \epsilon}. 
\end{equation*}
\vspace{-4mm}

\end{theorem}

\vspace{-3mm}
The proof is deferred to the appendix, being similar to that of Theorem \ref{thm:sponge_k_2_bot_k_short}. The main difference is in the conditions on $\taun,\taup$ for ensuring that the smallest (two) eigenvector(s) of $\Tbar$ correspond to the ground truth clustering; these are clearly weaker in Theorem \ref{thm:sponge_k_2_bot_k1_short}
compared to Theorem \ref{thm:sponge_k_2_bot_k_short}. For eg. if $\eta = 0$, % (no noise), 
then any $\taup,\taun > 0$   % would 
imply $v_n(\Tbar) = \infovec$ % as seen in
by Theorem \ref{thm:sponge_k_2_bot_k1_short}, while the analogous statement is not true in Theorem \ref{thm:sponge_k_2_bot_k_short}.

%\vspace{-1mm}
%\subsection{\small Numerical sensitivity analysis for $\tau^{+}$ and $\tau^{-}$}
% % parameter
%\vspace{-1mm}

Figure  \ref{fig:eigengap_SPONGES_LbarSym} (left) compares the $40$ smallest eigenvalues of \textsc{SPONGE}, \textsc{SPONGE}$_{sym}$, and $\bar{L}_{sym}$ in the scenario $n=10,000$, $p=0.01$, $\eta=0.1$, and $k=10$. \textsc{SPONGE}$_{sym}$ clearly exhibits the largest spectral gap between the $9^{th}$ and $10^{th}$ eigenvalue.
Figure \ref{fig:eigengap_SPONGES_LbarSym} (right) also compares the spectral densities of \textsc{SPONGE}$_{sym}$ for several $ \eta \in \{ 0, 0.1, 0.2 \}$. As expected, the spectral gap decreases as the noise level increases. 
%, which hinders the recovery process. 
% and the recovery process becomes more difficult. 
%
%\vspace{-1mm}
\begin{figure}[!htp]
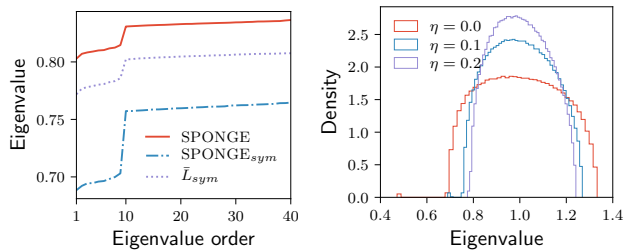
  % [t]
\small
\captionsetup[subfigure]{labelformat=empty}
\captionsetup[subfigure]{skip=-3pt}
\hspace{-4mm}
\begin{centering}
% \captionsetup{skip=10pt} % local setting for this subfigure
\subcaptionbox[Short Subcaption]{ % zzz here 
%  \label{subfig:sublabel1} 
}[ 0.50\columnwidth ]
{\includegraphics[width=0.55\columnwidth] {{{Figures/Eigenvalue_Plots/compare_eigengaps}}} }
% \hspace{0.01\textwidth} % separation
\subcaptionbox[Short Subcaption]{  % zzz here
% \label{subfig:sublabel3}
}[ 0.48\columnwidth ]
{\includegraphics[width=0.56\columnwidth] {{{Figures/Eigenvalue_Plots/eigenspectrum_SPONGEsym}}} }
% \hspace{0.02\textwidth} % separation
 \end{centering}
\vspace{-5mm} 
\captionsetup{width=1.05\linewidth}
\caption[Short Caption]{\footnotesize \hspace{-2mm} Left: % Eigenvalue gap for
Bottom spectrum %\noteHT{(Aldo-Peter: shouldn't this be bottom spectrum? Also, note that as per notation in the paper, we have $\lambda_1$ as smallest eigenvalue and $\lambda_n$ as largest. The x-axis in this plot seems to have opposite meaning..)} \notePD{Yes, it should.}
of  \textsc{SPONGE}$_{sym}$,  \textsc{SPONGE},  % compared  to that of 
and $\bar{L}_{sym}$. Right: Spectrum of \textsc{SPONGE}$_{sym}$ for three values of noise $\eta$,  % with 
($n = 10000$, $p=0.01$, and $k=10$).}
\label{fig:eigengap_SPONGES_LbarSym}
\end{figure}
%\vspace{-1mm}
%
%
Figure \ref{fig:SPONGES_k_km1} compares heatmaps of recovery rates for 
\textsc{SPONGE}  and  \textsc{SPONGE}$_{sym}$ for 
%$n=5000, p=0.012$, $\eta=0.125$ and 
$k=2$ clusters and varying $\taup,\taun > 0$. Observe that \eqref{eq:tau_cond_sponge_k2bot_k1_short} shows up when we consider only the smallest eigenvector for \textsc{SPONGE}. 
%Interestingly, a large range of $\taup,\taun$ are allowed when taking the smallest two eigenvectors
%when the embedding is given by only the smallest eigenvector versus the smallest two eigenvectors. 
%
%
Figure  \ref{fig:tauHeatmaps} shows similar plots for $k=8$, where we observe that 
\textsc{SPONGE}$_{sym}$ allows for a wider choice of $\taup,\taun > 0$ for successful clustering. 
%shows 
%recovery rates heatmaps for \textsc{SPONGE}  and  \textsc{SPONGE}$_{sym}$
% when recovering
%with 
%$k=8$ clusters, as we vary $\tau^+$ and $\tau^-$, for a 
%very sparse graph with $n=5000$ nodes, edge density $p=0.012$, 
% for two and noise levels $\eta = \{ 0, 0.125 \}$.
%

At a high level, our proof technique, using tools from 
matrix perturbation and random matrices,  %is not new and 
has been used before for analyzing 
spectral methods for clustering unsigned graphs \cite{QinRohe2013}. 
% (eg. \cite{QinRohe2013}).  
In the sparse regime where $p \rightarrow 0$ as $n \rightarrow \infty$, Theorems \ref{thm:sponge_k_2_bot_k_short}, 
\ref{thm:sponge_k_2_bot_k1_short} state that $p \gtrsim \frac{\log n}{n}$ ensures that 
the success probability tends to  one. Similar scalings are known for 
unsigned graphs, however there, the intra-cluster and inter-cluster edge probabilities necessarily % need to 
must be different (see Remark \ref{rem:ssbm}). 

\begin{figure}[!htp]
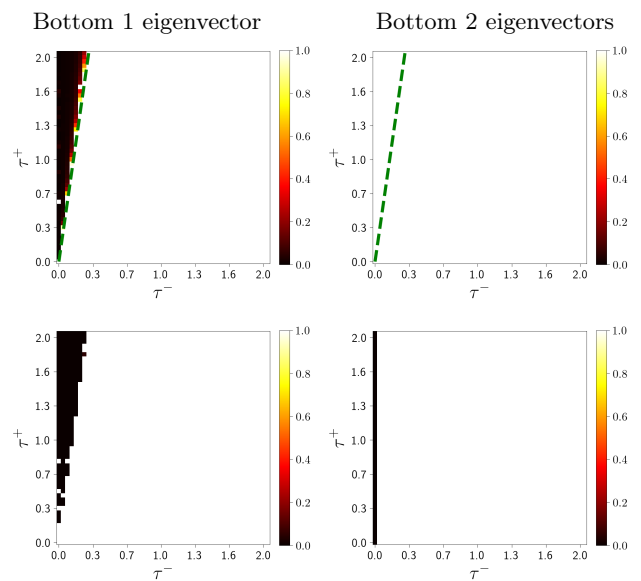

\small
\begin{centering}
\captionsetup[subfigure]{labelformat=empty}
\captionsetup[subfigure]{skip=-7pt}
\hspace{-10mm}
\subcaptionbox{\hspace{4mm} Bottom $1$ eigenvector}[ 0.49\columnwidth ]{}	
\subcaptionbox{ \hspace{4mm}  Bottom $2$ eigenvectors}[ 0.49\columnwidth ]{}
\subcaptionbox{  }[ 0.49\columnwidth ]{\includegraphics[width=0.5\columnwidth]{{{Figures/New_tau_figures/SPONGE_km1_n5000_p0.012_k2_eta0.125_grid50}}}}	
\subcaptionbox{ }[ 0.49\columnwidth ]{\includegraphics[width=0.5\columnwidth]{{{Figures/New_tau_figures/SPONGE_k_n5000_p0.012_k2_eta0.125_grid50}}}}	% \par
 \vspace{-1mm}
\subcaptionbox{}[ 0.49\columnwidth ]{\includegraphics[width=0.5\columnwidth]{{{Figures/New_tau_figures/SPONGE_sym_km1_n5000_p0.012_k2_eta0.125_grid50}}}}	
\subcaptionbox{}[ 0.49\columnwidth ]{\includegraphics[width=0.5\columnwidth]{{{Figures/New_tau_figures/SPONGE_sym_k_n5000_p0.012_k2_eta0.125_grid50}}}}
% \par
\end{centering}
\captionsetup{width=1.05\linewidth}
\vspace{-2mm}
\caption{\footnotesize Heatmap of recovery rates for $k=2$ clusters for \textsc{SPONGE} (top) and  \textsc{SPONGE}$_{sym}$ (bottom), with $n=5000, p=0.012$ and $\eta=0.125$, 
% when using the smallest eigenvector % (left) and the smallest two eigenvectors. %  (right). 
via the bottom one or two eigenvectors. 
The green dotted line is the condition \eqref{eq:tau_cond_sponge_k2bot_k1_short}.
% \noteMC{eqq} (top row).
% (a) \textsc{SPONGE}  with first eigenvector, (b) \textsc{SPONGE}  with first two eigenvectors, (c) \textsc{SPONGE}$_{sym}$ with first eigenvector (d)  \textsc{SPONGE}$_{sym}$ with first two eigenvectors.}
}
\label{fig:SPONGES_k_km1}
\end{figure}
\vspace{-3mm}
\begin{figure}
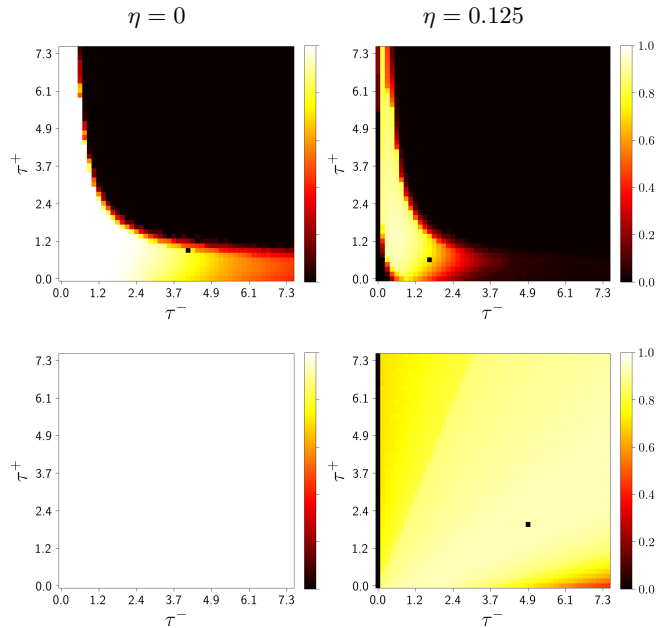

\small 
\begin{centering}
\captionsetup[subfigure]{labelformat=empty}
\captionsetup[subfigure]{skip=-7pt}
\subcaptionbox{  $ \eta=0$ }[ 0.49\columnwidth ]{}	
\subcaptionbox{  $ \eta=0.125$}[ 0.49\columnwidth ]{}
\subcaptionbox{ }[ 0.49\columnwidth ]{\includegraphics[width=0.55\columnwidth]{{{Figures/New_tau_figures/SPONGE_km1_n5000_p0.012_k8_eta0.0_grid50}}}}	
\subcaptionbox{   }[ 0.49\columnwidth ]{\includegraphics[width=0.55\columnwidth]{{{Figures/New_tau_figures/SPONGE_km1_n5000_p0.012_k8_eta0.125_grid50}}}}
 \vspace{-1mm}
\subcaptionbox{  }[ 0.49\columnwidth ]{\includegraphics[width=0.55\columnwidth]{{{Figures/New_tau_figures/SPONGE_sym_km1_n5000_p0.012_k8_eta0.0_grid50}}}}	
\subcaptionbox{ }[ 0.49\columnwidth ]{\includegraphics[width=0.55\columnwidth]{{{Figures/New_tau_figures/SPONGE_sym_km1_n5000_p0.012_k8_eta0.125_grid50}}}}
\end{centering}
\captionsetup{width=1.05\linewidth}
\vspace{-2mm}
\caption{\footnotesize Heatmap of recovery rates for \textsc{SPONGE (top)}  and  \textsc{SPONGE}$_{sym}$ (bottom), using the bottom $k-1$ eigenvectors, as we vary $\tau^+,\tau^-$, %for a % sparse 
with $n=5000$, % nodes with
$p=0.012$, $k=8$ and $\eta \in \{ 0, 0.125 \}$.} %\noteHT{Aldo-Peter: which version are we using - bottom $k$ or bottom $k-1$??} \notePD{k-1}
\label{fig:tauHeatmaps}
\end{figure}
\vspace{-3mm}

%%% Analyzing SPONGE under SSBM

% \section{Theoretical analysis of SPONGE under SSBM} 

% Stuff goes here...

%\subsection{Analysis in expectation}

% \subsection{Concentration and interpretation}

% \subsection{Extension to multiple clusters}

% \twocolumn

%%% \section{Theoretical analysis of the Signed Laplacian $\bar{L}$ under SSBM} 
%%% Signed Laplacian based clustering

\vspace{1mm}
%-------------------------------------
% Signed Laplacian based clustering
%-------------------------------------
% \section{Signed Laplacian based clustering}
\section{Theoretical analysis of the Signed Laplacian $\bar{L}$ under SSBM} \label{sec:Signed_Laplacian_theory}
\vspace{-1mm}

In this section, we theoretically analyze the popular Signed Laplacian based method of 
Kunegis et al. \cite{kunegis2010spectral} for clustering signed graphs under the SSBM. 
This method is particularly appealing due to its simplicity, but to our knowledge, there do not exist 
any theoretical guarantees on the performance of this approach. We fill this gap by providing a detailed 
analysis for the $k = 2$ case. This choice is for ease of exposition, but the proof outline clearly 
extends\footnote{This is part of work currently in progress.} to any $k \geq 2$.

Recall that for a signed graph $G$ with adjacency matrix $A \in \set{0,\pm 1}^{n \times n}$, and with the 
diagonal matrix $\Dbar$ consisting of the degree terms defined as $\Dbar_{ii} := \sum_{j=1}^n \abs{A_{ij}}$, 
the Signed Laplacian of $G$, denoted by $\Lbar \in \matR^{n \times n}$, is given by $\Lbar = \Dbar - A$.
Kunegis et al. \cite{kunegis2010spectral} showed that $\Lbar$ is positive semi-definite for any 
graph (see \cite[Theorem 4.1]{kunegis2010spectral}). Moreover, they also showed that $\Lbar$ is positive 
definite iff the graph is unbalanced \cite[Theorem 4.4]{kunegis2010spectral}. 
Kunegis et al. proposed using $\Lbar$ to first compute a lower dimensional embedding of the 
graph -- obtained from the smallest $k$ eigenvectors of $\Lbar$ (in fact, as we will see, taking only $k-1$ is sufficient and more effective in signed graphs) and then clustering the obtained 
points in $\matR^k$ (or $\matR^{k-1}$) using any standard clustering method (e.g. $k$-means).

%---------------------------------------------------------------------
% Main result for clustering used Signed Laplacian for k=2 under SSBM
%---------------------------------------------------------------------
Our main result for the Signed Laplacian based clustering 
approach of Kunegis et al. \cite{kunegis2010spectral} is stated below, 
and the proof is deferred to the appendix.
%
%--------------------------------------------------------
% Main theorem for Signed Laplacian clustering, k = 2
%--------------------------------------------------------
\begin{theorem} \label{thm:k_2_signed_laplacian}
Assuming $0 \leq \eta < 1/2$, it holds that $v_n(\ELbar) = w$, where $\infovec$ is defined in \eqref{eq:k_2_inform_vect}. 
Moreover, let $n \geq 2$ and for given $0 < \epsilon < 1$, $0 < \varepsilon \leq 1/2$ let 

\vspace{-5mm}
$$p \geq \frac{4((1+\varepsilon)2\sqrt{2} + 1)^2}{\epsilon^2(1-2\eta)^2} \frac{\log n}{n}.$$
\vspace{-5mm}

Then there exists a constant $c_{\varepsilon}  > 0$ depending only on $\varepsilon$ such that with probability at least $1 - \frac{2}{n} - n \exp(- \frac{pn}{4 c_{\varepsilon}})$ it holds that $\norm{(I - \infovec \infovec^T) v_n(\Lbar)}_2 \leq \frac{\epsilon}{1-\epsilon}$. 
\end{theorem}

\vspace{-2mm}
Theorem \ref{thm:k_2_signed_laplacian} states that for $n, p$ suitably large, $\calR(v_n(\Lbar)) \approx \calR(w)$ with high probability. 
%This means that up to a sign flip, $v_n(\Lbar)$ % is closely aligned 
%closely aligns 
%with $\infovec$. For the purpose of clustering, note that 
%we do not care about recovering $\infovec$, but rather % care 
%about finding a vector in $\calR(\infovec)$. 
In particular, if $\eta$ is bounded away from $1/2$, then in the sparse regime where $p \rightarrow 0$ as $n \rightarrow \infty$, the success probability approaches one if $p \gtrsim \frac{\log n}{n}$. 
As seen in the proof, $\ELbar$ is positive definite if $\eta \neq 0$, and positive semi-definite otherwise. This makes sense since for $\eta = 0$, the generated graph (under the SSBM) is balanced by construction and thus is positive semi-definite \cite[Theorem 4.4]{kunegis2010spectral}. The fact that $\ELbar$ is positive definite 
for $\eta \neq 0$ tells us that the resulting graph will be unbalanced with high probability. Finally, we note that as $\eta$ approaches $1/2$, the condition on $p$ becomes stricter since the expected number of intra-cluster positive edges is almost the same as the number of inter-cluster positive edges (similarly for negative edges). Hence, to get a non-trivial lower bound on $p$, we require $n$ to be sufficiently large.

\vspace{-4mm}
%%% Experiments 
%% \section{Experiments}   Empty for now... 
\section{Numerical experiments}  \label{sec:numericalExperiments}

% \section{Numerical experiments}  \label{sec:numericalExperiments}
%

\vspace{-3mm}
This section contains numerical experiments comparing our   %proposed
\textsc{SPONGE} and \textsc{SPONGE}$_{sym}$ algorithms\footnote{Our current Python implementations are available at \url{https://github.com/alan-turing-institute/signet}} (setting $\taup = \taun = 1$), with several existing spectral signed clustering techniques based on: % the following matrices: 
the adjacency matrix $A$, the Signed Laplacian matrix $\bar L$, its symmetrically normalized version $\bar L_{sym}$ \cite{kunegis2010spectral}, and the two algorithms introduced in \cite{DhillonBalNormCut} that optimize the  Balanced Ratio Cut and the Balanced Normalized Cut objectives.
In all cases, the bottom $k-1$ (or top $k-1$ in the case of adjacency matrix $A$) eigenvectors of the relevant matrix or generalized eigenvalue problem are considered as an embedding, and kmeans$++$ is applied to obtain a $k$-clustering.
%
% \textsc{SPONGE}
%
% \begin{itemize}
%     \item Principal component analysis (PCA) using the un-normalized adjacency matrix $A$
%     \item The signed Laplacian $\bar L$ of Kunegis et al. \cite{kunegis2010spectral}
%     \item The symmetric normalized signed Laplacian $\bar L_{sym}$
%     \item The Balanced Ratio Cut objective of Chiang et al. \cite{DhillonBalNormCut}
%     \item The Balanced Normalized Cut objective of Chiang et al. \cite{DhillonBalNormCut}.
% \end{itemize}
% 
%%%-------------------------------  
%%%  Leave out from the AISTATS submission - not to revela identity 
% The mentioned algorithms, as well as the newly proposed $\text{SPONGE}$ and $\text{SPONGE}_{sym}$, can be found in the SigNet package \cite{signet}. 
%%%-------------------------------  
% 
Section \ref{subsec:experiments_1} contains numerical experiments on synthetic graphs generated under the SSBM, while  Section \ref{subsec:experiments_2} details the results obtained on four different real-world data sets. 
% An extensive set of simulations can be found in the appendix.
Additional experiments are available in the appendix.

%\noteHT{Aldo-Peter: We should mention clearly what the ``setup'' is. For which algorithm are we taking bottom/top $k$ and for which bottom/top $k-1$? (I mean for SPONGE, BNC, $\Lbar$ etc.) What $\taup,\taun$ are we using? Also mention how the final clustering is being done, i.e., we run kmeans$++$ on the embeddings to obtain final clustering. All this needs to be mentioned clearly upfront (in case the setup is fixed throughout) or mention at the appropriate place in the text.}

\vspace{-3mm}
\subsection{Signed stochastic block model }
\label{subsec:experiments_1}
\vspace{-3mm}

This section compares all algorithms on a variety of synthetic graphs generated from the SSBM.
%
% Since the underlying cluster structure is known
% , we can objectively compare the listed algorithms. 
%
Since the ground truth partition is available, we measure accuracy by the Adjusted Rand Index (ARI) \cite{ARI_JMLR_Gates_Ahn},  % which is closely related and alleviates some of the issues of the
            an improved version of the 
popular  Rand Index  \cite{rand1971}. Both measures indicate how well the recovered partition matches % the 
   ground truth, % partition. A value close to 1 indicates almost perfect recovery, and a value close to 0 indicates an almost random assignment of the nodes into clusters. 
with a value close to 1, resp. 0, indicating an almost perfect recovery, resp. an almost random assignment of the nodes into clusters.

%% https://haifengl.github.io/smile/api/java/smile/validation/AdjustedRandIndex.html
% "Adjusted Rand Index. Rand index is defined as the number of pairs of objects that are either in the same group or in different groups in both partitions divided by the total number of pairs of objects. The Rand index lies between 0 and 1. When two partitions agree perfectly, the Rand index achieves the maximum value 1. A problem with Rand index is that the expected value of the Rand index between two random partitions is not a constant. This problem is corrected by the adjusted Rand index that assumes the generalized hyper-geometric distribution as the model of randomness. The adjusted Rand index has the maximum value 1, and its expected value is 0 in the case of random clusters. A larger adjusted Rand index means a higher agreement between two partitions. The adjusted Rand index is recommended for measuring agreement even when the partitions compared have different numbers of clusters."

\begin{figure} % [t]
\small
% \captionsetup[subfigure]{labelformat=empty}
\captionsetup[subfigure]{skip=1pt}
% \hspace{-4mm}
	\begin{centering}
	\captionsetup{skip=5pt} % local setting for this subfigure
		\subcaptionbox[]{ $ k=2, p=0.001 $ }[ 0.48\columnwidth ]{\includegraphics[width=0.54\columnwidth]{{{Figures/SSBM_graphs/eta_recovery_n10000_p0.001_k2_Suniform}}}}	
		\subcaptionbox[]{ $ k=5, p=0.001 $ }[ 0.48\columnwidth ]{\includegraphics[width=0.54\columnwidth]{{{Figures/SSBM_graphs/eta_recovery_n10000_p0.001_k5_Suniform}}}}
		\subcaptionbox[]{ $ k=20, p=0.01 $ }[ 0.49\columnwidth ]{\includegraphics[width=0.54\columnwidth]{{{Figures/SSBM_graphs/eta_recovery_n10000_p0.01_k20_Suniform}}}}
		\subcaptionbox[]{ $ k=50, p=0.1 $ }[ 0.49\columnwidth ]{\includegraphics[width=0.54\columnwidth]{{{Figures/SSBM_graphs/eta_recovery_n10000_p0.1_k50_Suniform}}}}
	% \par
	\end{centering}
	\vspace{0mm}
	\captionsetup{width=0.99\linewidth}
	\caption{\footnotesize ARI recovery scores  % as a function of the % sign flipping probability 
	versus  $\eta$ for increasing $k$, with communities of equal size and $n=10000$. }
	\label{fig:SSBMa}
\end{figure}
\vspace{-1mm}

The SSBM considered here and introduced in Section  \ref{subsec:SSBM_model} has four parameters: $n$, $k$, $p$ and $\eta$. 
In our experiments, we fix $n = 10000$, and let $k \in \{ 2, 5, 10, 20, 50 \}$ with clusters chosen of equal size $n/k$.
%
%Using the  % Adjusted Rand Index 
%ARI as the recovery score for the cluster assignments, 
We analyze the performance of all algorithms by plotting mean and standard error, over 20 repetitions, of the ARI as a function of $\eta$ for $p \in \{ 0.001, 0.01, 0.1 \}$.
%
%We defer an extensive set of simulations to the Supplemental Material, 
% While all the results can be found in the Supplemental Material
%, we comment on the pattern emerged from these tests, and summarize it in the graphs 
The results are reported in Figure \ref{fig:SSBMa}. 
When 
%there are only 
$k=2$ (Figure \ref{fig:SSBMa} (a)), $\Lbar_{sym}$ performs slightly better than all other algorithms. 
As $k$ increases, the \textsc{SPONGE} algorithms 
%proposed in this work 
start to significantly outperform all %  the other state-of-the-art ones.
other methods. 
In particular, while for intermediate values of $k$ (Figure \ref{fig:SSBMa} (b)) \textsc{SPONGE}  was  % found to be 
the best performer, once $k=20$ or $k=50$ (Figure \ref{fig:SSBMa} (c) and (d))  \textsc{SPONGE}$_{sym}$  was %  found to be 
greatly superior, being able to perfectly recover the cluster structure ($\text{ARI} = 1$) when all other methods %  are not able to perform better than 
                %  perform as good as
% random guessing 
completely fail 
($\text{ARI} \approx 0$). 
%
%
\iffalse  ---> to rephrase 
In Figure \ref{fig:SSBMa}, although the sparsity of the graph was fixed, the  performance of  \textsc{SPONGE}$_{sym}$ under SSBM with many communities was found to be superior for a wide range of 
%the sparsity parameter 
$p$ (see Figure \ref{fig:SSBMb} or appendix).
\fi 
We remark that similar results, showing excellent recovery for large  $k$ via \textsc{SPONGE}$_{sym}$, hold true over a wider range of values of the sparsity $p$, and are reported in the appendix.

\vspace{-1mm}
\begin{figure}
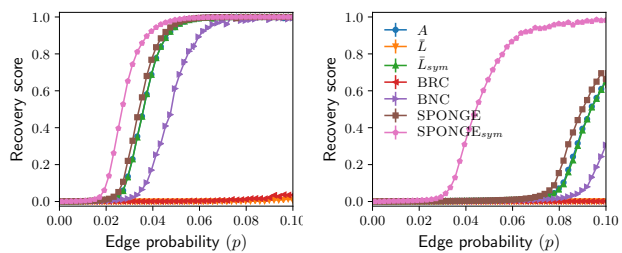
  % [t]
\captionsetup[subfigure]{skip=0pt}
	\begin{centering}
		\subcaptionbox{ $k=20, \eta=0.2$ }[ 0.49\columnwidth ]{\includegraphics[width=0.54\columnwidth]{{{Figures/SSBM_graphs/p_recovery_n10000_eta0.2_k20_Suniform}}}}
		\subcaptionbox{  $k=50, \eta=0.1$ }[ 0.49\columnwidth ]{\includegraphics[width=0.54\columnwidth]{{{Figures/SSBM_graphs/p_recovery_n10000_eta0.1_k50_Suniform}}}}
	\par\end{centering}
	\captionsetup{width=0.99\linewidth}
	\caption{\footnotesize ARI recovery scores as a function of the edge probability $p$, for $k=20$ and $k=50$ at two different noise levels. The communities are of equal size, and $n=10000$.}
\label{fig:SSBMb}	
\end{figure}
\vspace{-1mm}

\smallskip
We  % further 
also tested the algorithms on SSBM graphs with clusters of unequal sizes, with the probability of each node to be part of a given cluster being uniformly sampled  %  between 0 and 1
        in $[0,1]$, 
and subsequently normalized, which typically lead to widely different cluster sizes.  
%
%One should notice that this sampling choice typically generates widely different cluster sizes.
%
Under this setting (see Figure \ref{fig:SSBMc}),
% and appendix
\textsc{SPONGE}$_{sym}$ was still 
the best performer, 
        % found to be the best performing algorithm,  
although the extent of the performance gap was 
        % found to be
less pronounced. 
Interestingly, the performance of \textsc{BNC} often matched (but rarely overcame) that of  \textsc{SPONGE}$_{sym}$.
Overall, we find that  % when $k$ is large (at least 5),
for large enough $k \geq 5$, 
SPONGE and especially SPONGE$_{sym}$, outperform all state-of-art algorithms across a broad range of values for $p,\eta$, and for % all  values of 
$n$ sufficiently large for a clustering to be recoverable. 
%Armed with these encouraging results, we now test the algorithms on real-world examples of signed networks. 

%\noteAG{maybe we should mention that the p values in the first SSBM figure are the "most interesting" ones}
%\noteAG{maybe we should mention the bad performance of BRC, BNC and L}
%\noteAG{[conclusion of the subsection and shift to numerical tests on real data]}
%\notePD{I haven't really addressed your comments here, partly because of space but also because in my opinion, criticising performance of previous algorithms is best left implicit (since it's clear from the graphs)}

\vspace{-2mm}
\begin{figure}
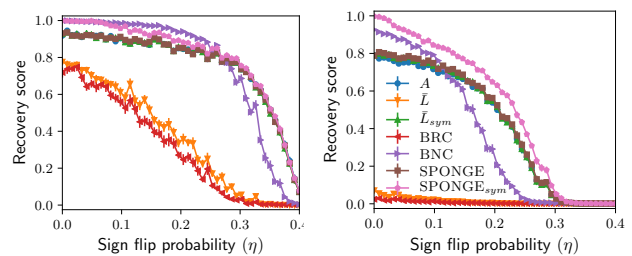
  % [t]
\captionsetup[subfigure]{skip=0pt}
	\begin{centering}
		\subcaptionbox{ $k=20, p=0.1$ }[ 0.49\columnwidth ]{\includegraphics[width=0.54\columnwidth]{{{Figures/SSBM_graphs/eta_recovery_n10000_p0.1_k20_Suneven}}}}
		\subcaptionbox{ $k=50, p=0.1$}[ 0.49\columnwidth ]{\includegraphics[width=0.55\columnwidth]{{{Figures/SSBM_graphs/eta_recovery_n10000_p0.1_k50_Suneven}}}}
	% \par
	\end{centering}
	\captionsetup{width=0.99\linewidth}
	\caption{\footnotesize ARI recovery scores of all algorithms, as a function of the noise $\eta$, for $ k \in \{20,50 \}$ clusters of randomly chosen sizes, and fixed edge density $p=0.1$. 
	% and for two values of $k$. The community sizes were randomly chosen.
}
\label{fig:SSBMc}	
\end{figure}
\vspace{-2mm}

\subsection{Real data}
\label{subsec:experiments_2}
\vspace{-2mm}
This section details the outcomes of experiments on a variety of real-world signed network data sets. Due to space constraints, we show results for the four algorithms that performed best on the synthetic experiments: \textsc{SPONGE}, \textsc{SPONGE}$_{sym}$,  \textsc{BNC} and $\bar{L}_{sym}$.
Since we no longer have ground truth, %  to compare with we instead show 
we compare the output of the algorithms by plotting the network adjacency matrix sorted by membership of the clusters produced. For our time series data applications, we also demonstrate visually that our algorithms have recovered meaningful information in their clusterings.

\vspace{-2mm}
\paragraph{Wikipedia elections.} We % start with 
consider the classic data set of Wikipedia Requests for Adminship  \cite{west2014exploiting}  % available 
from  
   % the Stanford Large Network Dataset Collection
SNAP \cite{snapnets}; a network of positive, neutral, and negative votes between Wikipedia editors running in adminship elections.  
   %%% The data is a type of social network, but where negative interactions also occur.
%
We construct a signed, undirected, weighted graph 
% by considering 
using the sums of edge weights for each pair of nodes. We then discard $0$-weighted edges and consider only the largest connected component of the resulting graph.  %As a result, 
Thus, we obtain a graph on $n=11,259$ nodes with $132,412$ (resp. $37,423$) positive (resp. negative) edges.  % We run the algorithms with $k=6$ clusters; 
Figure \ref{fig:wiki} shows the resulting adjacency matrix sorted by cluster membership 
   with  $k=6$, 
where blue (resp. red) denotes positive (resp. negative)  edges. 
Previous work on signed networks % , such as 
\cite{Mercado_2016_Geometric},   %has 
also succeeded in finding clustering structure in this data. However, the majority of the nodes are placed in a single large cluster which is very sparse and does not exhibit discernible associations. A major advantage of the clustering % shown 
in Figure \ref{fig:wiki} is that all clusters demonstrate a significantly higher ratio of positive to negative internal edges, compared to that of the graph as a whole. 
% We now apply our algorithms to signed graphs generated from three time series datasets.

\vspace{-1mm}
\begin{figure}[!htp]
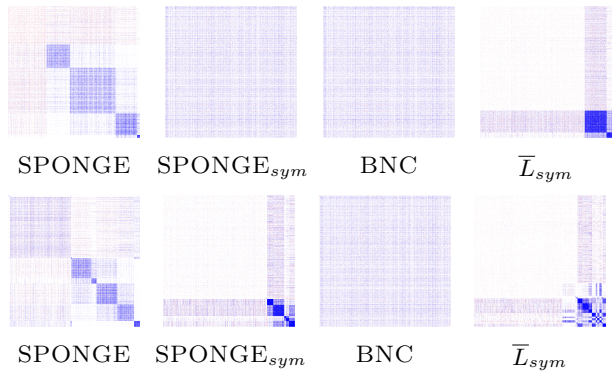

\captionsetup[subfigure]{labelformat=empty}
\begin{centering}
\subcaptionbox{\textsc{SPONGE} }{\includegraphics[width=0.24\columnwidth]{{{Figures/real_data_graphs_low_k/WIKIadjGEadd_k6}}}}	
\subcaptionbox{\textsc{SPONGE}$_{sym}$ }{\includegraphics[width=0.24\columnwidth]{{{Figures/real_data_graphs_low_k/WIKIadjGEmul_k6}}}}	
\subcaptionbox{\textsc{BNC} }{\includegraphics[width=0.24\columnwidth]{{{Figures/real_data_graphs_low_k/WIKIadjBNC_k6}}}}	
\subcaptionbox{$\bar{L}_{sym}$ }{\includegraphics[width=0.24\columnwidth]{{{Figures/real_data_graphs_low_k/WIKIadjLap_k6}}}}
\end{centering}
\begin{centering}
\subcaptionbox{\textsc{SPONGE} }{\includegraphics[width=0.24\columnwidth]{{{Figures/real_data_graphs_high_k/WIKIadjGEadd_k50}}}}	
\subcaptionbox{\textsc{SPONGE}$_{sym}$ }{\includegraphics[width=0.24\columnwidth]{{{Figures/real_data_graphs_high_k/WIKIadjGEmul_k50}}}}	
\subcaptionbox{\textsc{BNC} }{\includegraphics[width=0.24\columnwidth]{{{Figures/real_data_graphs_high_k/WIKIadjBNC_k50}}}}	
\subcaptionbox{$\bar{L}_{sym}$ }{\includegraphics[width=0.24\columnwidth]{{{Figures/real_data_graphs_high_k/WIKIadjLbar_k50}}}}	
% \par
\end{centering}
\captionsetup{width=0.99\linewidth}
\caption{ Sorted adjacency matrix of the Wikipedia graph
% sorted by (a,e) \textsc{SPONGE} (b,f) \textsc{SPONGE}$_{sym}$ (c,g) \textsc{BNC} (d,h)  $\bar{L}_{sym}$, 
for $k=6$ (top row) and $k=50$ (bottom row).
}
\label{fig:wiki}
\end{figure}
\vspace{-1mm}

\vspace{-2mm}
\paragraph{Correlations of financial market returns.}
We consider % time series 
daily prices for $n=1500$ stocks in the S\&P % Composite 
1500 Index, during 2003-2015, and build correlation matrices from market excess returns.  % (against the S\&P 500 ETF).
We refer the reader to the appendix, for a detailed overview of our steps. 
Figure \ref{fig:SP1500_k_10_30} shows that, for $k \in \{ 10, 30 \}$, we are able to find a meaningful segmentation of the market. In Figure \ref{fig:spsec}, we interpret  our results in light of the popular GICS sector decomposition \cite{GICS_citation}.  %  (at the sector level) 
% popular among practitioners, which is a segmentation of the US market into 10 different industries of the US economy  % , based entirely on fundamentals data 
%   $\cite{GICS_citation}. 
% and show in Figure \ref{fig:spsec} the cluster intersection with GICS. % the GICS sectors. 
% the instruments are sorted by cluster membership, with vertical black lines separating the clusters. Some of the clusters detected consist entirely of companies from a single sector, such as Financials or Utilities. 

\vspace{-1mm}
\begin{figure}[!htp]
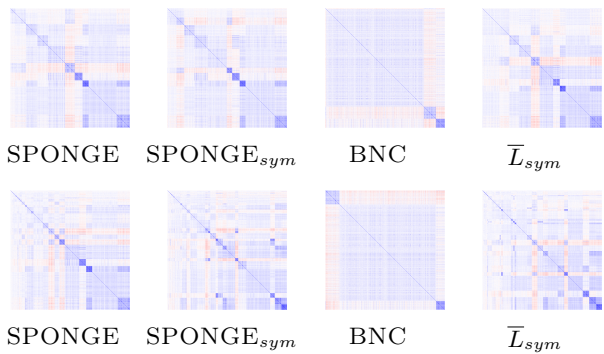

\captionsetup[subfigure]{labelformat=empty}
\begin{centering}
\subcaptionbox{\textsc{SPONGE}%, $k=10$
}{\includegraphics[width=0.24\columnwidth]{{{Figures/real_data_graphs_low_k/SP1500adjGEadd_k10}}}}	
\subcaptionbox{ \small  \textsc{SPONGE}$_{sym}$ % $k=10$
}{\includegraphics[width=0.24\columnwidth]{{{Figures/real_data_graphs_low_k/SP1500adjGEmul_k10}}}}	
\subcaptionbox{\textsc{BNC}  % , $k=10$
}{\includegraphics[width=0.24\columnwidth]{{{Figures/real_data_graphs_low_k/SP1500adjBNC_k10}}}}	
\subcaptionbox{$\bar{L}_{sym}$ % , $k=10$
}{\includegraphics[width=0.24\columnwidth]{{{Figures/real_data_graphs_low_k/SP1500adjLbar_k10}}}}	
% \par
\end{centering}
\begin{centering}
\subcaptionbox{\textsc{SPONGE}}{\includegraphics[width=0.24\columnwidth]{{{Figures/real_data_graphs_high_k/SP1500adjGEadd_k30}}}}	
\subcaptionbox{\textsc{SPONGE}$_{sym}$  %  $k=30$
}{\includegraphics[width=0.24\columnwidth]{{{Figures/real_data_graphs_high_k/SP1500adjGEmul_k30}}}}	
\subcaptionbox{\textsc{BNC} % , $k=30$
}{\includegraphics[width=0.24\columnwidth]{{{Figures/real_data_graphs_high_k/SP1500adjBNC_k30}}}}	
\subcaptionbox{$\bar{L}_{sym}$  % , $k=30$
}{\includegraphics[width=0.24\columnwidth]{{{Figures/real_data_graphs_high_k/SP1500adjLbar_k30}}}}	
\par\end{centering}
\captionsetup{width=0.99\linewidth}
\caption{\footnotesize Adjacency matrix of the S\&P 1500 data, sorted by cluster membership;  % $k \in \{10,30\}$.
$k=10$ (top) and $k=30$ (bottom). 
% (a,e) $SPONGE$ (b,f) $SPONGE_{sym}$ (c,g) $BNC$ (d,h)  $\bar{L}_{sym}$. The first row has $k=10$ while the second has $k=30$.
}
\label{fig:SP1500_k_10_30}
\end{figure}
%\vspace{0mm}
%
% \begin{figure}[h]
% \vspace{.3in}
% \centerline{\includegraphics[width=.9\linewidth]{./Figures/sp1500_GE.png}}
% \vspace{.3in}
% \caption{$10$-clustering of S\&P 1500 data-set}
% \label{fig:spclus}
% \end{figure}
%
%  Adjacency matrix of the S\&P 1500 data set sorted by cluster membership, for $k=\{10,30\}$
%

\vspace{-1mm}
\begin{figure}[h]
% \vspace{.3in}
\centerline{\includegraphics[width=0.93\linewidth]{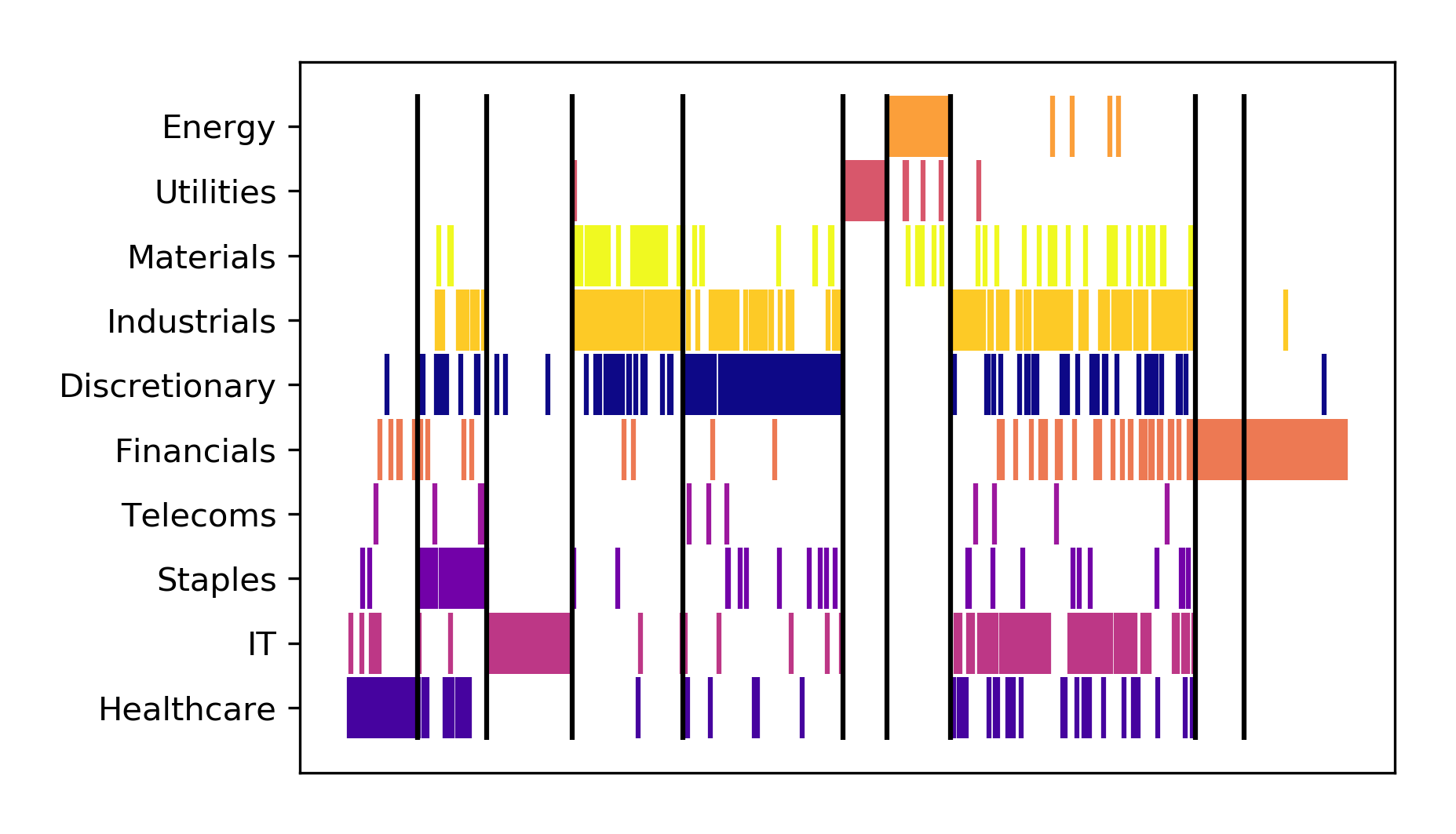}}
% \vspace{.3in}
\captionsetup{width=1.03\linewidth}
\caption{\footnotesize 
GICS  % sector
decomposition % for clusters of \textsc{SPONGE$_{sym}$}.}
for \textsc{SPONGE$_{sym}$} clusters. 
}
% GICS sector composition of clusters recovered by \textsc{SPONGE$_{sym}$}.} %\noteMC{Aldo+Peter: can you please confirm/correct this?} \notePD{It was SPONGE_{sym}}
\label{fig:spsec}
\end{figure}
\vspace{-2mm}

%%%---------------%%%---------------%%%---------------%%%---------------
%%%
%%% Moved to appendix - FX data set 
%%%
%%%---------------%%%---------------%%%---------------%%%---------------

\vspace{-1mm}
\paragraph{Correlations of Australian rainfalls.}
% Finally, 
We also consider time series % data 
of historical  rainfalls in locations throughout Australia. 
Edge weights are obtained 
% by % directly  calculating
from the pairwise Pearson correlation, % coefficient
% between each pair, %  of points, 
% obtaining 
leading to a complete signed graph on $n=306$ nodes.
\begin{wrapfigure}{l}{0.49\columnwidth}
\vspace{-3mm}
\centering
\includegraphics[width=0.52\columnwidth]{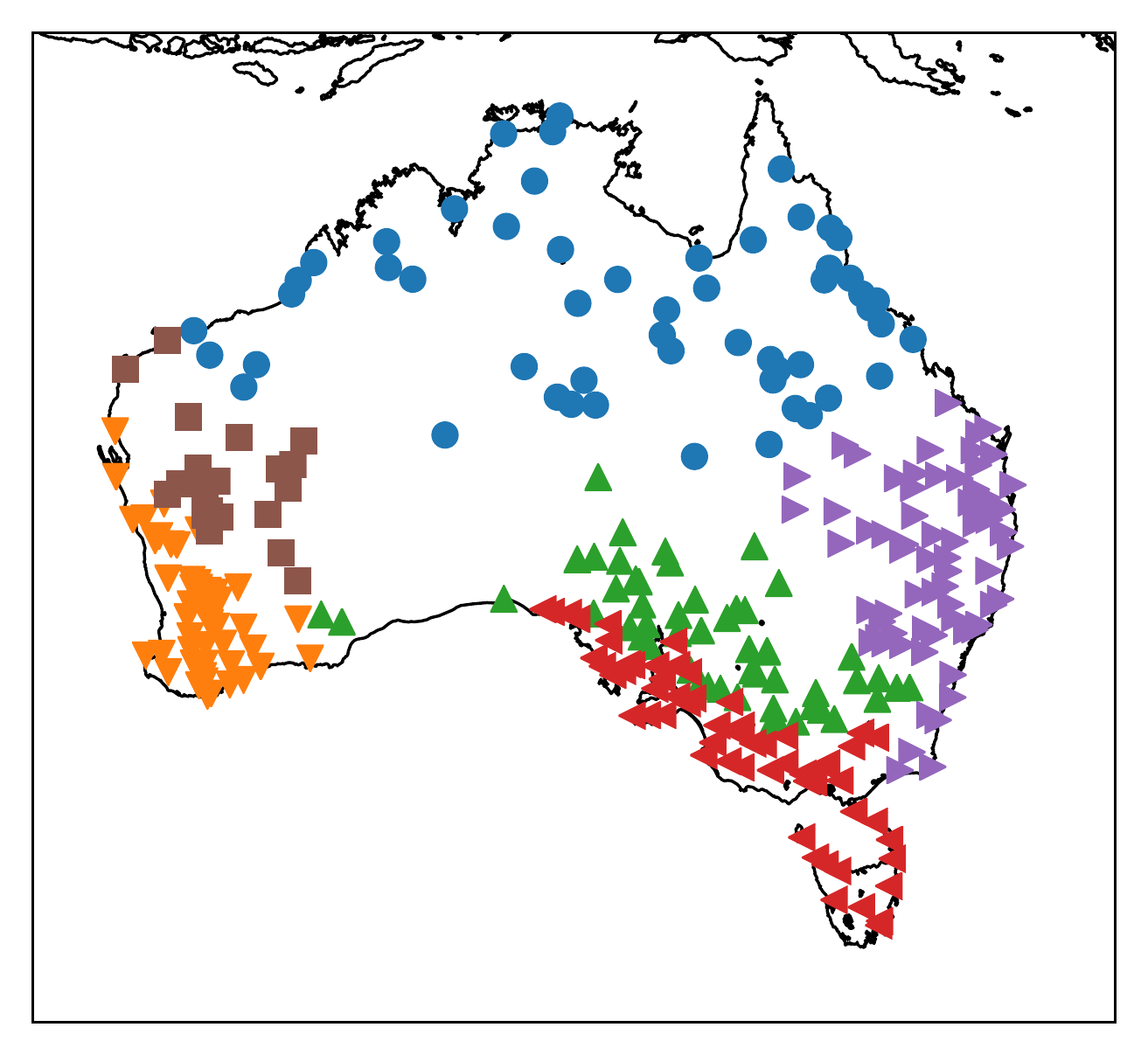} 
% \subcaption{  $\eta=0$ }
\caption{ \footnotesize
        % \textsc{SPONGE} clustering with $k = 6$ of the Australian rainfalls data set. %  clusters.
\textsc{SPONGE}: $k = 6$, Australian rainfalls data. %  set.
}
\label{fig:rainmap}
\end{wrapfigure}
\vspace{0mm}
Figure \ref{fig:rainclus} shows a clear clustering structure, for $k=\{6,10\}$ clusters, and Figure \ref{fig:rainmap} plots the points onto the corresponding geographic locations.
% at which the measurements were taken. 
% Indeed, the clustering algorithm 
\textsc{SPONGE} has very effectively identified geographic regions with similar climate, based only on the correlations of the rainfall measurements.

\vspace{-1mm}
\begin{figure}
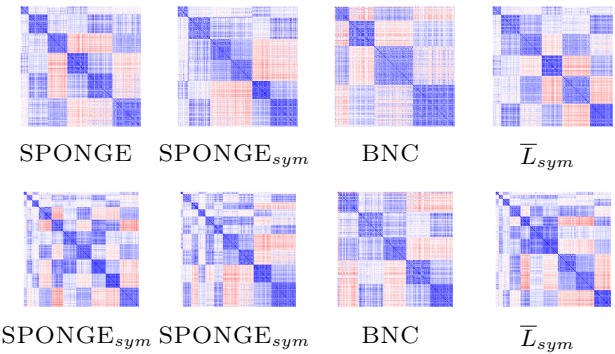

\captionsetup[subfigure]{labelformat=empty}
\begin{centering}
\subcaptionbox{ \textsc{SPONGE}%, $k = 6$
}{\includegraphics[width=0.24\columnwidth]{{{Figures/real_data_graphs_low_k/RAINadjGEadd_k6}}}}	
\subcaptionbox{\textsc{SPONGE}$_{sym}$%, $k = 6$
}{\includegraphics[width=0.24\columnwidth]{{{Figures/real_data_graphs_low_k/RAINadjGEmul_k6}}}}	
\subcaptionbox{\textsc{BNC}%, $k = 6$
}{\includegraphics[width=0.24\columnwidth]{{{Figures/real_data_graphs_low_k/RAINadjBNC_k6}}}}	
\subcaptionbox{$\bar{L}_{sym}$%, $k = 6$
}{\includegraphics[width=0.24\columnwidth]{{{Figures/real_data_graphs_low_k/RAINadjLap_k6}}}}	
\par\end{centering}
\begin{centering}
\subcaptionbox{ \textsc{SPONGE}$_{sym}$ %$k = 10$
}{\includegraphics[width=0.24\columnwidth]{{{Figures/real_data_graphs_high_k/RAINadjGEadd_k10}}}}	
\subcaptionbox{ \textsc{SPONGE}$_{sym}$ % $k = 10$
}{\includegraphics[width=0.24\columnwidth]{{{Figures/real_data_graphs_high_k/RAINadjGEmul_k10}}}}	
\subcaptionbox{\textsc{BNC}%, $k = 10$
}{\includegraphics[width=0.24\columnwidth]{{{Figures/real_data_graphs_high_k/RAINadjBNC_k10}}}}	
\subcaptionbox{$\bar{L}_{sym}$%, $k = 10$
}{\includegraphics[width=0.24\columnwidth]{{{Figures/real_data_graphs_high_k/RAINadjLbar_k10}}}}	
\par\end{centering}
\captionsetup{width=0.99\linewidth}
\caption{ \footnotesize  Sorted adjacency matrix of the Australian
rainfall data set, with  
% sorted by cluster membership,
$k=6$ (top) and $k=10$ (bottom).
% sorted by (a,e) $SPONGE$ (b,f) $SPONGE_{sym}$ (c,g) $BNC$ (d,h)  $\bar{L}_{sym}$. The first row has $k=6$ while the second has $k=10$.
}
\label{fig:rainclus}
\end{figure}
\vspace{-1mm}

%%% Conclusion

\vspace{-2mm}
% \section{Concluding remarks}  \label{sec:conclusion}
\section{ Discussion and future directions}   \label{sec:conclusion}
\vspace{-2mm}

%%% \noteMC{Perhaps leave out a few things :) ...}
We introduced a principled spectral algorithm (\textsc{SPONGE})  for clustering signed graphs, %  that stems from the relaxation of a hard combinatorial optimization problem and 
that amounts to solving a generalized eigenvalue problem,  and provided a theoretical analysis  % of \textsc{SPONGE} % of our algorithm for the case of 
for $k=2$ clusters. % , using tools from matrix perturbation theory.
% and random matrix theory. 
% We have augment our proposed methodology and theoretical findings with 
Extensive numerical experiments  % that
demonstrate its robustness to noise and sampling sparsity. In particular, for very sparse graphs and large number of clusters
$k$, we  are able to 
 recover clusterings when all state-of-the-art methods completely fail. 
% Hemant: shrink above

% Our work opens up a number of
There are several directions for future work such as:
(i) considering a more general SSBM that allows for different edge 
sampling probabilities and noise levels;
% for measurements within and across blocks; 
%
% (ii) obtaining robustness guarantees when augmenting the recovered embedding/eigenvectors with a k-means clustering step, for which theoretical results exist provided certain assumptions are met regarding  the separation between clusters and their structure; 
% perhaps leave out::
(ii) handling the challenging setting of very sparse graphs, % with $p$ on the order of $\frac{1}{n}$, by
where $p = \Theta(\frac{1}{n})$; 
% by leveraging recent regularization techniques \cite{joseph2013impactBinYu, le2015sparse_Vershynin}. 
% used in the challenging sparse regime of bounded expected degrees, for which the highly irregular distribution of node degrees render both the graph adjacency matrix and its Laplacian not to concentrate around their expectations; 
% 
% (iv) recovering $l_{\infty}$ eigenvector perturbation error bounds, using recent tools introduced in \cite{fan2016ell_infinity,abbe2017entrywiseFan}; 
%
(iii) exploring the usefulness of the \textsc{SPONGE} embedding 
% our approach 
as a dimensionality reduction tool in multivariate time series analysis;
% in particular, obtaining truthful low-dimensional representations of signed networks arising from high-dimensional financial data (e.g., correlation matrices), 
%and leveraging it for prediction and classification tasks;
% 
% (vi) Other potential directions pertain to 
(iv) exploring semidefinite programming based approaches, inspired by recent work % in the context of 
on community detection \cite{guedonVershynin2016community};  
% motif/graphlet-based approaches \cite{bensonGleichLeskovec2016higherScience}, 
and (v) investigating graph-based diffuse interface models utilizing the Ginzburg-Landau functionals, based on the MBO scheme \cite{signed_MBO_Yves,merkurjevBertozzi2013mbo}.  

\vspace{-3mm}
\paragraph{Acknowledgments.} We thank Sanjay Chawla, Ioannis Koutis, and Pedro Mercado for useful discussions. This work was funded by EPSRC grant EP/N510129/1 at The Alan Turing Institute.

% and corresponds to building non-linear risk models for financial instruments, a task of paramount importance when it comes to understanding the risks that govern a financial system.

\iffalse
(2) Signed clustering via Semidefinite Programming (SDP). This approach relies on a semidefinite programming-based formulation, inspired by recent work in the context of community detection in sparse networks. We efficiently solve the SDP program efficiently via a Burer-Monteiro approach, and extract clusters via minimal spanning tree-based clustering.

(3) An MBO scheme.  Another direction relates to graph-based diffuse interface models utilizing the Ginzburg-Landau functionals, based on an adaptation of the classic numerical Merriman-Bence-Osher (MBO) scheme for minimizing such graph-based functionals. The latter approach bears the advantage that it can easily incorporate labeled data, in the context of semi-supervised clustering.

% (5) Graph motifs. This approach relies on extending recent work on clustering the motif/graphlet adjacency matrix, as proposed recently in a Science paper by Benson, Gleich, and Leskovec.
\fi

%-----------------------
% Bibliography 
%-----------------------
\newpage
% todo: delete extra references 
% \bibliography{references,references_initial}
% \bibliography{bibliography,references}
% \bibliography{references}
% \bibliography{references_initial}
\bibliographystyle{plain}
\bibliography{AAA_references}

%\end{document}

%
%
%
%
% \section*{Appendix}
% \lipsum[1-2]

\appendix

\clearpage

\onecolumn

%----------------------------------
% Matrix perturbation results
%----------------------------------
\section{Matrix perturbation analysis} \label{app:sec_perturb_theory}
Let $A \in \mathbb{C}^{n \times n}$ be Hermitian with eigenvalues $\lambda_1 \geq \lambda_2 \geq \cdots \geq \lambda_n$ 
and corresponding eigenvectors $v_1,v_2,\dots,v_n \in \mathbb{C}^n$. 
Let $\widetilde{A} = A + W$ be a perturbed version of $A$, with the perturbation matrix 
$W \in \mathbb{C}^{n \times n}$ being Hermitian. Let us denote the eigenvalues of $\tilde{A}$ and $W$ by
$\tilde{\lambda}_1 \geq \cdots \geq \tilde{\lambda}_n$ and 
$\epsilon_1 \geq \epsilon_2 \geq \cdots \geq \epsilon_n$ respectively.

To begin with, we would like to quantify the perturbation of the eigenvalues of $\widetilde{A}$ with respect to the 
eigenvalues of $A$. Weyl's inequality \cite{Weyl1912} is a very useful result in this regard.
%
%----------------------
% Weyls inequality
%-----------------------
\begin{theorem} [Weyl's Inequality \cite{Weyl1912}] \label{thm:Weyl} 
For each $i = 1,\dots,n$, it holds that
\begin{equation}
 \lambda_i + \epsilon_n  \leq  \tilde{\lambda}_i \leq \lambda_i + \epsilon_1.
 \end{equation}
In particular, this implies that $\tilde{\lambda}_i \in [\lambda_i - \norm{W}_2, \lambda_i + \norm{W}_2]$.
\end{theorem} 
One can also quantify the perturbation of the subspace spanned by eigenvectors of $A$, this was established 
by Davis and Kahan \cite{daviskahan}. Before introducing the theorem, we need some definitions. 
Let $U,\widetilde{U} \in \mathbb{C}^{n \times k}$ (for $k \leq n$) have orthonormal columns respectively and 
let $\sigma_1 \geq \dots \geq \sigma_k$ denote the singular values of $U^{*}\widetilde{U}$. Also, 
let us denote $\calR(U)$ to be the range space of the columns of $U$, same for $\calR(\widetilde U)$. 
Then the $k$ principal angles between $\calR(U), \calR(\widetilde{U})$ are 
defined as $\theta_i := \cos^{-1}(\sigma_i)$ for $1 \leq i \leq k$, with each $\theta_i \in [0,\pi/2]$. 
It is usual to define $k \times k$ diagonal matrices 
$\Theta(\calR(U), \calR(\widetilde{U})) := \text{diag}(\theta_1,\dots,\theta_k)$ 
and $\sin \Theta(\calR(U), \calR(\widetilde{U})) := \text{diag}(\sin \theta_1,\dots,\sin \theta_k)$. 
Denoting $||| \cdot |||$ to be any unitarily invariant norm (Frobenius, spectral, etc.), 
the following relation holds (see for eg., \cite[Lemma 2.1]{li94}, \cite[Corollary I.5.4]{stewart1990matrix}).
\begin{equation*} 
|||  \sin \Theta(\calR(U), \calR(\widetilde{U}))  |||  =  ||| (I - \tilde{U}  \tilde{U}^{*} ) U |||.
\end{equation*}
With the above notation in mind, we now introduce a version of the Davis-Kahan theorem taken from \cite[Theorem 1]{dkuseful} 
(see also \cite[Theorem V.3.6]{stewart1990matrix}).
%
%-----------------------
% Davis Kahan theorem
%-----------------------
\begin{theorem}[Davis-Kahan] \label{thm:DavisKahan} 
Fix $1 \leq r \leq s \leq n$, let $d = s-r+1$, and let 
$U = (u_r,u_{r+1},\dots,u_s) \in \mathbb{C}^{n \times d}$ and 
$\widetilde{U} = (\widetilde{u}_r,\widetilde{u}_{r+1},\dots,\widetilde{u}_s) \in \mathbb{C}^{n \times d}$. Write
\begin{equation*}
 \delta = \inf\set{\abs{\hat\lambda - \lambda}: \lambda \in [\lambda_s,\lambda_r], \hat\lambda \in (-\infty,\widetilde\lambda_{s+1}] \cup [\widetilde \lambda_{r-1},\infty)}
\end{equation*}
where we define $\widetilde\lambda_0 = \infty$ and $\widetilde\lambda_{n+1} = -\infty$ and assume that $\delta > 0$. Then
\begin{equation*}
 ||| \sin \Theta(\calR(U), \calR(\widetilde{U}))|||  = ||| (I - \tilde{U}  \tilde{U}^{*} ) U ||| \leq  \frac{ ||| W ||| }{ \delta}.
 \end{equation*}
\end{theorem} 
For instance, if $r = s = j$, then by using the spectral norm $\norm{\cdot}_2$, we obtain 
\begin{equation} \label{eq:dk_useful}
\sin \Theta(\calR(\widetilde{v}_j), \calR(v_j)) = \norm{(I - v_j v_j^{*})\widetilde{v}_j}_2 \leq \frac{\norm{W}_2}{\min\set{\abs{\widetilde{\lambda}_{j-1}-\lambda_j},\abs{\widetilde\lambda_{j+1}-\lambda_j}}}.
\end{equation}
%

%------------------------------------
% Useful concentration inequalities
%------------------------------------
\section{Useful concentration inequalities}
\subsection{Chernoff bounds} \label{app:subsec_chernoff_bern}
Recall the following Chernoff bound for sums of independent Bernoulli random variables.
\begin{theorem}[{\cite[Corollary 4.6]{upfal05}}] \label{thm:chernoff_bern}
Let $X_1,\dots,X_n$ be independent Bernoulli random variables with $\prob(X_i = 1) = p_i$. Let $X =\sum_{i=1}^n X_i$ and $\mu = \expec[X]$. For $\delta \in (0,1)$, 
\begin{equation*}
\prob(\abs{X - \mu} \geq \delta \mu) \leq 2 \exp(-\mu \delta^2 / 3).
\end{equation*}
\end{theorem}
%

%-------------------------------------------
% Hoeffding type concentration inequality
%--------------------------------------------
%
%
%\begin{proposition} \cite[Proposition 5.10]{vershynin2012} \label{prop:hoeff_subgauss_conc}
%Let $X_1,\dots,X_n$ be independent  centered sub-Gaussian random variables 
%and let $K = \max_i \norm{X_i}_{\psi_2}$. Then for every $\veca \in \matR^n$, and 
%every $t \geq 0$, we have
%%
%%
%\begin{equation}
%\prob(\abs{\sum_{i=1}^n a_i X_i} \geq t) 
%\leq e \cdot \exp\left(-\frac{c^{\prime} t^2}{K^2 \norm{\veca}_2^2}\right), 
%\end{equation}
%%
%where $c^{\prime} > 0$ is an absolute constant.
%\end{proposition}

%-----------------------------------
% Spectral norm of random matrices
%------------------------------------
\subsection{Spectral norm of random matrices} \label{app:subsec_spec_rand_mat}
We will make use of the following result for bounding the spectral norm of 
symmetric matrices with independent, centered and bounded random variables.
\begin{theorem}[{\cite[Corollary 3.12, Remark 3.13]{bandeira2016}}] \label{app:thm_symm_rand}
Let $X$ be an $n \times n$ symmetric matrix whose entries $X_{ij}$ $(i \leq j)$ are 
independent, centered random variables. There there exists for any $0 < \varepsilon \leq 1/2$ 
a universal constant $c_{\varepsilon}$ such that for every $t \geq 0$, 
\begin{equation} \label{eq:afonso_conc}
\prob(\norm{X}_2 \geq (1+\varepsilon) 2\sqrt{2}\tilde{\sigma} + t) 
\leq n\exp\left(-\frac{t^2}{c_{\varepsilon}\tilde{\sigma}_{*}^2} \right)
\end{equation}
where
\begin{equation*} 
\tilde{\sigma}:= \max_{i} \sqrt{\sum_{j} \expec[X_{ij}^2]}, 
\quad \tilde{\sigma}_{*}:= \max_{i,j} \norm{X_{ij}}_{\infty}.
\end{equation*}
\end{theorem}
Note that it suffices to employ upper bound estimates on $\tilde{\sigma},\tilde{\sigma}_{*}$ in
\eqref{eq:afonso_conc}. Indeed, if $\tilde{\sigma} \leq \tilde{\sigma}^{(u)}$ and $\tilde{\sigma}_{*} \leq \tilde{\sigma}_{*}^{(u)}$, then
\begin{equation*}
    \prob(\norm{X}_2 \geq (1+\varepsilon) 2\sqrt{2}\tilde{\sigma}^{(u)} + t) \leq \prob(\norm{X}_2 \geq (1+\varepsilon) 2\sqrt{2}\tilde{\sigma} + t) \leq n\exp\left(-\frac{t^2}{c_{\varepsilon}{\tilde{\sigma}_{*}}^2} \right) 
    \leq n\exp\left(-\frac{t^2}{c_{\varepsilon} (\tilde{\sigma}_{*}^{(u)})^2} \right).  
\end{equation*}
%--------------------------------------------------------------------
% Definition of Adjacency matrix entries (random variables) in SSBM
%--------------------------------------------------------------------
%
%
\section{Signed stochastic block model (SSBM)} \label{appsec:ssbm_defs}
Let $A \in \set{0, \pm 1}^{n \times n}$ denote the adjacency matrix of $G$, with  
$A_{ii} = 0$, and $A_{ij} = A_{ji}$. Under the SSBM, we observe for each $i < j$ that

\vspace{2mm}
\begin{minipage}{0.49\linewidth}  
$ \quad  \quad  \quad$ \underline{$i,j$ lie in \textbf{same} cluster}
\begin{equation}
 A_{ij}= \left\{
\begin{array}{rl}
1 \quad ; & \text{w. p } \quad p (1-\eta)  \\
-1\quad ; & \text{w. p }  \quad p\eta \\
0 \quad ; & \text{w. p } \quad (1- p)
\end{array} \right.
\label{eq:defA_same}
\end{equation}
%\vspace{-2mm}
%
\end{minipage}  \hspace{0.2cm}  
\begin{minipage}{0.49\linewidth}  
$ \quad  \quad \quad$  \underline{$i,j$ lie in \textbf{different} clusters}
\begin{equation}
  A_{ij}= \left\{
\begin{array}{rl}
 1  \quad ; & \text{w. p } \quad p  \eta  \\
-1 \quad ; & \text{w. p }  \quad p (1-\eta) \\
 0  \quad ; & \text{w. p } \quad (1- p)
\end{array} \right. .
\label{eq:defA_different}
\end{equation}
\end{minipage}
\vspace{2mm}

In particular, $(A_{ij})_{i \leq j}$ are independent random variables. 
Next, we recall that $A$ can be decomposed as
%
%\vspace{-5mm}
\begin{equation}\label{eq:defA}
	A = A^+ - A^-,
\end{equation}
% \vspace{-1mm}
where $ A^+, A^- \in \set{0,1}^{n \times n}$ are the adjacency matrices of the unsigned graphs 
$\Gp,\Gn$ respectively. For any given $ i < j$, we have 
% Here $A^+, A^- \in \set{0,1}^{n \times n}$ are the (unsigned) symmetric adjacency matrices corresponding to the respective  ``positive'' and ``negative'' subgraphs with $A_{ij}^+ = \max\set{A_{ij},0}$ and $A_{ij}^- = \max\set{-A_{ij},0}$.
% Hence $A^+, A^- $ are symmetric matrices. 
% For each $ i < j$  % we have that

%\vspace{-1mm}
%----------------
% same clusters
%----------------
\begin{minipage}{0.49\linewidth}  
$ \quad  \quad  \quad $ \underline{$i,j$ lie in \textbf{same} cluster}
%\vspace{-2mm}
\begin{equation}
 A_{ij}^+= \left\{
\begin{array}{rl}
1 \quad ; & \text{w. p } \quad p (1-\eta)  \\
0 \quad ; & \text{w. p } \quad 1 - p (1-\eta)
\end{array} \right. ,
\label{eq:Aijp_same}
\end{equation}
%\vspace{-4mm}

%\vspace{-3mm}
% \noindent
\begin{equation}
%\hspace{-16mm}
 A_{ij}^-= \left\{
\begin{array}{rl}
1 \quad ; & \text{w. p } \quad p \eta  \\
0 \quad ; & \text{w. p } \quad 1 - p \eta
\end{array} \right. ,
\label{eq:Aijn_same}
\end{equation}
%\vspace{-3mm}

\end{minipage}  \hspace{0.2cm}  
%-----------------------
%  different clusters
%-----------------------
\begin{minipage}{0.49\linewidth}  
$ \quad  \quad \quad$ \underline{$i,j$ lie in \textbf{different} clusters}
%\vspace{-2mm}
\begin{equation}
%\hspace{-16mm}
  A_{ij}^+= \left\{
\begin{array}{rl}
 1  \quad ; & \text{w. p } \quad p  \eta  \\
0 \quad ; & \text{w. p }  \quad 1 - p  \eta 
\end{array} \right. ,
\label{eq:Aijp_different}
\end{equation}
%\vspace{-3mm}
%
%\vspace{-3mm}
\begin{equation}
  A_{ij}^-= \left\{
\begin{array}{rl}
 1  \quad ; & \text{w. p } \quad p  (1-\eta)  \\
0 \quad ; & \text{w. p }  \quad 1 -p  (1-\eta)  
\end{array} \right. .
\label{eq:Aijn_different}
\end{equation}
\end{minipage} 
\vspace{2mm}

Since $\Ap_{ij} = \max\set{A_{ij},0}$, therefore $ (A_{ij}^+)_{i \leq j} $ are independent 
random variables. Similarly $ (A_{ij}^-)_{i \leq j}$ are also independent. 
But clearly, for given $i,j \in [n]$ with $i \neq j$, the entries 
$A_{ij}^+$ and  $A_{ij}^-$ are dependent random variables. 
%
%

%----------------------------------------------------
% Proof of theorem for SPONGE, k = 2, top k vectors
%-----------------------------------------------------
\section{Proof of Theorem \ref{thm:sponge_k_2_bot_k_short}}
We will prove the following more precise version of Theorem \ref{thm:sponge_k_2_bot_k_short} in this section.
%
%---------------------------------
% Main theorem : bottom k eigvecs
%---------------------------------
\begin{theorem} \label{thm:sponge_k_2_bot_k}
Assuming $\eta \in [0, 1/2)$ let $ \tau^+, \tau^- > 0$ satisfy 
$\tau^- <  \tau^+  \Big( \frac{ \frac{n}{2} -1 + \eta  }{ \frac{n}{2} - \eta } \Big)$. 
Then it holds that $\set{v_{n-1}(\Tbar), v_{n}(\Tbar)} = \set{\ones,\infovec}$ where $\infovec$ is defined 
in \eqref{eq:k_2_inform_vect}. 
Moreover, assuming $n \geq 6 $, for given $0 < \varepsilon \leq 1/2$, $ \epsilon \in (0,1) $ and $\varepsilon_{\tau} \in (0,1)$ let $\tau^- \leq \varepsilon_{\tau}  \tau^+  \Big( \frac{ \frac{n}{2} -1 + \eta  }{ \frac{n}{2} - \eta } \Big)$ and 
\begin{align*}
p & \geq   
 \max  \bigg\{ 
24,  
\frac{36 \tce^2}{ (\tau^+)^2},   
\frac{36 \tce^2}{ (\tau^-)^2}, 
  % & \quad 
\Big(  \frac{ \bar{c}(\varepsilon, \tau^+, \tau^-)    }{  \epsilon   \min  \Big\{   \frac{2}{3} \frac{ (1-\varepsilon_{\tau}) } { (1 + \tau^+  ) },  \frac{(1-2 \eta)}{3}   \frac{ (3 + \tau^+ + \tau^- ) }{ (1+\tau^+)^2 }    \Big\}   }    \Big)^4 
\bigg\}     
\Big( \frac{\log n}{n} \Big), 
\end{align*}
where $ \tce	= (1 + \varepsilon) 2 \sqrt{2} +1  +\sqrt{3}$, and 
\begin{align*}
\bar{c}(\varepsilon, \tau^+, \tau^-)
& = \frac{ 3^{3/2} \sqrt{2} \; \tce^{1/2} (1+\tau^-) } {  (\tau^+)^{3/2}   } + 
\frac{3 \tce }{ \tau^+ }  % \\  &
+  \frac{6^{3/2} \; \tce^{3/2} }{ (\tau^+)^{3/2} }  
  + \frac{18 \; \tce^{2} }{ (\tau^+)^{2} } + \frac{9 \; \tce (1+\tau^-) }{ (\tau^+)^{2} }.   
\end{align*}
Then for $c_{\varepsilon} > 0$  depending only on $\varepsilon$, 
%and with 
% 
%\begin{equation*}
% c( \eta, p) = \sqrt{ \frac{p}{2} \Big[   (1-\eta)    \big[  1 - p (1-\eta) \big]  +   \eta  (1 - p \eta )  \Big]  },  
% \end{equation*}
%
it holds  with probability at least 
$\Big( 1 - \frac{4}{n} - 2n \exp{ \big( \frac{ - p n }{c_{\varepsilon}} \big) }  \Big)$ that

\vspace{-3mm}
\begin{equation*}
	\norm{( I - V_2(T) V_2(T)^T) V_2(\bar{T})}_2  \leq  \frac{ \epsilon }{ 1 - \epsilon}. 
\end{equation*}
\end{theorem}
%
%
%-----------------------
% Proof of theorem
%-----------------------
The proof is outlined in the following steps.
%
%------------------------------------------------------------------
% Step 1 : Analysis of the spectra of $\ELn$, $ \ELp$, $EDn$ and $\EDp$
%------------------------------------------------------------------
\subsection{Step 1: Analysis of the spectra of $ \ELn$, $ \ELp$, $\EDn$ and $\EDp$}
%
%
%-----------------------------------------------------------------------------------
% Lemma for Step 1: Analyzing the spectra of $ \ELn$, $ \ELp$, $\EDn$ and $\EDp$
%------------------------------------------------------------------------------------
\begin{lemma} \label{lem:expecs_posneg_mats} 
With $\infovec$ as defined in \eqref{eq:k_2_inform_vect}, the following holds true regarding the spectra of $ \ELp$ and $\EDp$.
\begin{enumerate}
\item $\EDp = d^+ I = p \left( \frac{n}{2} -1 + \eta \right) I$.
\item $\lambda_n^+ = \lambda_n(\ELp) = 0$,  $v_n^+ = v_n(\ELp) = \frac{1}{\sqrt{n}}\mb{1}$.
\item $\lambda_{n-1}^+ = \lambda_{n-1}(\ELp) = p n \eta$,  $v_{n-1}^+ = v_{n-1}(\ELp) = \infovec$. 
\item $\lambda_{l}^+ =  \lambda_{l}(\ELp) = \frac{n}{2} p$, $\forall l = 1, \ldots, n-2$.
\end{enumerate}
%
%%%%%%%%%%%%%%%%%%%%%%%%%%%%%%%%
%%%%%%%%%%%%%%%%%%%%%%%%%%%%%%%%
Similarly, the following holds for the spectra of $\ELn$ and $\EDn$. 
\begin{enumerate}
\item $\EDn = d^- I = p \left( \frac{n}{2} - \eta \right) I$.
\item $\lambda_n^- = \lambda_n(\ELn) = 0,  \quad  v_n^- = v_n(\ELn) = \frac{1}{\sqrt{n}}\mb{1}$.
\item $\lambda_{1}^- = \lambda_{1}(\ELn) = p n (1-\eta)$, $v_{1}^- = v_1(\ELn) = \infovec$.
\item $\lambda_{l}^- =  \lambda_{l}(\ELn) = \frac{n}{2} p$,   $\forall l = 2, \ldots, n-1$. 
\end{enumerate}
\end{lemma}
Before going to the proof, we can see from Lemma \ref{lem:expecs_posneg_mats} that 
$\ELp$ and  $\ELn$ have the same eigenspaces. In particular, the following decomposition holds true
\begin{equation} \label{eq:ELp_spec_decomp}
\ELp=  \left[
\begin{array}{c|c|c}
& & \\
& & \\
\underbrace{v_n^+}_{ \mb{1}}  & \underbrace{v_{n-1}^+ }_{ w}  & \tilde{V}_{n \times (n-2)}   \\
& & \\
& &
\end{array} 
\right]   
\left[ \begin{array}{cccc}
\lambda_n^+ &  &   &  \\
 & \lambda_{n-1}^+ &   &  \\
 &  & \ddots &   \\
 &  &  &      \\
\end{array} \right]
\begin{bmatrix}
          (v_n^+)^T \\
          \\
          (v_{n-1}^+)^T  \\
                    \\
           \tilde{V}^T
         \end{bmatrix}
          = U \Lambda^+ U^T
\end{equation}
\begin{equation} \label{eq:ELn_spec_decomp}
\ELn=  \left[
\begin{array}{c|c|c}
& & \\
& & \\
\underbrace{v_n^-}_{ \mb{1}}  &  \underbrace{v_{1}^-}_{w}  & \tilde{V}_{n \times (n-2)}   \\
& & \\
& &
\end{array} 
\right]   
\left[ \begin{array}{cccc}
\lambda_n^- &  &   &  \\
 & \lambda_{1}^- &   &  \\
 &  & \ddots &   \\
 &  &  &      \\
\end{array} \right]
\begin{bmatrix}
          (v_n^-)^T \\
          \\
          (v_{1}^-)^T  \\
            \\
           \tilde{V}^T
         \end{bmatrix}
          = U \Lambda^- U^T.
\end{equation}
%
%
%----------
% Proof
%---------- 
\begin{proof}
To begin with, let us note that for every $i \neq j$,
\begin{equation*}
\expec{ [ A^+_{ij} ] } = \left\{
\begin{array}{rl}
p (1-\eta)  \quad ; & \text{if } \quad i,j  \text{ lie in same cluster }  \\
p \eta  \quad ; & \text{if } \quad i,j  \text{ lie in different clusters} 
\end{array} \right.,    
\end{equation*}
% % % 
\begin{equation*}
\text{ and } \quad \expec{ [ A^-_{ij} ] } = \left\{
\begin{array}{rl}
p \eta  \quad ; & \text{if } \quad i,j  \text{ lie in same cluster }  \\
p ( 1 - \eta) \quad ; & \text{if } \quad i,j   \text{ lie in different clusters} 
\end{array} \right. .
\end{equation*}
This leads to the following block structure for the matrices $\EAp, \EAn$.
%
%\begin{equation}
%\EAp =
%\left[
%\begin{array}{c|c}
%p (1-\eta) & p \eta   \\  
%\hline
%\underbrace{  p \eta }_{n/2} &  \underbrace{  p (1-\eta)}_{ n/2}
%\end{array} 
%\right] 
%\end{equation}
% \right\}\text{Segment a}\\
%- choose either representation/drawing-
\[
\EAp  =  
\begin{tikzpicture}[mymatrixenv]
   \matrix[mymatrix] (m)  {
        p (1-\eta) \ones\ones^T & p \eta \ones\ones^T  \\
        p \eta \ones\ones^T  &   p (1-\eta) \ones\ones^T \\
    };
    \mymatrixbraceright{1}{1}{$n/2$}
    \mymatrixbracetop{2}{2}{$n/2$}
\end{tikzpicture} - p(1 - \eta) I = M^+ - p(1 - \eta) I,
\]
and similarly 
\begin{equation*}
\EAn =
\left[
\begin{array}{c|c}
 p \eta \ones\ones^T & p (1-\eta)  \ones\ones^T   \\  
\hline
p (1-\eta) \ones\ones^T  &    p \eta \ones\ones^T 
\end{array} 
\right] - p\eta I = M^- - p\eta I.
\end{equation*} 
We can observe that both $ M^+  $ and $M^-  $ are rank-2 matrices.

%-------------------------
% Computing $\EDp, \EDn$.
%--------------------------
\paragraph{Computing $\EDp, \EDn$.}
It can be easily verified that
\begin{equation*}
	\EAp \mb{1} = \left\{   \frac{n}{2} \left[  p (1-\eta) + p \eta ) \right]  - p ( 1 - \eta )  \right\}   \mb{1} = p \Big[ \frac{n}{2} - 1 + \eta \Big]     \; \mb{1}
\end{equation*}
% %  
% %  \left\{   \right\}
\begin{equation*}
\text{ and so, } \quad \EDp = p \Big( \frac{n}{2} - 1 + \eta \Big) I.
\end{equation*}
Similarly, one can also verify that
\begin{equation*} 
	\EAn \mb{1} = \left\{ \frac{n}{2} \left[  p \eta + p (1-\eta)  \right]  - p  \eta   \right\}  
	\mb{1} = p \Big[ \frac{n}{2} - \eta \Big]   \mb{1} 
\end{equation*}
\begin{equation*}
\text{ and so, } \EDn = p \Big( \frac{n}{2} -  \eta \Big) I.
\end{equation*}

%----------------------------
% Spectra of $\EAp, \EAn$.
%----------------------------
\paragraph{Spectra of $\EAp, \EAn$.}
From the preceding calculations, we easily see that 
\begin{equation*}
	\EAp \mb{1} = p \Big[ \frac{n}{2} - 1 + \eta \Big]  \ones \quad   \Rightarrow   \quad 
		\lambda_1(\EAp) = p \Big[ \frac{n}{2} - 1 + \eta \Big], \quad   \quad   v_1(\EAp) = \frac{1}{\sqrt{n}}\ones.
\end{equation*}
Recall that the informative vector is defined as 
$\infovec := \frac{1}{\sqrt{n}}( \underbrace{ 1, \ldots, 1}_{n/2}, \;\; \underbrace{-1, \ldots, -1}_{n/2})^T.$ Clearly, 
\begin{equation*}
 	M^+ \; \; \infovec=  \left[
\begin{array}{c|c}
p (1-\eta) \ones\ones^T & p \eta \ones\ones^T  \\  
\hline
  p \eta \ones\ones^T &  p (1-\eta) \ones\ones^T 
\end{array} 
\right]   
 \infovec
= \frac{n}{2} \Big[ p (1 - 2 \eta) \Big] \infovec.      
\end{equation*}
Therefore 
\begin{align*}
	\EAp  \infovec  
	& = \bigg(  \frac{n}{2}  \Big[ p (1-\eta) - p \eta \Big] - p (1-\eta) \bigg)  \infovec   \\
	& = \bigg(  \frac{n}{2}   p (1-2\eta) - p (1-\eta) \bigg)  \infovec 	 \\
	& = \underbrace{ p \bigg(  \frac{n}{2} (1-2\eta) - (1-\eta) \bigg) }_{ \lambda_2(\EAp) }   \underbrace{\infovec  }_{v_2(\EAp)}
\end{align*}
with $ \lambda_1(\EAp) \geq \lambda_2(\EAp) >0 $. Since $M^{+}$ is rank $2$, therefore  
$\lambda_3(\EAp) =  \ldots =  \lambda_n(\EAp)  = - p ( 1 - \eta)$.

Next, we repeat the above same procedure for $A^-$. Firstly, 
\begin{equation*}
	\EAn \ones = p \Big( \frac{n}{2} - \eta  \Big)    \ones   
	\quad   \Rightarrow   \quad   \lambda_1(\EAn) = p \Big( \frac{n}{2} -  \eta   \Big), \quad   \quad   v_1(\EAn) = \frac{1}{\sqrt{n}}\ones.
\end{equation*}
Moreover, 
\begin{align*}
	M^-  \infovec 
=  \left[
\begin{array}{c|c}
	p \eta \ones\ones^T 		& 		p (1-\eta) \ones\ones^T  \\  
\hline 
	p (1-\eta) \ones\ones^T &  p  \eta \ones\ones^T
\end{array} 
\right]  \infovec  
=  \frac{n}{2} [ p \eta - p ( 1- \eta ) ]  \infovec  
=  \frac{n}{2}  p (2 \eta -  1)  \infovec
\end{align*}
which leads to   
\begin{align*}
	\EAn \infovec 
	& = \Big[  \frac{n}{2} p (2 \eta - 1)  - p \eta  \Big]   \; \infovec   
	=  p \Big[  \frac{n}{2} (2 \eta - 1)  -  \eta   \Big]  \infovec \\ \Rightarrow  
	\lambda_n(\EAn) &= p \Big[  \underbrace{ \frac{n}{2}  ( 2 \eta - 1 )  -\eta}_{<0}  \Big],  
				\quad      v_n(\EAn) = \infovec.
\end{align*}
Since $M^{-}$ is rank $2$, hence 
\begin{equation*}
\lambda_2(\EAn) =  \ldots =  \lambda_{n-1}(\EAn)  = - p  \eta.
\end{equation*}

%----------------------------
% Spectra of $\ELp, \ELn$.
%----------------------------
\paragraph{Spectra of $\ELp, \ELn$.}
Since $\ELp = \EDp  - \EAp  =  \big(  \frac{n}{2} - \eta +1 \big) p I - \EAp$,  
therefore the smallest two largest eigenvalue of $\ELp$ are given by 
\begin{align*}
  \lambda_n(\ELp) &= p \big(  \frac{n}{2}  -1 + \eta  \big)  - p \big(  \frac{n}{2}  -1 + \eta  \big)  = 0, \\
  \lambda_{n-1}(\ELp) &= p \big(  \frac{n}{2}  -1 + \eta  \big)  - p \big[  \frac{n}{2}(1- 2 \eta) - (1-\eta) \big]  = n p \eta.
\end{align*} 
For $1 \leq l \leq n-2$, the remaining eigenvalues are given by
\begin{equation*}
  \lambda_{l}(\ELp) = p \Big(  \frac{n}{2}  -1 + \eta  \Big)  + p (1-\eta)  = \frac{n}{2} p.
\end{equation*}
Note that, since $\eta < \frac{1}{2}$, it holds true that
$\lambda_{l}(\ELp) > \lambda_{n-1}(\ELp), \quad \forall 1 \leq l \leq n-2.$
Also note that the eigenvectors of $\ELp$ are the same as for $\EAp$.
 
Repeating the process for $ \ELn $ using $\ELn = p \big(  \frac{n}{2}  - \eta \big)I  - \EAn$,
we obtain
\begin{align*}
	\lambda_{n}(\ELn) &= 0 = p \Big(  \frac{n}{2} - \eta  \Big)  - p \Big(  \frac{n}{2} - \eta \Big), \\
	\lambda_{1}(\ELn) &= p \Big( \frac{n}{2}  - \eta  \Big)  - p  \Big[   \frac{n}{2} (2 \eta - 1) - \eta \Big]  = n p (1 - \eta), \\
	\lambda_{l}(\ELn) &= p \Big( \frac{n}{2}  - \eta  \Big)  + p  \eta = \frac{np}{2}    \quad \Big( < \lambda_1(\ELn)\Big), \quad \forall l = 2, \ldots, n-1.
\end{align*}
\end{proof}

%--------------------------------------------
% Step 2: Analyzing the spectra of $\Tbar$
%--------------------------------------------
\subsection{Step 2: Analyzing the spectra of $\Tbar$} 
%
%
%--------------------------------
% Main lemma for Step 2
%--------------------------------
\begin{lemma} \label{lem:sponge_Tbar_spectra}
Let $\tau^+, \tau^- >0$ satisfy 
$\tau^-  <  \tau^+  \bigg( \frac{  \frac{n}{2} - 1 + \eta }{ \frac{n}{2} - \eta  } \bigg)$ 
and let $\eta < \frac{1}{2}$.  Then the following is true.
\begin{enumerate}
\item $\Big\{  \lambda_{n}(\Tbar),   \lambda_{n-1}(\Tbar)  \Big\}
 =  \Bigg\{ \frac{ \tau^- (\frac{n}{2} - \eta)  }{ \tau^+ ( \frac{n}{2} - 1 + \eta ) },   \frac{ n \eta +  \tau^- (\frac{n}{2} - \eta)  }{ n (1-\eta) + \tau^+ ( \frac{n}{2} - 1 + \eta ) } \Big\}$ and
$\lambda_{l}(\Tbar) = \frac{ n + 2 \tau^- ( \frac{n}{2} - \eta )}{ n + 2 \tau^+ ( \frac{n}{2} - 1 +  \eta ) },$ for $l=1,\dots,n-2$.
%---------------------

\item $\set{ v_n(\Tbar), v_{n-1}(\Tbar)} = \set{\frac{1}{\sqrt{n}}\ones, \infovec} $.
\end{enumerate}
Moreover, if $ n \geq 6$ and $ \tau^- \leq \varepsilon_{\tau}  \tau^+ \Big(  \frac{ \frac{n}{2}-1+\eta  }{\frac{n}{2}-\eta}  \Big),$ 
then the \textit{spectral gap} ($\lambda_{gap}$) between $\Big\{\lambda_{n}(\Tbar),   \lambda_{n-1}(\Tbar)  \Big\}$ 
and $\lambda_i(\Tbar) (i=n-2,\ldots,1)$ satisfies 
\begin{equation*}
\lambda_{gap} =  \lambda_{n-2}(\Tbar) -  \lambda_{n-1}(\Tbar) \geq \min 
\Big\{    \frac{2 (1 - \varepsilon_{\tau})}{ 3(1+\tau^+) },  \;\;   
\frac{ (1 - 2 \eta) }{ 3 }  \frac{ (3 + \tau^+ + \tau^-) }{ ( 1 + \tau^+ )^2 }   \Big\}.
\end{equation*}
\end{lemma}
%
%
%-----------
% Proof
%-----------
\begin{proof}
Using \eqref{eq:ELp_spec_decomp}, \eqref{eq:ELn_spec_decomp} and Lemma \ref{lem:expecs_posneg_mats}, we can write $\Tbar$ as
\begin{align*}
\Tbar & = (\ELn + \tau^+ \EDp )^{-1/2}   (\ELp  + \tau^{-} \EDn )  (\ELn + \tau^+ \EDp )^{-1/2} \\
   & = (U \Lambda^- U^T + \tau^+ d^+ I )^{-1/2}  (U \Lambda^+ U^T + \tau^- d^- I )  (U \Lambda^- U^T + \tau^+ d^+ I )^{-1/2}  \\
   & = U \underbrace{ (\Lambda^- + \tau^+ d^+ I)^{-1} ( \Lambda^+ + \tau^- d^- I )}_{\Lambda_{\Tbar}} U^T .
\end{align*}
$\Lambda_{\Tbar}$ has at most three distinct values which we denote as  
\begin{align*}
\underbrace{\Lambda_{\Tbar}^{(1)}}_{ \text{eigenvector } \frac{1}{\sqrt{n}}\ones} 
&= \frac{ \lambda_n^+ + \tau^- d^-  }{ \lambda_n^- + \tau^+ d^+ } 
= \frac{ \tau^- (\frac{n}{2} - \eta)  }{ \tau^+ ( \frac{n}{2} - 1 + \eta ) },  \\
\underbrace{ \Lambda_{\Tbar}^{(2)} }_{ \text{ eigenvector }  \infovec} 
&= \frac{ \lambda_{n-1}^+ + \tau^- d^-  }{ \lambda_1^- + \tau^+ d^+ } 
= \frac{ n \eta +  \tau^- (\frac{n}{2} - \eta)  }{ n (1-\eta) + \tau^+ ( \frac{n}{2} - 1 + \eta ) },	\\
\Lambda_{\Tbar}^{(3)} 
&= \frac{ \lambda_l^+ + \tau^- d^-  }{  \lambda_{l'}^- + \tau^+ d^+  } 
= \frac{ n + 2 \tau^- ( \frac{n}{2} - \eta )}{ n + 2 \tau^+ ( \frac{n}{2} - 1 +  \eta ) } 
\quad \text{ (for any $1 \leq l \leq n-2$, $2 \leq l' \leq n-1$)}. 
\end{align*}
We would like to ensure that $\Lambda_{\Tbar}^{(1)}, 	\Lambda_{\Tbar}^{(2)} < 	\Lambda_{\Tbar}^{(3)}$ holds 
in order to obtain the right embedding. To this end, we consider next the following cases.
\begin{enumerate}[nolistsep]
\item   $\mathbf{ \Lambda_{\Tbar}^{(1)} > \Lambda_{\Tbar}^{(2)} }.$ This is equivalent to

\begin{equation}
  \frac{ \tau^- (\frac{n}{2} - \eta)  }{ \tau^+ ( \frac{n}{2} - 1 + \eta ) } 
 >    \frac{  n \eta +  \tau^- (\frac{n}{2} - \eta)  }{ n (1-\eta) + \tau^+ ( \frac{n}{2} - 1 + \eta ) } 
\Longleftrightarrow 
 \tau^-  > \tau^+  \bigg( \frac{ \eta  ( \frac{n}{2} - 1 + \eta )   }{  (1-\eta) ( \frac{n}{2} - \eta )  } \bigg).
 \label{Stage2_a1}
\end{equation}

\item   $ \mathbf{\Lambda_{\Tbar}^{(1)} < \Lambda_{\Tbar}^{(3)} }.$ This is equivalent to 
\begin{equation}
  \frac{ \tau^- (\frac{n}{2} - \eta)  }{ \tau^+ ( \frac{n}{2} - 1 + \eta ) } 
 <    \frac{  n + 2 \tau^- (\frac{n}{2} - \eta)  }{ n  + 2 \tau^+ ( \frac{n}{2} - 1 + \eta ) } 
\Longleftrightarrow 
 \tau^-  <  \tau^+  \bigg( \frac{  \frac{n}{2} - 1 + \eta }{ \frac{n}{2} - \eta  }   \bigg).
 \label{Stage2_b2}
\end{equation}

\item   $ \mathbf{\Lambda_{\Tbar}^{(2)}  <  \Lambda_{\Tbar}^{(3)}}.$ This is equivalent to
\begin{align}
 &\frac{  n \eta +  \tau^- (\frac{n}{2} - \eta)  }{ n (1-\eta) + \tau^+ ( \frac{n}{2} - 1 + \eta ) } < \frac{ n + 2 \tau^- (\frac{n}{2} - \eta)  }{   n  + 2 \tau^+ ( \frac{n}{2} - 1 + \eta )  } \nonumber \\
& \Leftrightarrow
 n^2 \eta + 2 n \eta \tau^+ \Big(\frac{n}{2}-1+ \eta \Big) + n \tau^-  \Big(\frac{n}{2}- \eta \Big) <   
 n^2(1- \eta )   + 2 n \tau^- (1-\eta) \Big( \frac{n}{2} - \eta \Big)  + n \tau^+ \Big( \frac{n}{2} -1 +\eta  \Big) \nonumber \\
 & \Leftrightarrow
n (1 - 2\eta) + \tau^+ \Big( \frac{n}{2} -1 + \eta \Big) (1 - 2 \eta) + \tau^- \Big( \frac{n}{2} - \eta \Big) (1 - 2 \eta)  > 0, 
 \label{Stage2_c3}
\end{align}
which holds true since $n \geq 2$ and $ \eta < \frac{1}{2}$. 
\end{enumerate}
Therefore, we can conclude that if $ \eta < \frac{1}{2} $, and 
if $ \tau^-  <  \tau^+  \bigg( \frac{  \frac{n}{2} - 1 + \eta }{ \frac{n}{2} - \eta  } \bigg) $, then 
  $\Lambda_{\Tbar}^{(1)},    \Lambda_{\Tbar}^{(2)} <    \Lambda_{T}^{(3)}$. 
We would like to lower bound the following \textit{spectral gap},  
$\lambda_{gap} := \min \{   \Lambda_{\Tbar}^{(3)} - \Lambda_{\Tbar}^{(2)},  \Lambda_{\Tbar}^{(3)} - \Lambda_{\Tbar}^{(1)}  \}$,
a quantity we analyze next.  
\begin{enumerate}
\item \textbf{Lower bounding $\Lambda_{\Tbar}^{(3)} - \Lambda_{\Tbar}^{(1)}$.}
\begin{align*}
\Lambda_{\Tbar}^{(3)} - \Lambda_{\Tbar}^{(1)} 
& = \frac{ n + 2 \tau^-  (\frac{n}{2} - \eta) }{ n + 2 \tau^+ ( \frac{n}{2} -1+ \eta ) } - \frac{ \tau^- (\frac{n}{2} - \eta) }{ \tau^+  (\frac{n}{2} -1 + \eta) }  \\
& = \frac{  n \Big[ \tau^+ (\frac{n}{2} -1+ \eta)  - \tau^- (\frac{n}{2} - \eta) \Big] }{ \Big[ n + 2 \tau^+ (\underbrace{\frac{n}{2} -1 +\eta}_{\leq n/2}) \Big] \tau^+ (\underbrace{\frac{n}{2} -1+ \eta}_{\leq n/2})  } \\
& \geq  \frac{ n \Big[ \tau^+ (\frac{n}{2} -1+ \eta) - \tau^- (\frac{n}{2} - \eta) \Big] }{ (n + \tau^+ n) \tau^+ \frac{n}{2} }  \\
& = \frac{ 2 \Big[ \tau^+ (\frac{n}{2} -1+ \eta)  - \tau^- (\frac{n}{2} - \eta) \Big]  }{ n (1+\tau^+) \tau^+ }  \\
& \geq \frac{ 2 \Big[ \tau^+ (\frac{n}{2} -1+ \eta) (1-\varepsilon_{\tau}) \Big]  }{ n (1+\tau^+) \tau^+ }  
\quad \left(\text{using $ \tau^- \Big(\frac{n}{2} - \eta \Big) \leq  \varepsilon_{\tau}  \tau^+ \Big(\frac{n}{2} -1+ \eta \Big) $ for $ \varepsilon_{\tau} \in (0,1)$} \right) \\
& \geq \frac{ 2 \tau^+ \frac{n}{3} (1-\varepsilon_{\tau}) }{ n (1+\tau^+) \tau^+ } 
\quad \left(\text{since $\frac{n}{2} -1+ \eta \geq \frac{n}{3}$ if $n \geq 6$ }\right) \\
&= \frac{2 (1 - \varepsilon_{\tau})}{ 3(1+\tau^+) }. 
\end{align*}
%
%where the last inequality stems from the assumptions 
%$ \tau^- (\frac{n}{2} - \eta) \leq  \varepsilon_{\tau}  \tau^+ (\frac{n}{2} -1+ \eta) $ for $ \varepsilon_{\tau} \in (0,1)$, 
%and $n \geq 6$. 

\item \textbf{Lower bounding $\Lambda_{\Tbar}^{(3)} - \Lambda_{\Tbar}^{(2)}$.} 
\begin{align*}
\Lambda_{\Tbar}^{(3)} - \Lambda_{\Tbar}^{(2)}  
& = \frac{ n + 2 \tau^-  (\frac{n}{2} - \eta) }{ n + 2 \tau^+ ( \frac{n}{2} -1+ \eta ) } - \frac{ n \eta  + \tau^- (\frac{n}{2} - \eta) }{ n (1-\eta) + \tau^+  (\frac{n}{2} -1 + \eta) }  \\
& = n  \frac{ n (1-2\eta) + \tau^+ (\frac{n}{2} -1+ \eta) (1-2\eta) + \tau^- (\frac{n}{2} - \eta) (1 - 2\eta)  } {  \Big[  n + 2 \tau^+ (\frac{n}{2} -1+ \eta)  \Big]  \Big[   n (1-\eta) + \tau^+  (\frac{n}{2} -1+ \eta)   \Big]}  \\
& = n (1-2\eta)    \frac{ n + \tau^+ (\frac{n}{2} -1+ \eta) + \tau^- (\frac{n}{2} - \eta)}{  \Big[ n + 2\tau^+ \underbrace{(\frac{n}{2} -1+ \eta)}_{\leq n/2} \Big]  \Big[  \underbrace{ n (1-\eta)}_{\leq n} + \tau^+  \underbrace{(\frac{n}{2} -1+ \eta)}_{\leq n}  \Big]  }    \\
& \geq  \frac{ n (1-2\eta) ( n + \frac{\tau^+ n}{3} + \frac{\tau^- n }{3} ) }{ n^2 (1+ \tau^+)^2 }  \quad  \quad  (\text{if } n \geq 6)  \\
& = \frac{ (1 - 2 \eta) }{ 3 }  \frac{ (3 + \tau^+ + \tau^-) }{ ( 1 + \tau^+ )^2 }.
\end{align*}
\end{enumerate}
We conclude that if $\eta < 1/2$, $n \geq 6$ and 
$\tau^- (\frac{n}{2} - \eta) \leq  \varepsilon_{\tau}  \tau^+ (\frac{n}{2} -1+ \eta) $ for $ \varepsilon_{\tau} \in (0,1)$, 
then 
\begin{align*}
\lambda_{gap} 
&= \min \set{\Lambda_{\Tbar}^{(3)} - \Lambda_{\Tbar}^{(2)},  \Lambda_{\Tbar}^{(3)} - \Lambda_{\Tbar}^{(1)}} \\
&= \lambda_{n-2}(\Tbar) - \lambda_{n-1}(\Tbar) 
\geq \min \set{\frac{ (1 - 2 \eta) }{ 3 }  \frac{ (3 + \tau^+ + \tau^-) }{ ( 1 + \tau^+ )^2 }, 
\frac{2 (1 - \varepsilon_{\tau})}{ 3(1+\tau^+) }}. 
\end{align*}
This completes the proof.
\end{proof}

%-----------------------------------
% Step 3: Perturbation of $\Tbar$
%-----------------------------------
\subsection{Step 3: Perturbation of $\Tbar$}
%
%
%-------------------------------------------------
% Lemma for deterministic perturbation of $\Tbar$
%-------------------------------------------------
\begin{lemma} [Perturbation of $\Tbar$] \label{lem:sponge_Tbar_pert} 
Let us denote 
\begin{align*}
P &= \Ln + \tau^+ \Dp,    \quad  \quad  \bar{P} = \ELn + \tau^+  \EDp \\
Q &= \Lp + \tau^- \Dn,    \quad  \quad \bar{Q} = \ELp + \tau^-  \EDn.
\end{align*}
Assume that $ ||  A^{\pm}  - \expec{ [A^{ \pm }] }    ||_2 \leq  \Delta_A  $ and 
$ ||  D^{\pm}  - \expec{ [D^{ \pm }] }    ||_2 \leq  \Delta_D$ holds. 
Moreover, let the perturbation terms $\Delta_A, \Delta_D$ satisfy
\begin{equation*}
\underbrace{ \Delta_A + \Delta_D (1+\tau^+) }_{\Delta_{AD}^{+}}    \leq   \frac{\tau^+ p}{2}    \left( \frac{n}{2} -1+ \eta \right), \quad \quad 
\underbrace{ \Delta_A + \Delta_D (1+\tau^-) }_{\Delta_{AD}^{-} }   \leq   \frac{\tau^- p}{2}    \left( \frac{n}{2} - \eta \right). 
\end{equation*}
Then, $ P, Q \succ 0 $, and the following holds true.
\begin{align*}
|| \underbrace{ P^{-1/2} Q  P^{-1/2} }_{T} -  \underbrace{ \bar{P}^{-1/2}   \bar{Q}   \bar{P}^{-1/2} }_{ \Tbar } ||_2   
& \leq  \frac{ 2\sqrt{2} ( \Delta_{AD}^+)^{1/2}    }{ [\tau^+ p ( \frac{n}{2} -1+ \eta )]^{3/2}  }  \Big(  \frac{n}{2}p + \tau^- p (\frac{n}{2} -\eta )  \Big)  
+ \frac{ \Delta_{AD}^{-} }{ \tau^+ p (\frac{n}{2} -1 +\eta ) }  \\
&  +   
\frac{ 2\sqrt{2}  \Delta_{AD}^-   ( \Delta_{AD}^+)^{1/2} }{ [\tau^+ p ( \frac{n}{2} -1+ \eta )]^{3/2}  }   
+  \quad +  \frac{ 2 \Delta_{AD}^+   \Delta_{AD}^- }{ [\tau^+ p ( \frac{n}{2} -1+ \eta )]^{2} } \\
&    + \frac{ 2 \Delta_{AD}^+  }{ [\tau^+ p ( \frac{n}{2} -1+ \eta )]^{2}  } \Big[ \frac{n}{2} p + \tau^- p ( \frac{n}{2} - \eta)  \Big].
\end{align*}
\end{lemma}

%
%-----------------
% Proof of Lemma
%-----------------
\begin{proof}
To begin with, we have via triangle inequality that 
\begin{equation*}
   ||  L^+  - \ELp ||_2  \leq  \Delta_A + \Delta_D, \quad ||  L^- - \ELn ||_2  \leq  \Delta_A + \Delta_D.
\end{equation*}
This in turn implies the bounds
\begin{align}
 ||  P - \bar{P} ||_2   &\leq  \Delta_A + \Delta_D + \tau^+ \Delta_D
 =  \Delta_A + \Delta_D (1+ \tau^+)  \quad (=: \Delta_{AD}^+), \label{eq:p_pbar_bd}\\
 ||  Q - \bar{Q} ||_2   &\leq  \Delta_A + \Delta_D + \tau^- \Delta_D
 =  \Delta_A + \Delta_D (1+ \tau^-)  \quad (=: \Delta_{AD}^-). \label{eq:q_qbar_bd}
\end{align}
By Weyl's inequality \cite{Weyl1912} (see Theorem \ref{thm:Weyl}) , it follows 
for each $l=1,\dots,n$ that 
\begin{align} \label{eq:lamda_pq_bound}
\lambda_l(P)  \in [\lambda_l(\bar{P})  \pm ( \Delta_A + \Delta_D(1+\tau^+) ),  \quad 
\lambda_l(Q)  \in [\lambda_l(\bar{Q})  \pm ( \Delta_A + \Delta_D(1+\tau^-) ).
\end{align}
By inspection, the eigenvalues of $\bar{P}, \bar{Q}$ are easily derived as detailed below.
\begin{enumerate}
\item $ \lambda_1(\bar{P})  = \lambda_1(\ELn) + \tau^+ p \Big( \frac{n}{2}-1+\eta \Big)   
 =   p n (1-\eta)  +   \tau^+  p \Big( \frac{n}{2}-1+\eta \Big).$ 

\item $\lambda_n(\bar{P}) = 0 + \tau^+ p \Big( \frac{n}{2}-1+\eta \Big)  =  \tau^+  p \Big( \frac{n}{2}-1+\eta \Big)$.

\item $\lambda_l(\bar{P}) =  \frac{n}{2} p +   \tau^+  p \Big( \frac{n}{2}-1+\eta \Big), \;\;  \forall l=2, \ldots, n-1.$ 

\item $\lambda_l(\bar{Q}) = \lambda_l(\ELp) + \tau^- p \Big( \frac{n}{2} - \eta \Big)   
=   \frac{n}{2} p   +   \tau^-  p \Big( \frac{n}{2}- \eta \Big), \;\;  \forall l=1, \ldots, n-2.$   
 
\item $\lambda_{n-1}(\bar{Q}) = p n \eta  + \tau^-  p \Big( \frac{n}{2}- \eta \Big)$ and 
$\lambda_{n}(\bar{Q}) = \tau^-  p \Big( \frac{n}{2}- \eta \Big).$ 
\end{enumerate}
Now using \eqref{eq:lamda_pq_bound} we can bound the extremal eigenvalues of $P, Q$ as follows. 
\begin{enumerate}
\item 
\begin{align*}
  \lambda_n(P) 
	& \geq  \lambda_n( \bar{P} ) - ( \Delta_A + \Delta_D (1+\tau^+)) \\ 
	& =  \tau^+ p \Big(  \frac{n}{2}-1+ \eta \Big) - ( \Delta_A + \Delta_D (1+\tau^+))
	\geq  \frac{ \tau^+ p }{ 2}  \Big(  \frac{n}{2}-1+ \eta \Big)  > 0 
\end{align*}
if $ ( \Delta_A + \Delta_D (1+\tau^+)) \leq    \frac{ \tau^+ p }{ 2}  \Big(  \frac{n}{2}-1+ \eta \Big)$ and $n \geq 2$.

\item
\begin{align*}
 \lambda_1(Q)  
  \leq    \lambda_1( \bar{Q} ) + ( \Delta_A + \Delta_D (1 - \tau^-)) 
  = \frac{n}{2} p + \tau^- p  \Big( \frac{n}{2} - \eta \Big) + ( \Delta_A + \Delta_D (1 - \tau^-)).
\end{align*}

\item
$ \lambda_n(Q) \geq  \lambda_n( \bar{Q} )     - ( \Delta_A + \Delta_D (1 + \tau^-))  
 =  \tau^- p (\frac{n}{2} - \eta)  - ( \Delta_A + \Delta_D (1 - \tau^-))   
  \geq  \frac{\tau^- p}{2}  \Big( \frac{n}{2} - \eta \Big)  > 0 $
if  $( \Delta_A + \Delta_D (1 + \tau^-)) \leq \frac{\tau^- p}{2}   \Big( \frac{n}{2} - \eta \Big)$. 

\item 
$ \lambda_1(P) \leq \lambda_1( \bar{P} ) + ( \Delta_A + \Delta_D (1 + \tau^+))
  =  p n (1-\eta)   +  \tau^+ p \Big( \frac{n}{2} -1+ \eta \Big)  +( \Delta_A + \Delta_D (1 + \tau^+)).$
\end{enumerate}
Next, we would like to bound the following quantity
\begin{equation*}
|| \underbrace{ P^{-1/2} Q  P^{-1/2} }_{T} -  \underbrace{ \bar{P}^{-1/2}   \bar{Q}   \bar{P}^{-1/2} }_{ \Tbar } ||_2
\end{equation*}
where $P, Q, \bar{P}, \bar{Q} \succ  0 $. Before proceeding, let us observe that 
as a consequence of the bounds on the spectra of $P, \bar{P}$, we obtain
\begin{equation} \label{eq:bd_invnor_P}
 ||  P^{-1/2} ||_2 \leq \Big(\frac{ 2 }{  \tau^+ p ( \frac{n}{2}-1+ \eta ) } \Big)^{1/2}, \quad 
|| \bar{P}^{-1/2}  ||_2 = \Big(\frac{ 1 }{  \tau^+ p ( \frac{n}{2}-1+ \eta ) } \Big)^{1/2}.
\end{equation}
Moreover, since $P, \bar{P} \succ 0$, therefore 
\begin{equation} \label{eq:op_monotone}
		\norm{P^{1/2}  - \bar{P}^{1/2}}_2 \leq || P - \bar{P}  ||_2^{1/2} 
\end{equation}
holds as $(\cdot)^{1/2}$ is operator monotone (see \cite[Theorem X.1.1]{bhatia1996matrix}). 
With these observations in mind, we obtain the bound  
\begin{align*}
||  P^{-1/2}  - \bar{P}^{-1/2} ||_2  
& = ||  P^{-1/2}  (  P^{1/2}  - \bar{P}^{1/2} )  \bar{P}^{-1/2}  ||_2  \\
& \leq  ||P^{-1/2} ||_2 || P^{1/2}  - \bar{P}^{1/2} ||_2  || \bar{P}^{-1/2}  ||_2  
\quad \text{(submultiplicativity of $\norm{\cdot}_2$ norm)} \\
&\leq ||  P^{-1/2} ||_2 || P - \bar{P}  ||_2^{1/2} || \bar{P}^{-1/2}  ||_2  \quad \text{(due to \eqref{eq:op_monotone})}\\
& \leq  \Big(\frac{ 2 }{  \tau^+ p ( \frac{n}{2}-1+ \eta ) } \Big)^{1/2}
           \Big(\frac{ 1 }{  \tau^+ p ( \frac{n}{2}-1+ \eta ) } \Big)^{1/2}
           	(\underbrace{ \Delta_A + \Delta_D (1 + \tau^+) }_{ \Delta_{AD}^+ })^{1/2}  \quad \text{(due to \eqref{eq:bd_invnor_P},\eqref{eq:p_pbar_bd})} \\
&  =   \frac{  \sqrt{2}   (\Delta_{AD}^+)^{1/2}   }{ \tau^+ p ( \frac{n}{2}-1+ \eta )}.
\end{align*}
Therefore, denoting  $ P^{-1/2} = \bar{P}^{-1/2}  + E_P  $ and $ Q = \bar{Q} + E_Q  $, we have shown thus far
\begin{equation*}
\norm{E_P}_2 \leq   \frac{ \sqrt{2} }{\tau^+ p ( \frac{n}{2} - 1 + \eta )} ( \Delta_{AD}^+ )^{1/2},  \quad     
\norm{E_Q}_2 \leq   \Delta_{AD}^-.  
\end{equation*}
Also, from the spectra of $\bar{Q}$ computed earlier, we see that 
$|| \bar{Q} ||_2 = \left( \frac{n}{2} p  +  \tau^- p  ( \frac{n}{2} - \eta )\right).$
Using these bounds, we can now upper bound $\norm{T - \Tbar}_2$ as follows. 
\begin{align*}
\norm{T - \Tbar}_2 &=  ||  (\bar{P}^{-1/2} + E_P) (\bar{Q} + E_Q) (\bar{P}^{-1/2} + E_P) 
- \bar{P}^{-1/2} \bar{Q} \bar{P}^{-1/2})  ||_2  \\
 &  \leq  
    || \bar{P}^{-1/2} \bar{Q}  E_P ||_2  +  || \bar{P}^{-1/2} E_Q \bar{P}^{-1/2}   ||_2 +  || \bar{P}^{-1/2} E_Q  E_P ||_2  \\
 &  \quad \quad  \quad + ||   E_P \bar{Q}  \bar{P}^{-1/2}  ||_2  + || E_P \bar{Q} E_P ||_2 +  || E_P  E_Q  \bar{P}^{-1/2} ||_2
 +  || E_P E_Q E_P ||_2 \quad \text{(triangle inequality)}  \\ 
&  \leq    2 \;  || E_P ||_2 \;  || \bar{Q} ||_2 \;  ||\bar{P}^{-1/2}  ||_2   \; \;+  \;\;  || \bar{P}^{-1/2} ||_2^2  \;  || E_Q  ||_2  \;\;  \\
&  \quad \quad \quad  + \;\;  2 || \bar{P}^{-1/2} ||_2  \;   || E_P ||_2 \;  || E_Q ||_2 \;\;
+ \;  || E_P ||_2^2 \;  || E_Q ||_2  \;\;  +  \;\;   || E_P ||_2^2  \;  || \bar{Q}  ||_2  \quad \text{(submultiplicativity of $\norm{\cdot}_2$ norm)}\\
& \leq  \frac{ 2\sqrt{2} ( \Delta_{AD}^+)^{1/2}    }{ [\tau^+ p ( \frac{n}{2} -1+ \eta )]^{3/2}  }  \Big(  \frac{n}{2}p + \tau^- p (\frac{n}{2} -\eta )  \Big)  + \frac{ \Delta_{AD}^{-} }{ \tau^+ p (\frac{n}{2} -1 +\eta ) } +   
\frac{ 2\sqrt{2}  \Delta_{AD}^-   ( \Delta_{AD}^+)^{1/2} }{ [\tau^+ p ( \frac{n}{2} -1+ \eta )]^{3/2}  }   \\ 
&  \quad \quad \quad +  \frac{ 2 \Delta_{AD}^+   \Delta_{AD}^- }{ [\tau^+ p ( \frac{n}{2} -1+ \eta )]^{2} } + \frac{ 2 \Delta_{AD}^+  }{ [\tau^+ p ( \frac{n}{2} -1+ \eta )]^{2}  } \Big[ \frac{n}{2} p + \tau^- p ( \frac{n}{2} - \eta)  \Big].
\end{align*}
\end{proof}

%----------------------------------------------------
% Step 4: Concentration bounds for $A+, A-, D+, D-$
%----------------------------------------------------
\subsection{Step 4: Concentration bounds for $A^+, A^-, D^+, D^-$}
%
%
%-------------------------------------------
% Concentration results for $A+, A-, D+, D-$
%--------------------------------------------
\begin{lemma} \label{lem:sponge_conc_k_2}
%Denoting $c( \eta, p) = \sqrt{\frac{p}{2} \Big[ (1-\eta) \big[  1 - p (1-\eta) \big] + \eta (1 - p\eta ) \Big]}$, 
The following holds true. 
\begin{enumerate}[nolistsep]
\item For every $0 < \varepsilon \leq 1/2$, there is a constant $c_{\varepsilon} > 0$ such that 
\begin{equation*}
\prob \Big(    ||  A^+ - \EAp ||_2  \leq      
   \underbrace{ \big(  (1 + \varepsilon)  2 \sqrt{2} +1 \big) \sqrt{np}  \big) }_{\Delta_A}
   \Big)   \geq  1 - n  \exp{ \Big(  \frac{ - p n }{ c_{\varepsilon}    } \Big)  }.
\end{equation*}
 
\item For every $0 < \varepsilon \leq 1/2$, there is a constant $c_{\varepsilon} > 0$ such that
\begin{equation*}
\prob \Big(    ||  A^- - \EAn ||_2  \leq      
   \underbrace{ \big(  (1 + \varepsilon)  2 \sqrt{2} +1 \big) \sqrt{np}  \big) }_{\Delta_A}
   \Big)   \geq  1 - n  \exp{ \Big(  \frac{ - p n }{ c_{\varepsilon}    } \Big)  }.
\end{equation*}

\item If $ p > \frac{ 6 \log n }{\frac{n}{2} -1+ \eta} $ then 
\begin{equation*}
\quad  \quad     \prob \Big(    ||  D^+ - \EDp ||_2  \leq      
   \underbrace{  \sqrt{ 3 p n \log n} }_{\Delta_D}
   \Big)   \geq  1 - \frac{2}{n} .
\end{equation*}

\item If $ p > \frac{ 6 \log n }{\frac{n}{2} - \eta} $ then 
\begin{equation*}
\prob \Big(    ||  D^- - \EDn ||_2  \leq  \underbrace{  \sqrt{ 3 p n \log n} }_{\Delta_D} \Big) \geq  1 - \frac{2}{n} .  
\end{equation*}
\end{enumerate}
\end{lemma}

%
%
%----------------
% Proof
%----------------
\begin{proof}
%--------------------------------------------------
% Bounding $\mathbf{\norm{A^+ - \bar{A}^+}_2}$.
%--------------------------------------------------

\textbf{Bounding $\norm{\Ap - \EAp}_2$.} 
Recall that $\Ap$ is a symmetric matrix with $\Ap_{ii} = 0$ and where the random variables $(\Ap_{ij})_{i < j}$ are 
independent and defined in \eqref{eq:Aijp_same} (when $i,j$ are in same cluster) and 
\eqref{eq:Aijp_different} (when $i,j$ are in different clusters). 
Let us denote $Z_{ij}^+ = A_{ij}^+  -  \expec{ [ A_{ij}^+ ] }$ so that $Z_{ii}^+ = 0$, $Z_{ij}^+ = Z_{ji}^+$ 
and $(Z_{ij}^+)_{i < j}$ are independent centered random variables defined as follows.

\vspace{3mm}
\begin{minipage}{0.55\linewidth}  
$ \quad  \quad  \quad$ \underline{ $i,j$ lie in \textbf{same} cluster}
\begin{equation*}
 Z_{ij}^+= \left\{
\begin{array}{rl}
1- p (1-\eta) \quad ; & \text{w. p. }  p (1-\eta)  \\
-p (1 - \eta)  \quad ; & \text{w. p. }  1 - p (1-\eta)
\end{array} \right. ,
%\label{eq:Zijp_same}
\end{equation*}
\end{minipage}  %\hspace{0.2cm}  
\begin{minipage}{0.45\linewidth}  
$ \quad  \quad \quad$  \underline{$i,j$ lie in \textbf{different} clusters}
\begin{equation*}
  Z_{ij}^+= \left\{
\begin{array}{rl}
 1-  p \eta  \quad ; & \text{w. p } \quad p  \eta  \\
-p \eta       \quad ; & \text{w. p }  \quad 1 - p  \eta 
\end{array} \right. .
%\label{eq:Zijp_different}
\end{equation*}
\end{minipage}
\vspace{1mm}

For $i,j$ in the same cluster, we have
\begin{align*}
	\expec{ \big[ (Z_{ij}^+)^2 \big] } 
& = p (1-\eta) \big[  1 - p (1-\eta) \big]^2     +   	 \big[  1 - p (1-\eta) \big] p^2 (1-\eta)^2   \\ 
& = p (1-\eta) \big[  1 - p (1-\eta) \big]   \big[  1 - p (1-\eta) + p (1- \eta)  \big]   \\ 
& = p (1-\eta) \big[  1 - p (1-\eta) \big] .
\end{align*}

For $i,j$ in different clusters, we have
\begin{align*}
	\expec{ \big[ (Z_{ij}^+)^2 \big] } = p  \eta   (  1 - p \eta) ^2     + p^2 \eta^2  (1- p \eta)  = p \eta   (1 - p  \eta).
\end{align*}
Hence for each $i = 1,\dots,n$ we have that 
\begin{align*}
 \sqrt{\sum_{j=1}^{n}   \expec{ \big[ (Z_{ij}^+)^2 \big] }} 
& = \sqrt{p (1-\eta)  \big[  1 - p (1-\eta) \big]  \big( \frac{n}{2}-1 \big)   + \frac{n}{2} p \eta  (1 - p \eta)}   \\ 
& \leq \sqrt{\frac{n}{2} p \Big[ \underbrace{  (1-\eta)    \big[  1 - p (1-\eta) \big]}_{\leq 1}  +  \underbrace{ \eta  (1 - p \eta ) }_{\leq 1} \Big]} %\\
\leq  \sqrt{np}.
\end{align*}
%
%where $c( \eta, p) = \sqrt{ \frac{p}{2} \Big[   (1-\eta)    \big[  1 - p (1-\eta) \big]  +   \eta  (1 - p \eta )  \Big]  }$.
Hence, $\tilde{\sigma}^{+} := \max_i \sqrt{\sum_{j=1}^{n}   \expec{ \big[ (Z_{ij}^+)^2 \big] }} \leq \sqrt{np}$. Moreover, $\tilde{\sigma}^{+}_{*} := \max_{i,j} \norm{Z_{ij}^+}_{\infty} \leq 1$. 
Therefore we can bound $\norm{Z^+}_2 = \norm{\Ap - \EAp}_2$ using Theorem \ref{app:thm_symm_rand} (with $t = \sqrt{np}$)
which tells us that for any given $0 < \varepsilon \leq 1/2$, 
\begin{equation*}
|| A^+ - \EAp ||_2 \leq \big(  (1 + \varepsilon)  2 \sqrt{2} +1 \big) \sqrt{np} 
\end{equation*}
with probability at least  $ 1 - n\exp{ \Big( \frac{ - p n }{c_{\varepsilon}} \Big)}$. 
Here $c_{\varepsilon} > 0$ depends only on $\varepsilon$.

%-----------------------------------------
% Bounding $  ||  A^-  -  \EAn  ||_2$
%-----------------------------------------
\paragraph{Bounding $ \norm{\An  -  \EAn}_2$.}
Using the mixture model defined in  \eqref{eq:Aijn_same} and  \eqref{eq:Aijn_different} 
for the subgraph of negative edges, we can proceed in an identical fashion as above by 
replacing $\eta$  with $ 1 - \eta $. We then obtain for any given $0 < \varepsilon \leq 1/2$ that  
\begin{equation}
||  A^-  - \EAn ||_2  \leq   \big(  (1 + \varepsilon)  2 \sqrt{2} +1 \big) \sqrt{np} .
\end{equation}
with probability at least  $ 1 - n\exp{ \Big(\frac{ - p n }{c_{\varepsilon} } \Big)}$.

%--------------------------------------
% Bounding $  ||  D^+  -  \EDp  ||_2$.
%--------------------------------------
\paragraph{Bounding $  ||  \Dp  -  \EDp  ||_2$.}
Note that for any given $i$, $ D^+_{ii} = \sum_{j=1}^{n} A_{ij}^+$ with $ (A_{ij}^+)_{j=1}^n $ being independent 
Bernoulli random variables. Denoting $\mu = \expec{ [ D^+_{ii} ] } = p ( \frac{n}{2} - 1 + \eta ) $, 
we then obtain via standard Chernoff bounds (see Theorem \ref{thm:chernoff_bern}) that 
\begin{equation*}
 \prob( |  D^+_{ii} - \mu |  \geq \delta \mu  ) \leq 2 \exp \Big( - \frac{\mu \delta^2}{3}  \Big) 
 \end{equation*}
for any given $\delta \in (0,1)$. 
Letting $ \delta = \sqrt{  \frac{ 6 \log n}{ \mu}  }$ and assuming $ p > \frac{ 6 \log n }{\frac{n}{2} - 1 + \eta} $ 
(so $\delta \in (0,1)$), we have for any given $i$ that
\begin{equation*}
\prob \Big( \big| D_{ii}^+ - \mu  \big|   \geq  \sqrt{ 6 (\log n) \mu } \Big)   \leq 2 \exp{ \big(-2 \log n \big) } = \frac{2}{ n^2 }. 
\end{equation*}
Then by applying the union bound, we finally conclude that
\begin{equation*}
   ||  D^+ - \EDp  ||_2 \leq  \sqrt{ 6 p \Big( \frac{n}{2} - 1 + \eta \Big) \log n  } \leq  \sqrt{ 3 n p \log n }
\end{equation*}
with probability at least $ 1 - \frac{2}{n} $.

%---------------------------------------
%  Bounding $||  D^-  -  \EDn  ||_2$  
%---------------------------------------
\paragraph{Bounding $  ||  D^-  -  \EDn  ||_2$.}
For this quantity, we obtain the same bound as above with $\eta$ replaced with $ 1-\eta $. So we have that 
\begin{equation*}
   ||  D^-  - \EDn  ||_2 \leq  \sqrt{ 6 p \Big( \frac{n}{2} - \eta \Big) \log n  } \leq  \sqrt{ 3 n p \log n }
\end{equation*}
with probability at least $ 1 - \frac{2}{n} $ if $ p > \frac{ 6 \log n }{\frac{n}{2} - \eta} $. This completes the proof.
\end{proof}

%
%
%-------------------------------
% Step 5: Putting it together
%-------------------------------
\subsection{Step 5: Putting it together}
From Lemma \ref{lem:sponge_conc_k_2}, we can see via the union bound that all the events 
hold simultaneously with probability at least 
\begin{equation}
    1 - \frac{4}{n}  - 2n \exp{  \Big(    \frac{-p n }{ c_{\varepsilon} } \Big) }, 
\end{equation}
provided $ p > \frac{6 \log n}{\min \set{\frac{n}{2} -1+ \eta, \frac{n}{2} - \eta}}$. 
For $0 \leq \eta < 1/2$, we have $\frac{n}{2} - 1+ \eta < \frac{n}{2} - \eta$. 
Moreover, if $ n \geq 6 $, then $\frac{n}{2} -1+ \eta > \frac{n}{4}$, 
and so the condition $ p \geq  \frac{24  \log n}{n}$ clearly 
implies $p > \frac{6 \log n}{\frac{n}{2} - 1 + \eta}$.

Let us now look at the requirements in Lemma \ref{lem:sponge_Tbar_pert}. Plugging 
$$\Delta_A = \big(  (1 + \varepsilon)  2 \sqrt{2} +1 \big) \sqrt{np}, \quad 
\Delta_D = \sqrt{ 3 p n \log n}$$ 
from Lemma \ref{lem:sponge_conc_k_2}, and using the definition of $\Delta_{AD}^{+}$, we obtain 
\begin{align*}
\Delta_{AD}^{+} 
& =   \big(  (1 + \varepsilon)  2 \sqrt{2} +1 \big) \sqrt{np} + \sqrt{ 3 p n \log n }  \\
& \leq  \Big(  \big(  (1 + \varepsilon)  2 \sqrt{2} +1 \big) \sqrt{p}  + \sqrt{3p} \Big)  \sqrt{n \log n }  \\
&=  \tce  \sqrt{ n p \log n } 
\end{align*}
where $\tce = (1 + \varepsilon) 2 \sqrt{2} +1  +\sqrt{3}$. Now note that 
if $ n \geq 6 $, then 
\begin{equation}
\frac{ \tau^+ p }{ 2 } \left(\frac{n}{2} - 1 + \eta \right)  
\geq \frac{ \tau^+ p }{ 2 } \left(  \frac{n}{3}  \right)  
= \frac{ \tau^+  n p }{ 6 }.
\end{equation} 
Therefore the condition $\Delta_{AD}^{+}  \leq  \frac{ \tau^+ p }{ 2 }  \Big(\frac{n}{2} -1+ \eta \Big)$
is satisfied if 
\begin{align*}
	  \tce  \sqrt{ n p \log n } \leq   \frac{ \tau^+ n p }{6}  \Leftrightarrow 
	   \sqrt{   \frac{np}{\log n}  } \geq   \frac{6   \tilde{c}_{\varepsilon}}{ \tau^+ } 
	  \Leftrightarrow  p \geq  \bigg(\frac{36 \; \tilde{c}_{\varepsilon}^2 }{ (\tau^+)^2 } \bigg) \frac{\log n}{n}.	
\end{align*} 
In an identical fashion, one can readily verify that the condition 
\begin{equation*}
		\Delta_{AD}^{-}  \leq    \frac{ \tau^- p }{ 2 } \Big(  \frac{n}{2} -\eta \Big) 
\end{equation*} 
is satisfied if $ p  \geq  \bigg(   \frac{36 \; \tilde{c}_{\varepsilon}^2 }{ (\tau^-)^2 } \bigg) \frac{\log n}{n}. $  
Since $\frac{n}{2} - 1+ \eta \geq \frac{n}{3}$ holds if $n \geq 6$, then by using this along with the bounds 
$\Delta_{AD}^{+},  \Delta_{AD}^{-}  \leq   \tce  \sqrt{n p \log n }$ 
in Lemma \ref{lem:sponge_Tbar_pert}, we obtain 
\begin{align*}
|| \underbrace{ P^{-1/2} Q  P^{-1/2} }_{T} -  \underbrace{ \bar{P}^{-1/2}   \bar{Q}   \bar{P}^{-1/2} }_{ \Tbar } ||_2   
& \leq  \frac{ 2 \sqrt{2}  \tilde{c}_{\epsilon}^{1/2}  ( n p \log n)^{1/4}  ( \frac{n}{2} p (1+ \tau^-) )  }{  \Big( \tau^+ p \frac{n}{3} \Big)^{3/2} } +  \frac{ \tce \sqrt{ n p \log n } }{ \tau^+ p \frac{n}{3} }   \\
&  \quad \quad \quad +  \frac{ 2 \sqrt{2} \;  \tce^{3/2}  (n p \log n)^{3/4} }{  \Big( \tau^+ p \frac{n}{3} \Big)^{3/2}  }     
    +  \frac{ 2 \tce^2 \;  (n p \log n) }{  \Big( \tau^+ p \frac{n}{3} \Big)^{2} } \\
&  \quad \quad \quad  \quad \quad \quad \quad  +  \frac{ 2 \tce  \sqrt{ n p \log n }  \Big( \frac{n}{2} p (1+\tau^-) \Big)  }{   \Big( \tau^+ p \frac{n}{3} \Big)^{3/2} } \\
& =  \frac{ 3^{3/2} \; \tce^{1/2}  (1+\tau^-) }{ (\tau^+)^{3/2} } \Big( \frac{\log n}{np} \Big)^{1/4}  + 
     \frac{ 3 \; \tce }{ \tau^+} \Big( \frac{\log n}{np} \Big)^{1/2}  \\
 &       + 
\frac{ 6^{3/2} \; \tce^{3/2} }{ (\tau^+)^{3/2} } \Big( \frac{\log n}{np} \Big)^{3/4}
+   \frac{ 18 \; \tce^{2}}{ (\tau^+)^{2} } \Big( \frac{\log n}{np} \Big) + 
\frac{ 9 \tce (1+\tau^-) } { ( \tau^+)^2 }  \Big(  \frac{\log n}{ np }  \Big)^{1/2}
\end{align*} 
Since $\frac{\log n}{np} \leq 1$ by assumption, the above bound simplifies to 
\begin{align*}
|| P^{-1/2} Q  P^{-1/2}  -   \bar{P}^{-1/2}   \bar{Q}   \bar{P}^{-1/2}  ||_2 
& \leq  \Big( \frac{ 3^{3/2} \sqrt{2} \; \tce^{1/2} (1+\tau^-) } {  (\tau^+)^{3/2}   } + 
\frac{3 \tce }{ \tau^+ } +  \frac{6^{3/2} \; \tce^{3/2} }{ (\tau^+)^{3/2} }  \\
& + \frac{18 \; \tce^{2} }{ (\tau^+)^{2} } + \frac{9 \; \tce (1+\tau^-) }{ (\tau^+)^{2} } 
\Big )   \Big(  \frac{\log n}{np} \Big)^{1/4}   \\
& = \bar{c}(\varepsilon, \tau^+, \tau^-)   \Big(  \frac{\log n}{np} \Big)^{1/4}. 
\end{align*} 
where $\bar{c}(\varepsilon, \tau^+, \tau^-)$ is as defined in the statement of Theorem \ref{thm:sponge_k_2_bot_k}.  
To summarize, so far, we have seen that $T =  \bar{T} + R =   \bar{P}^{-1/2} \bar{Q} \bar{P}^{-1/2}  + R $
with $ || R ||_2  \leq  \bar{c}(\varepsilon, \tau^+, \tau^-)  \Big(  \frac{\log n}{np} \Big)^{1/4}.$

Let $ \Big(\lambda_i(T), v_i(T) \Big) $ denote the eigenpairs of $T$ for $i=1,\dots,n$. 
Using Weyl's inequality \cite{Weyl1912} (see Theorem \ref{thm:Weyl}), we obtain 
\begin{equation} \label{eq:weyls_eigbd_T}
	\lambda_i(T) \in \bigg[  \lambda_i( \bar{T}) \pm  \bar{c}(\varepsilon, \tau^+, \tau^-)  
	\Big(  \frac{\log n}{np} \Big)^{1/4} \bigg]  \quad  \forall i = 1,\dots,n.
\end{equation}
In particular, this means that 
\begin{equation*}
	\lambda_{n-2}(T)  \geq 	\lambda_{n-2}(\Tbar) - \bar{c}(\varepsilon, \tau^+, \tau^-)  
	\Big(  \frac{\log n}{np} \Big)^{1/4} > \lambda_{n-1}(\Tbar)
\end{equation*}
if the following condition holds.
\begin{equation}
\bar{c}(\varepsilon, \tau^+, \tau^-) \Big(\frac{\log n}{np} \Big)^{1/4} 
< \underbrace{\lambda_{n-2}(\Tbar) -  \lambda_{n-1}(\Tbar)  }_{ = \lambda_{gap}}.
\label{eq:tce_star_cond}
\end{equation}
\begin{figure}[h!]
\begin{center}
% \subfigure[Ground truth ranking]{
\includegraphics[width=0.5\textwidth]{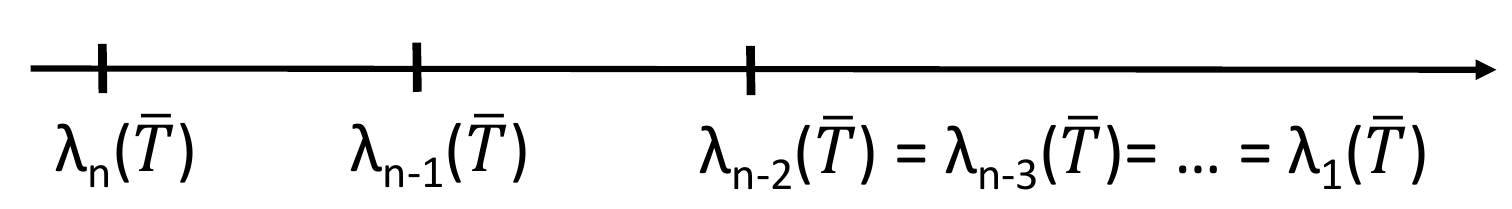}
\caption{Spectrum of $\Tbar.$}
\end{center}
\label{fig:FitLambdas}
\end{figure}

Now let $ V_2(T), V_2(\bar{T}) \in \mathbb{R}^{n \times 2} $ denote matrices whose columns 
are the eigenvectors corresponding to the  smallest two eigenvalues of $T, \Tbar$ respectively. 
We then obtain from the Davis-Kahan theorem \cite{daviskahan} (see Theorem \ref{thm:DavisKahan}) that
\begin{equation} \label{eq:sintheta_bd_sponge1}
	||  \sin \Theta \big( \calR(V_2(T)), \calR(V_2(\bar{T}))  \big)   ||_2 = \norm{( I - V_2(T) V_2(T)^T) V_2(\bar{T})}_2 
	\leq \frac{ || R ||_2 }{ \delta }  \leq  \frac{ \bar{c}(\epsilon, \tau^+, \tau^-)  \Big(  \frac{\log n}{np} \Big)^{1/4}  } { \delta }
\end{equation}
holds, provided $\delta :=  \min \Big\{ | \hat{\lambda} - \lambda |:  \lambda \in [ \lambda_{n}(\Tbar) , \lambda_{n-1}(\Tbar) ],  \hat{ \lambda } \in ( -\infty,  \underbrace{ \lambda_{n+1}(T)}_{ = -\infty} ] \cup  [ \lambda_{n-2}(T) , \infty )   \Big\} > 0$. 
If  \eqref{eq:tce_star_cond}  holds then we can see that 
\begin{align*}
\delta & =  \lambda_{n-2}(T) - \lambda_{n-1}(\Tbar)  \\
&\geq \lambda_{n-2}(\Tbar) -  \lambda_{n-1}(\Tbar)  -   \bar{c}(\epsilon, \tau^+, \tau^-)     \Big(  \frac{\log n}{np} \Big)^{1/4}  \quad \quad \text{(using \eqref{eq:weyls_eigbd_T})} \\ 
&> 0.
\end{align*}
Recall the lower bound on $\lambda_{n-2}(\Tbar) -  \lambda_{n-1}(\Tbar)$ from Lemma \ref{lem:sponge_Tbar_spectra}. 
Hence for any given $\epsilon \in (0,1)$, if the condition
\begin{align}
	\bar{c}(\varepsilon, \tau^+, \tau^-) \Big(  \frac{\log n}{np} \Big)^{1/4}   
	&\leq  
	\epsilon   \min  \Big\{   \frac{2}{3} \frac{ (1-\varepsilon_{\tau}) } { (1 + \tau^+  ) },  
	\frac{(1-2 \eta)}{3}   \frac{ (3 + \tau^+ + \tau^- ) }{ (1+\tau^+)^2 }    \Big\}  \label{eq:cbar_cond_1}\\
\Leftrightarrow 
p &\geq  	 \Bigg( \frac{ \bar{c}(\varepsilon, \tau^+, \tau^-) }{\epsilon   \min  \Big\{\frac{2}{3} \frac{ (1-\varepsilon_{\tau}) } { (1 + \tau^+  ) },  \frac{(1-2 \eta)}{3}   \frac{ (3 + \tau^+ + \tau^- ) }{ (1+\tau^+)^2 }    \Big\}   }    \Bigg)^4  \Big( \frac{\log n}{n} \Big)	\nonumber
\end{align}
holds, then $\delta$ can be lower bounded as
\begin{align} \label{eq:delt_low_bd}
\delta & \geq  (1 - \epsilon)  \min  \Big\{   \frac{2}{3} \frac{ (1-\varepsilon_{\tau}) } { (1 + \tau^+ ) },  \frac{(1-2 \eta)}{3}   \frac{ (3 + \tau^+ + \tau^- ) }{ (1+\tau^+)^2 }    \Big\}. 
\end{align}
Using \eqref{eq:delt_low_bd}, \eqref{eq:cbar_cond_1} in \eqref{eq:sintheta_bd_sponge1} we obtain the 
stated error bound on $\norm{( I - V_2(T) V_2(T)^T) V_2(\bar{T})}_2$. 
This completes the proof.
%
%
%
%--------------------------------------------------------
% Proof of theorem for SPONGE, k = 2, bottom k-1 vectors
%--------------------------------------------------------
\section{Proof of Theorem \ref{thm:sponge_k_2_bot_k1_short}}
We will prove the following more precise version of Theorem \ref{thm:sponge_k_2_bot_k1_short} in this section.
%--------------------------------------
% Main theorem k=2: bottom k-1 eigvecs
%--------------------------------------
\begin{theorem} \label{thm:sponge_k_2_bot_k1}
Assuming $\eta \in [0, 1/2)$ let $ \tau^+, \tau^- > 0$ satisfy 
$$\tau^- >  \left(\frac{\eta}{1-\eta}\right)   
\Big( \frac{ \frac{n}{2} -1 + \eta  }{ \frac{n}{2} - \eta } \Big) \tau^+.$$ 
Then it holds that $v_{n}(\Tbar) = \infovec$ with $\infovec$ defined in \eqref{eq:k_2_inform_vect}. 

Moreover, assuming $n \geq 6 $, for given $0 < \varepsilon \leq 1/2,  \epsilon \in (0,1) $ and $\varepsilon_{\tau} \in (0,1)$ let $\tau^- \geq  \frac{1}{\varepsilon_{\tau}} \left(\frac{\eta}{1-\eta}\right)   
\Big( \frac{ \frac{n}{2} -1 + \eta  }{ \frac{n}{2} - \eta } \Big) \tau^+$ and
\begin{align*}
p  \geq   
\max  \left\{ 
24,  
\frac{36 \tce^2}{ (\tau^+)^2},   
\frac{36 \tce^2}{ (\tau^-)^2},  
 \quad \left(  \frac{ \bar{c}(\varepsilon, \tau^+, \tau^-) }{ \epsilon   
\min  \Big\{ \frac{ \eta (\frac{1}{\varepsilon_{\tau}} - 1) } { 1 - \eta + \tau^+   },  \frac{(1-2 \eta)}{3}   \frac{ (3 + \tau^+ + \tau^- ) }{ (1+\tau^+)^2 }    \Big\}   }    \right)^4 
\right\}     
\Big( \frac{\log n}{n} \Big), 
\end{align*}
where $\tce$ and $\bar{c}(\varepsilon, \tau^+, \tau^-)$ are as defined in Theorem \ref{thm:sponge_k_2_bot_k}. 
Then for $c_{\varepsilon} > 0$  depending only on $\varepsilon$, 
it holds with probability at least 
$\Big( 1 - \frac{4}{n} - 2n \exp{ \big( \frac{ - p  n }{c_{\varepsilon}} \big) }  \Big)$ that

\vspace{-3mm}
\begin{equation*}
	\norm{( I - v_n(T) v_n(T)^T) w}_2  \leq  \frac{ \epsilon }{ 1 - \epsilon}. 
\end{equation*}
\end{theorem}
%
%
%--------------------
% Proof of theorem
%--------------------
\begin{proof}[Proof of Theorem \ref{thm:sponge_k_2_bot_k1}]
The proof is identical to that of Theorem \ref{thm:sponge_k_2_bot_k} barring minor changes, hence we only 
highlight the differences. 

To begin with, Lemma \ref{lem:expecs_posneg_mats} holds as it is. Lemma \ref{lem:sponge_Tbar_spectra} changes however as we now 
seek conditions under which $w$ is the smallest eigenvector. This is stated in the following Lemma.
%
%
%
%-----------------------------------------------
% Main lemma for Step 2: bottom k minus 1 case
%-----------------------------------------------
\begin{lemma}[Analogue of Lemma \ref{lem:sponge_Tbar_spectra}] \label{lem:sponge_Tbar_spectra_botk_minus1}
Let $\tau^+, \tau^- >0$ satisfy 
$\tau^-  >    \frac{\eta}{1-\eta} \bigg( \frac{  \frac{n}{2} - 1 + \eta }{ \frac{n}{2} - \eta  } \bigg) \tau^+$.  Then for $\eta \in [0, \frac{1}{2})$, the following is true.
\begin{enumerate}
\item $\lambda_{n}(\Tbar) = \frac{n\eta + \tau^- (\frac{n}{2} - \eta)}{n(1-\eta) + \tau^+ (\frac{n}{2} - 1 + \eta)}$ and
$\lambda_{l}(\Tbar) \in \set{\frac{\tau^-(\frac{n}{2} - \eta)}{\tau^+(\frac{n}{2} - 1 + \eta)}, 
\frac{ n + 2 \tau^- ( \frac{n}{2} - \eta )}{ n + 2 \tau^+ ( \frac{n}{2} - 1 +  \eta ) }},$ 
for $l=1,\dots,n-1$.
%---------------------

\item $v_n(\Tbar) =  \infovec$.
\end{enumerate}
Moreover, if $ n \geq 6$ and for a given $\varepsilon_{\tau} \in (0,1)$, 
$ \tau^-  \geq  \frac{1}{\varepsilon_{\tau}} \frac{\eta}{1-\eta} \bigg( \frac{\frac{n}{2} - 1 + \eta}{\frac{n}{2} - \eta} \bigg) \tau^+$ holds, 
then the \textit{spectral gap} ($\lambda_{gap}$) between $\lambda_{n}(\Tbar)$ and  
and $\set{\lambda_{n-1}(\Tbar), \dots, \lambda_{1}(\Tbar)}$ satisfies 
\begin{equation*}
\lambda_{gap} =  \lambda_{n-1}(\Tbar) -  \lambda_{n}(\Tbar) \geq \min 
\Big\{ \frac{ \eta (\frac{1}{\varepsilon_{\tau}} - 1) } { 1 - \eta + \tau^+   },  \frac{(1-2 \eta)}{3}   \frac{ (3 + \tau^+ + \tau^- ) }{ (1+\tau^+)^2 }    \Big\}.
\end{equation*}
\end{lemma}
The proof of the Lemma is deferred to end of this section. Carrying on, Lemma's \ref{lem:sponge_Tbar_pert}, 
\ref{lem:sponge_conc_k_2} remain unchanged so we now just need to combine the results of 
Lemma's \ref{lem:expecs_posneg_mats}, \ref{lem:sponge_Tbar_spectra_botk_minus1}, 
\ref{lem:sponge_Tbar_pert}, \ref{lem:sponge_conc_k_2}.  

Recall the bounds on the eigenvalues of $T$ as in \eqref{eq:weyls_eigbd_T}. This in particular implies 
that 
\begin{equation*}
	\lambda_{n-1}(T)  \geq 	\lambda_{n-1}(\Tbar) - \bar{c}(\varepsilon, \tau^+, \tau^-)  
	\Big(  \frac{\log n}{np} \Big)^{1/4} > \lambda_{n}(\Tbar)
\end{equation*}
if the following condition holds.
\begin{equation}
\bar{c}(\varepsilon, \tau^+, \tau^-) \Big(\frac{\log n}{np} \Big)^{1/4} 
< \underbrace{\lambda_{n-1}(\Tbar) -  \lambda_{n}(\Tbar)  }_{ = \lambda_{gap}}.
\label{eq:tce_star_cond_1}
\end{equation} 
We then obtain from the Davis-Kahan theorem \cite{daviskahan} (see Theorem \ref{thm:DavisKahan}) that
\begin{equation} \label{eq:sintheta_bd_sponge_kminus1_1}
	\norm{( I - w w^T) v_n(T)}_2 
  \leq  \frac{ \bar{c}(\epsilon, \tau^+, \tau^-)  \Big(  \frac{\log n}{np} \Big)^{1/4}  } { \delta }
\end{equation}
holds provided $\delta:= \min \Big\{ | \hat{\lambda} - \lambda |:  \lambda =  \lambda_{n}(\Tbar),  
\hat{ \lambda } \in  [ \lambda_{n-1}(T) , \infty )   \Big\} > 0$. Note that if \eqref{eq:tce_star_cond_1} holds, then 
$$\delta = \lambda_{n-1}(T) - \lambda_{n}(\Tbar) \geq \lambda_{n-1}(\Tbar) -  \lambda_{n}(\Tbar)  -   \bar{c}(\epsilon, \tau^+, \tau^-)     \Big(  \frac{\log n}{np} \Big)^{1/4}> 0.$$ 
Recall the lower bound on $\lambda_{n-1}(\Tbar) -  \lambda_{n}(\Tbar)$ from 
Lemma \ref{lem:sponge_Tbar_spectra_botk_minus1}. 
We then see for any given $\epsilon \in (0,1)$ that if the condition 
\begin{align}
\bar{c}(\varepsilon, \tau^+, \tau^-) \Big(\frac{\log n}{np} \Big)^{1/4} 
&\leq 
\epsilon \min \Big\{ \frac{ \eta (\frac{1}{\varepsilon_{\tau}} - 1) } { 1 - \eta + \tau^+ },  
\frac{(1-2 \eta)}{3}   \frac{ (3 + \tau^+ + \tau^- )}{ (1+\tau^+)^2 } \Big\} \label{eq:cbar_cond_2}\\
\Leftrightarrow p &\geq \left(  \frac{ \bar{c}(\varepsilon, \tau^+, \tau^-) }{ \epsilon   
\min  \Big\{ \frac{ \eta (\frac{1}{\varepsilon_{\tau}} - 1) } { 1 - \eta + \tau^+   },  \frac{(1-2 \eta)}{3}   \frac{ (3 + \tau^+ + \tau^- ) }{ (1+\tau^+)^2 } \Big\}} \right)^4 \frac{\log n}{n} \nonumber
\end{align}
holds then $\delta$ can be lower bounded as
\begin{align} \label{eq:delt_low_bd_kminus1}
\delta & \geq  (1 - \epsilon)  \Big\{ \frac{ \eta (\frac{1}{\varepsilon_{\tau}} - 1) } { 1 - \eta + \tau^+ },  
\frac{(1-2 \eta)}{3}   \frac{ (3 + \tau^+ + \tau^- )}{ (1+\tau^+)^2 } \Big\}. 
\end{align}
Finally, using \eqref{eq:cbar_cond_2}, \eqref{eq:delt_low_bd_kminus1} in \eqref{eq:sintheta_bd_sponge_kminus1_1}, we 
obtain the stated bound on $\norm{( I - w w^T) v_n(T)}_2$. This completes the proof.
\end{proof}

%--------------------
% Proof of Lemma
%--------------------
\begin{proof}[Proof of Lemma \ref{lem:sponge_Tbar_spectra_botk_minus1}]
In the proof of Lemma \ref{lem:sponge_Tbar_spectra}, recall that $\Lambda_{\Tbar}^{(2)}$ is the eigenvalue of 
$\Tbar$ corresponding to the eigenvector $w$. We can see from \eqref{Stage2_a1}, \eqref{Stage2_c3} that 
$\Lambda_{\Tbar}^{(2)} < \Lambda_{\Tbar}^{(1)}, \Lambda_{\Tbar}^{(3)}$ holds if 
$n \geq 2, \eta \in [0,1/2)$ and $\tau^-,\tau^+$ satisfy 
$ \tau^-  > \tau^+  \bigg( \frac{ \eta  ( \frac{n}{2} - 1 + \eta )   }{  (1-\eta) ( \frac{n}{2} - \eta )  } \bigg)$. 
 
Now, recall that if $n \geq 6$, then 
$\Lambda_{\Tbar}^{(3)} - \Lambda_{\Tbar}^{(2)} \geq \frac{ (1 - 2 \eta) }{ 3 }  \frac{ (3 + \tau^+ + \tau^-) }{ ( 1 + \tau^+ )^2 }.$ 
If $\tau^{-},\tau^+$ additionally satisfy 
\begin{equation} \label{eq:taupn_kminus1_cond}
\tau^-  \geq \frac{1}{\varepsilon_{\tau}} \tau^+  \bigg( \frac{ \eta  ( \frac{n}{2} - 1 + \eta )   }{  (1-\eta) ( \frac{n}{2} - \eta )  } \bigg), \quad \text{for } \varepsilon_{\tau} \in (0,1),
\end{equation}
then we can lower bound $\Lambda_{\Tbar}^{(1)} - \Lambda_{\Tbar}^{(2)}$ as follows.
\begin{align*}
\Lambda_{\Tbar}^{(1)} - \Lambda_{\Tbar}^{(2)}  
& =  \frac{ \tau^- (\frac{n}{2} - \eta) }{ \tau^+ (\frac{n}{2} -1+ \eta) }    
- \frac{ n\eta +\tau^- (\frac{n}{2} - \eta) }{ n (1-\eta) + \tau^+ (\frac{n}{2} -1+\eta)  }       \\
& = n \frac{ \tau^- (1-\eta) (\frac{n}{2} - \eta) - \tau^+ \eta (\frac{n}{2} -1+ \eta) } {\tau^+ (\frac{n}{2} -1+ \eta) \Big[  n (1-\eta) + \tau^+ (\frac{n}{2} -1+ \eta)   \Big] } \\
&\geq n \left[\frac{\eta \tau^+ (\frac{n}{2} -1+ \eta) (\frac{1}{\varepsilon_{\tau}} - 1)}{\tau^+ (\frac{n}{2} -1+ \eta) \Big[n(1-\eta) + \tau^+ (\frac{n}{2} -1+ \eta)\Big]} \right] \quad (\text{using } \eqref{eq:taupn_kminus1_cond})\\
&= \frac{n\eta (\frac{1}{\varepsilon_{\tau}} - 1)}{n(1-\eta) + \tau^+(\underbrace{\frac{n}{2} -1+ \eta}_{\leq n})} 
\geq \frac{\eta (\frac{1}{\varepsilon_{\tau}} - 1)}{1-\eta + \tau^+}.
\end{align*}
Hence the stated lower bound on the spectral gap follows, which completes the proof.
\end{proof}

%----------------------------------------------------------
% Proof of Theorem for k =2, signed laplacian clustering
%----------------------------------------------------------
\section{Proof of Theorem \ref{thm:k_2_signed_laplacian}} \label{sec:signed_laplacian_proof_k_2}
The proof is divided into the following steps.
\paragraph{Step $1$: Spectrum of $\mathbf{\ELbar}$.}
To begin with, we first observe for any $i,j$ that 
\begin{equation}
\expec{A_{ij}} = \left\{
\begin{array}{rl}
p (1-2\eta)  \quad ; & \text{if } \quad i,j \text{ lie in same cluster }  \\
-p (1 - 2\eta) \quad ; & \text{if } \quad i,j \text{ lie in different clusters} \\
0  \quad ; & \text{if } \quad i = j
\end{array} \right. .
\end{equation}
Due to the construction of $C_1, C_2$ as per the SSBM, this means that
\[
\EA  =  
\begin{tikzpicture}[mymatrixenv]
   \matrix[mymatrix] (m)  {
        p (1-2\eta)\ones\ones^T & -p(1 - 2\eta)\ones\ones^T  \\
        -p(1 - 2\eta)\ones\ones^T     &  p (1-2\eta)\ones\ones^T  \\
    };
    \mymatrixbraceright{1}{1}{$n/2$}
    \mymatrixbracetop{2}{2}{$n/2$}
\end{tikzpicture} - p(1 - 2\eta) I = M - p(1- 2\eta) I.
\]  
$M$ is clearly a rank $1$ matrix, indeed, $M = np(1 - 2\eta) \infovec \infovec^T$ 
where $\infovec$ is defined in \eqref{eq:k_2_inform_vect}. Therefore, we obtain 
\begin{equation}
\lambda_i(\EA) = \left\{
\begin{array}{rl}
p(n-1)(1-2\eta)  \quad ; & i = 1  \\
-p (1 - 2\eta) \quad ; & i = 2,\dots,n \\
\end{array} \right. , 
\end{equation}
and also $v_1(\EA) = v_1(M) = \infovec$. Moreover, one can easily check that $\EDbar = (n-1)p I$.
Therefore $\ELbar = \EDbar - \EA = (n-1)p I - \EA$ and hence
\begin{equation}
\lambda_i(\ELbar) = \left\{
\begin{array}{rl}
2\eta(n-1)p  \quad ; & i = n  \\
(n-1)p + p(1-2\eta) = (n-2\eta)p  \quad ; & i = 1,\dots,n-1 \\
\end{array} \right. , 
\end{equation}
with $v_n(\ELbar) = v_n(\EA) = \infovec$.

%----------------------
% Concentration bound
%----------------------
\paragraph{Step $2$: Bounding $\mathbf{\norm{\Lbar - \ELbar}_2}$.} Next, we will like to bound $\norm{\Lbar - \ELbar}_2$. 
Since $$\norm{\Lbar - \ELbar}_2 \leq \norm{\Dbar - \EDbar}_2 + \norm{A - \EA}_2,$$ we will bound the terms 
on the RHS individually starting with the first term.

Recall that $\Dbar_{ii} = \sum_{j \neq i} \abs{A_{ij}} = \sum_{ j \neq i} Z_{ij}$ where
$Z_{ij} = 1$ with probability $p$ and is $0$ with probability $1 - p$. Also, for a given $i$, 
note that $(Z_{ij})_{j=1, j\neq i}^n$ are i.i.d. Therefore from Chernoff bounds for 
sums of independent Bernoulli random variables (see Theorem \ref{thm:chernoff_bern}), 
it follows for any given $\delta \in (0,1)$ that 
\begin{align*}
\prob(\abs{\Dbar_{ii} - (n-1)p} \geq \delta(n-1)p) 
&\leq 2\exp\left(-\frac{(n-1)p\delta^2}{3}\right) \\
&\leq 2\exp\left(-\frac{np\delta^2}{6}\right) \quad \text{(if $n \geq 2$)}.
\end{align*}
If $p > \frac{12 \log n}{n}$ then we can set $\delta = \sqrt{\frac{12\log n}{np}}$ and apply the union bound. 
We then have that 
\begin{equation} \label{eq:Dbar_bound_k_2}
\norm{\Dbar - \EDbar}_2 = \max_{i} \abs{\Dbar_{ii} - \expec[\Dbar_{ii}]} \leq \sqrt{12 p n \log n} 
\end{equation}
with probability at least $1 - \frac{2}{n}$.

We now look to bound $\norm{A - \EA}_2$. Since $A$ is a random symmetric matrix with 
$(A_{ij})_{i \leq j}$ being independent, bounded random variables, we will use Theorem \ref{app:thm_symm_rand} to bound $\norm{A - \EA}_2$ 
with high probability. For given $i,j$ with $i \neq j$, if $i,j$ belong to the same cluster, then
\begin{equation}
  A_{ij} - \expec[A_{ij}] = \left\{
\begin{array}{rl}
1 - p(1-2\eta) \quad ; & \text{w. p } \quad p (1-\eta)  \\
-1 - (1-2\eta)p \quad ; & \text{w. p }  \quad p\eta \\
-p(1-2\eta) \quad ; & \text{w. p } \quad (1- p)
\end{array} \right.,
\label{eq:defA_same_cent}
\end{equation}
and if $i,j$ belong to different clusters, then 
\begin{equation}
  A_{ij} - \expec[A_{ij}]= \left\{
\begin{array}{rl}
 1 + (1 - 2\eta)p  \quad ; & \text{w. p } \quad p  \eta  \\
-1 + (1 - 2\eta)p \quad ; & \text{w. p }  \quad p (1-\eta) \\
 (1 - 2\eta)p  \quad ; & \text{w. p } \quad (1- p)
\end{array} \right. .
\label{eq:defA_diff_cent}
\end{equation}
In order to use Theorem \ref{app:thm_symm_rand}, we need to compute (upper bounds on) the quantities 
\begin{equation*}
\tilde \sigma := \max_i \sqrt{\sum_{j=1}^n \expec[(A_{ij} - \expec[A_{ij}])^2]}, \quad 
\tilde{\sigma}_{*} := \max_{i,j} \norm{  A_{ij} - \expec[A_{ij}]}_{\infty}.
\end{equation*}
Note that $\tilde{\sigma}_{*} \leq 1 + (1-2\eta)p \leq 2$. Moreover, for any $i \neq j$ (irrespective of 
whether in same cluster or not), we obtain from \eqref{eq:defA_same_cent},\eqref{eq:defA_diff_cent} that
\begin{align*}
\expec[(A_{ij} - \expec[A_{ij}])^2] 
&= (1 - p(1 - 2\eta)^2) p(1 - \eta) + (1 + (1-2\eta)p)^2 p\eta + (1-p) p^2 (1 - 2 \eta)^2 \\
&= [(1 + p^2(1 - 2\eta)^2 -2p(1-2\eta))(1-\eta) + (1 + (1-2\eta)^2p^2 2(1-2\eta)p)\eta]p \\ 
   &+ (1-p)p^2 (1-2\eta)^2 \\ 
&= [1 - p(1-2\eta)^2 (2-p)]p + (1-p)p^2(1-2\eta)^2 \\
&= (1-p(1-2\eta)^2)p \leq p.
\end{align*}
This gives us $\tilde \sigma \leq \sqrt{n-1} \sqrt{p} \leq \sqrt{np}$. Then using 
Theorem \ref{app:thm_symm_rand} with $t = \sqrt{np}$ we obtain for any $0 < \varepsilon \leq 1/2$ that 
\begin{align} \label{eq:A_concbd_k_2}
\prob(\norm{A - \EA}_2 \geq ((1+\varepsilon)2\sqrt{2} + 1) \sqrt{np}) 
\leq n\exp\left(-\frac{p n}{4 c_{\varepsilon}} \right),
\end{align}
where $c_{\varepsilon} > 0$ depends only on $\varepsilon$. Using \eqref{eq:Dbar_bound_k_2}, \eqref{eq:A_concbd_k_2} 
and applying the union bound, we have that 
\begin{align}
\norm{\Lbar - \ELbar}_2 
&\leq \underbrace{((1+\varepsilon)2\sqrt{2} + 1)}_{\geq \sqrt{12}}  \sqrt{np} + \sqrt{12 p n \log n}  \nonumber \\
&\leq 2 ((1+\varepsilon)2\sqrt{2} + 1) \sqrt{n p \log n} \label{eq:Lbar_bd_kunegis}
\end{align} 
holds with probability at least 
$1 - \frac{2}{n} - n \exp\left(- \frac{pn}{4 c_{\varepsilon}}  \right)$.
%
%

%-----------------------------------------
% Subpspace perturbation via Davis Kahan
%-----------------------------------------
\paragraph{Step $3$: Using Davis-Kahan theorem.}
Say $\norm{\Lbar - \ELbar}_2 \leq \pert$ holds. Then from Weyl's inequality \cite{Weyl1912} (see Theorem \ref{thm:Weyl}), it holds 
that $\abs{\lambda_i(\Lbar) - \lambda_i(\ELbar)} \leq \norm{\Lbar - \ELbar}_2 \leq \pert$ for $i=1,\dots,n$. 
Moreover, from  the Davis-Kahan theorem (see Theorem \ref{thm:DavisKahan}), we have that 
\begin{equation} \label{eq:tempbd_sgnlap_1}
\norm{(I - v_n(\ELbar) v_n(\ELbar)^T) v_n(\Lbar)}_2 \leq \frac{\pert}{\abs{\lambda_{n-1}(\Lbar) - \lambda_n(\ELbar)}}
\end{equation}
holds if $\abs{\lambda_{n-1}(\Lbar) - \lambda_n(\ELbar)} > 0$. Now, 
\begin{align*}
 \lambda_{n-1}(\Lbar) - \lambda_n(\ELbar) 
&\geq \lambda_{n-1}(\ELbar) - \lambda_n(\ELbar) - \pert \\
&= (n - 2\eta)p - 2\eta(n-1)p - \pert \\
&= np(1 - 2\eta) - \pert > 0
\end{align*}
if $\pert < np(1-2\eta)$ holds. Therefore, for $0 < \epsilon < 1$, if 
$\pert \leq \epsilon np(1-2\eta)$ is satisfied, then from \eqref{eq:tempbd_sgnlap_1}, we obtain the bound  
$\norm{(I - v_n(\ELbar) v_n(\ELbar)^T) v_n(\Lbar)}_2 \leq \frac{\epsilon}{1-\epsilon}$.
Finally, from \eqref{eq:Lbar_bd_kunegis}, we have that 
$$\pert < np(1-2\eta) \Leftrightarrow p \geq \frac{4((1+\varepsilon)2\sqrt{2} + 1)^2}{\epsilon^2(1-2\eta)^2} \frac{\log n}{n}.$$ 
The above bound on $p$ also implies $p > 12\log n/n$ which we required earlier for deriving \eqref{eq:Dbar_bound_k_2}. 
This completes the proof.

%%% Supplemental Material
% It has additional experiments. Please leave the ordering of the appendices as is [add additional figures here].
\clearpage

\section{Additional experiments}
In this section, we show the results of additional supporting experiments: plots demonstrating how the parameters $\taup$ and $\taun$ affect the performance of SPONGE 
% and SPONGE$_{sym}$  %% - since we left out the OLD Figures 15 and 16
in a range of regimes, and plots comparing the performance of all algorithms on the SSBM under a wide range of parameters $k,p$ and $\eta$.

\subsection{Parameter analysis for $\tau^+$ and $\tau^-$}  
Here we plot performance of SPONGE 
% and SPONGE$_{sym}$ while
when varying the parameters $\taup$ and $\taun$, to motivate our choices thereof.

In Figure \ref{tab:hm1}, we show the performance of SPONGE on SSBM graphs, when considering  $k-1$ eigenvectors, plotted for a range of values for $\taup$ and $\taun$. In most cases, we observe that performance is not too sensitive to the choice of these parameters, which could be interpreted as a strength of our approach. When there exist regions of both good and poor performance, we see that the region of good performance is concentrated around the axes, i.e. when either $\taup$ and $\taun$ are low. We note that the point $\taup=\taun=1$ always falls within the region of maximum recovery when it is present (with the exception of the top left plot).

\newcommand{\wid}{1.4in}
\newcolumntype{C}{>{\centering\arraybackslash}m{\wid}}
\begin{table*}\sffamily
\begin{tabular}{l*4{C}@{}}
  & $p=0.001$ & $p=0.023$ & $p=0.045$ & $p=0.1$ \\ 
$k=2$ & \includegraphics[width=\wid]{{{Figures/New_appendix_figures/SPONGE_km1_n5000_p0.001_k2_eta0.05_grid50_log}}} & \includegraphics[width=\wid]{{{Figures/New_appendix_figures/SPONGE_km1_n5000_p0.023_k2_eta0.05_grid50_log}}} & \includegraphics[width=\wid]{{{Figures/New_appendix_figures/SPONGE_km1_n5000_p0.045_k2_eta0.05_grid50_log}}} & \includegraphics[width=\wid]{{{Figures/New_appendix_figures/SPONGE_km1_n5000_p0.1_k2_eta0.05_grid50_log}}} \\ 
$k=3$ & \includegraphics[width=\wid]{{{Figures/New_appendix_figures/SPONGE_km1_n5000_p0.001_k3_eta0.05_grid50_log}}} & \includegraphics[width=\wid]{{{Figures/New_appendix_figures/SPONGE_km1_n5000_p0.023_k3_eta0.05_grid50_log}}} & \includegraphics[width=\wid]{{{Figures/New_appendix_figures/SPONGE_km1_n5000_p0.045_k3_eta0.05_grid50_log}}} & \includegraphics[width=\wid]{{{Figures/New_appendix_figures/SPONGE_km1_n5000_p0.1_k3_eta0.05_grid50_log}}} \\ 
$k=5$ & \includegraphics[width=\wid]{{{Figures/New_appendix_figures/SPONGE_km1_n5000_p0.001_k5_eta0.05_grid50_log}}} & \includegraphics[width=\wid]{{{Figures/New_appendix_figures/SPONGE_km1_n5000_p0.023_k5_eta0.05_grid50_log}}} & \includegraphics[width=\wid]{{{Figures/New_appendix_figures/SPONGE_km1_n5000_p0.045_k5_eta0.05_grid50_log}}} & \includegraphics[width=\wid]{{{Figures/New_appendix_figures/SPONGE_km1_n5000_p0.1_k5_eta0.05_grid50_log}}} \\ 
$k=13$ & \includegraphics[width=\wid]{{{Figures/New_appendix_figures/SPONGE_km1_n5000_p0.001_k13_eta0.05_grid50_log}}} & \includegraphics[width=\wid]{{{Figures/New_appendix_figures/SPONGE_km1_n5000_p0.023_k13_eta0.05_grid50_log}}} & \includegraphics[width=\wid]{{{Figures/New_appendix_figures/SPONGE_km1_n5000_p0.045_k13_eta0.05_grid50_log}}} & \includegraphics[width=\wid]{{{Figures/New_appendix_figures/SPONGE_km1_n5000_p0.1_k13_eta0.05_grid50_log}}} \\ 
\end{tabular}
%\captionsetup{width=0.8\linewidth}
\captionof{figure}{SSBM recovery of $\text{SPONGE}$ with $k-1$ eigenvectors as a function of $\tau^+$ and $\tau^-$ with $n = 5000$, $\eta = 0.05$ and varying values of $k$ and $p$.}
\label{tab:hm1}
\end{table*} 

Figure \ref{tab:hm2} shows the same experiments, but using the bottom $k$ eigenvectors for the SPONGE algorithm, instead of $k-1$. The range of values for $\taup$ and $\taun$ was kept the same as in Figure \ref{tab:hm1}. 
The plots look similar to those of Figure \ref{tab:hm1} in most cases, but with two main differences: firstly, as we see most clearly from the rightmost three plots in the top row, there are some marginal regions (close to the boundary of the $y$-axis) where recovery with $k-1$ eigenvectors was poor, and with $k$ eigenvectors is greatly improved. This is because previously, an informative eigenvector was being displaced by a non-informative one, and so by taking the extra eigenvector we capture this useful information. However, there are also some regions where recovery drops from near-perfect to mediocre. These are regions where we were already capturing all informative eigenvectors, so by taking another one we \emph{dilute} the quality of our embedding with what is effectively an extra dimension of noise. In general, if we pick a natural parameter choice such as $\taup=\taun=1$, our recovery score is not improved, and in some cases is worsened, if we use $k$ eigenvectors instead of $k-1$. For example, for $k=2$, and $p=\{ 0.023, 0.045, 0.1 \}$, using $k-1$ eigenvectors gives better results than using $k$ eigenvectors for a wide range of values as we move further away from the origin, except for the region where $\taun$ is very small (note the black stripe in the top three left plots in Figure \ref{tab:hm1}), in which case the opposite statement holds true.

\begin{table*}\sffamily
\begin{tabular}{l*4{C}@{}}
  & $p=0.001$ & $p=0.023$ & $p=0.045$ & $p=0.1$ \\ 
$k=2$ & \includegraphics[width=\wid]{{{Figures/New_appendix_figures/SPONGE_k_n5000_p0.001_k2_eta0.05_grid50_log}}} & \includegraphics[width=\wid]{{{Figures/New_appendix_figures/SPONGE_k_n5000_p0.023_k2_eta0.05_grid50_log}}} & \includegraphics[width=\wid]{{{Figures/New_appendix_figures/SPONGE_k_n5000_p0.045_k2_eta0.05_grid50_log}}} & \includegraphics[width=\wid]{{{Figures/New_appendix_figures/SPONGE_k_n5000_p0.1_k2_eta0.05_grid50_log}}} \\ 
$k=3$ & \includegraphics[width=\wid]{{{Figures/New_appendix_figures/SPONGE_k_n5000_p0.001_k3_eta0.05_grid50_log}}} & \includegraphics[width=\wid]{{{Figures/New_appendix_figures/SPONGE_k_n5000_p0.023_k3_eta0.05_grid50_log}}} & \includegraphics[width=\wid]{{{Figures/New_appendix_figures/SPONGE_k_n5000_p0.045_k3_eta0.05_grid50_log}}} & \includegraphics[width=\wid]{{{Figures/New_appendix_figures/SPONGE_k_n5000_p0.1_k3_eta0.05_grid50_log}}} \\ 
$k=5$ & \includegraphics[width=\wid]{{{Figures/New_appendix_figures/SPONGE_k_n5000_p0.001_k5_eta0.05_grid50_log}}} & \includegraphics[width=\wid]{{{Figures/New_appendix_figures/SPONGE_k_n5000_p0.023_k5_eta0.05_grid50_log}}} & \includegraphics[width=\wid]{{{Figures/New_appendix_figures/SPONGE_k_n5000_p0.045_k5_eta0.05_grid50_log}}} & \includegraphics[width=\wid]{{{Figures/New_appendix_figures/SPONGE_k_n5000_p0.1_k5_eta0.05_grid50_log}}} \\ 
$k=13$ & \includegraphics[width=\wid]{{{Figures/New_appendix_figures/SPONGE_k_n5000_p0.001_k13_eta0.05_grid50_log}}} & \includegraphics[width=\wid]{{{Figures/New_appendix_figures/SPONGE_k_n5000_p0.023_k13_eta0.05_grid50_log}}} & \includegraphics[width=\wid]{{{Figures/New_appendix_figures/SPONGE_k_n5000_p0.045_k13_eta0.05_grid50_log}}} & \includegraphics[width=\wid]{{{Figures/New_appendix_figures/SPONGE_k_n5000_p0.1_k13_eta0.05_grid50_log}}} \\ 
\end{tabular}
\captionof{figure}{SSBM recovery of the $\text{SPONGE}$ algorithm with $k$ eigenvectors as a function of $\tau^+$ and $\tau^-$ with $n = 5000$, $\eta = 0.05$ and varying values of $k$ and $p$.}
\label{tab:hm2}
\end{table*}

\vspace{-1mm}
\subsection{Numerical experiments on the SSBM}
\vspace{-1mm}
% In the following figures we plot performance of \textsc{SPONGE} and \textsc{SPONGE}$_{sym}$ against five benchmark algorithms, on graphs generated from the SSBM, with equal size planted clusters, and with performance measured by the Adjusted Rand Index (ARI) against the ground truth.

In Figure \ref{tab:scores1} we plot the performance of \textsc{SPONGE} and \textsc{SPONGE}$_{sym}$ against five benchmark algorithms, on graphs generated from the SSBM, with equal size planted clusters, and with performance measured by the Adjusted Rand Index (ARI) against the ground truth. 
% In Figure \ref{tab:scores1}, 
We fix the parameter values $p=\{0.001, 0.01, 0.1\}$ and $k=\{2,5,10,20,50\}$, and plot the ARI score against the flip probability $\eta$. We see that for $k=2$, the symmetric Signed Laplacian $\bar L_{sym}$ of Kunegis et al. \cite{kunegis2010spectral} can tolerate the highest $\eta$. However, as $k$ increases, the previous state-of-the-art algorithms are quickly overtaken, first by  \textsc{SPONGE}, and then \textsc{SPONGE}$_{sym}$. When $k$ is large (rows $k=20$ and $k=50$), \textsc{SPONGE}$_{sym}$ outperforms all other algorithms by a large margin.

\renewcommand{\wid}{1.8in}
\begin{table*}\sffamily
\begin{tabular}{l*3{C}@{}}
  & $p=0.001$ & $p=0.01$ & $p=0.1$ \\ 
$k=2$ &\includegraphics[width=\wid]{{{Figures/SSBM_graphs_all/eta_recovery_n10000_p0.001_k2_Suniform}}} & \includegraphics[width=\wid]{{{Figures/SSBM_graphs_all/eta_recovery_n10000_p0.01_k2_Suniform}}} &  \includegraphics[width=\wid]{{{Figures/SSBM_graphs_all/eta_recovery_n10000_p0.1_k2_Suniform}}} \\ 
$k=5$ & \includegraphics[width=\wid]{{{Figures/SSBM_graphs_all/eta_recovery_n10000_p0.001_k5_Suniform}}} & \includegraphics[width=\wid]{{{Figures/SSBM_graphs_all/eta_recovery_n10000_p0.01_k5_Suniform}}} &  \includegraphics[width=\wid]{{{Figures/SSBM_graphs_all/eta_recovery_n10000_p0.1_k5_Suniform}}} \\ 
$k=10$ & \includegraphics[width=\wid]{{{Figures/SSBM_graphs_all/eta_recovery_n10000_p0.001_k10_Suniform}}} & \includegraphics[width=\wid]{{{Figures/SSBM_graphs_all/eta_recovery_n10000_p0.01_k10_Suniform}}} &  \includegraphics[width=\wid]{{{Figures/SSBM_graphs_all/eta_recovery_n10000_p0.1_k10_Suniform}}} \\ 
$k=20$ & \includegraphics[width=\wid]{{{Figures/SSBM_graphs_all/eta_recovery_n10000_p0.001_k20_Suniform}}} & \includegraphics[width=\wid]{{{Figures/SSBM_graphs_all/eta_recovery_n10000_p0.01_k20_Suniform}}} &  \includegraphics[width=\wid]{{{Figures/SSBM_graphs_all/eta_recovery_n10000_p0.1_k20_Suniform}}} \\ 
$k=50$ & \includegraphics[width=\wid]{{{Figures/SSBM_graphs_all/eta_recovery_n10000_p0.001_k50_Suniform}}} & \includegraphics[width=\wid]{{{Figures/SSBM_graphs_all/eta_recovery_n10000_p0.01_k50_Suniform}}} &  \includegraphics[width=\wid]{{{Figures/SSBM_graphs_all/eta_recovery_n10000_p0.1_k50_Suniform}}} \\ 
\end{tabular}

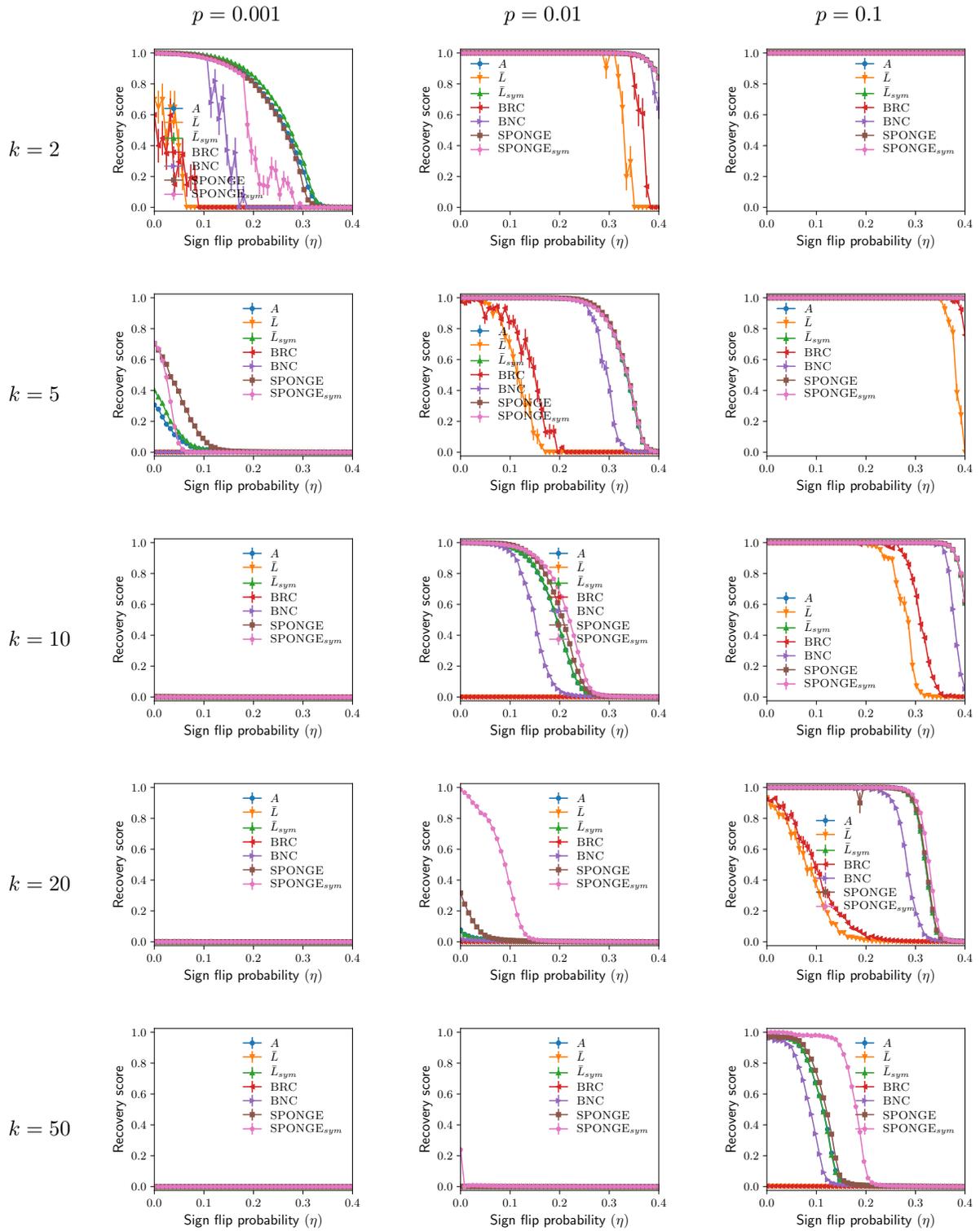
\captionof{figure}{ Adjusted Rand Index achieved by several algorithms as a function of the noise level $\eta$ for different values of $p$ and $k$,  $n=10000$, and clusters of fixed equal size.}
\label{tab:scores1}
\end{table*}

\clearpage

% \clearpage
\subsection{Numerical experiments on real data sets}

This section details the results of our numerical experiments on two additional data sets.

\vspace{-2mm}
% \paragraph{Correlations of currency values.}

\paragraph{Correlations of financial market returns - S\&P 500.}
% We begin with 
We detail here the results of our experiment on financial equity time series data, corresponding to constituents of S\&P 500. The procedure for obtaining the signed network is the same as the one for S\&P 1500, discussed in the main text.

We consider time series price data for $n=500$ stocks in the S\&P % Composite 
500 Index, during 2003-2015, containing approximately $n_d=3000$ trading days. We work with daily log returns of the prices 
     \begin{equation}
        R_{i,t} = \log \frac{ P_{i,t} }{ P_{i,t-1}},
     \end{equation} 
where $P_{i,t}$ denotes the market close price of instrument $i$ on day $t$. Next,  we compute the daily market excess return for each instrument, 
 \begin{equation} 
    \tilde{R}_{i,t} = R_{i,t} -  R_{ \text{SPY},t},  \forall i=1, \ldots, n, \; t=1,\ldots,n_d
\end{equation} 
where $R_{\text{SPY},t}$ denotes the daily log return of  \textsc{SPY}, the S\&P 500 index ETF used as a proxy for the market. 
% Finally, we cluster the resulting Pearson correlation matrix  of the $n$ stocks.
We then calculate the Pearson correlation coefficient between historical return for each pair of companies, and use that as an edge weight in our signed graph.
Figure \ref{fig:SP500_k_10_20} shows that, for $k=\{ 10, 20 \}$, we are able to recover the clustering structure covering the entire network. 

\iffalse 
We also interpret our results in light of the GICS sector division, %  (at the sector level) 
popular among practitioners, which is a segmentation of the US market into 10 different industries of the US economy  % , based entirely on fundamentals data 
\cite{GICS_citation}.  
%% 
In Figure \ref{fig:spsec}, the instruments are sorted by cluster membership, with vertical black lines separating the clusters. Some of the clusters detected consist entirely of companies from a single sector, such as Financials or Utilities. 
% MAYBE_LOUT  This information was not provided to the clustering algorithm, and has been inferred from the correlations alone.
\fi

% \vspace{-1mm}
\begin{figure}[!htp]
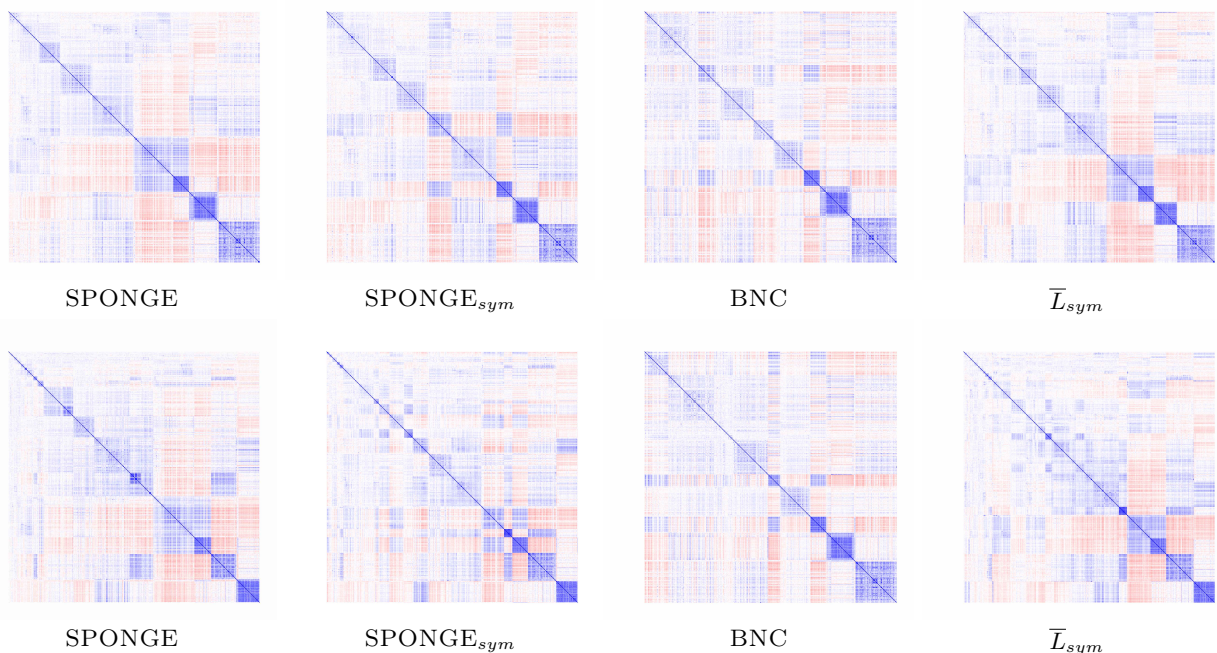

\captionsetup[subfigure]{labelformat=empty}
\begin{centering}
\subcaptionbox{\textsc{SPONGE}%, $k=10$
}{\includegraphics[width=0.24\columnwidth]{{{Figures/real_data_graphs_low_k/SP500adjSPONGE_k10}}}}	
\subcaptionbox{ \small  \textsc{SPONGE}$_{sym}$ % $k=10$
}{\includegraphics[width=0.24\columnwidth]{{{Figures/real_data_graphs_low_k/SP500adjSPONGEsym_k10}}}}	
\subcaptionbox{\textsc{BNC}  % , $k=10$
}{\includegraphics[width=0.24\columnwidth]{{{Figures/real_data_graphs_low_k/SP500adjBNC_k10}}}}	
\subcaptionbox{$\bar{L}_{sym}$ % , $k=10$
}{\includegraphics[width=0.24\columnwidth]{{{Figures/real_data_graphs_low_k/SP500adjLap_k10}}}}	
% \par
\end{centering}
\begin{centering}
\subcaptionbox{\textsc{SPONGE}}{\includegraphics[width=0.24\columnwidth]{{{Figures/real_data_graphs_high_k/SP500adjSPONGE_k20}}}}	
\subcaptionbox{\textsc{SPONGE}$_{sym}$  %  $k=20$
}{\includegraphics[width=0.24\columnwidth]{{{Figures/real_data_graphs_high_k/SP500adjSPONGEsym_k20}}}}	
\subcaptionbox{\textsc{BNC} % , $k=20$
}{\includegraphics[width=0.24\columnwidth]{{{Figures/real_data_graphs_high_k/SP500adjBNC_k20}}}}	
\subcaptionbox{$\bar{L}_{sym}$  % , $k=20$
}{\includegraphics[width=0.24\columnwidth]{{{Figures/real_data_graphs_high_k/SP500adjLap_k20}}}}	
\par\end{centering}
\captionsetup{width=0.99\linewidth}
\caption{ Adjacency matrix of the S\&P 500 data set sorted by cluster membership, for $k=\{10,20\}$.
% (a,e) $SPONGE$ (b,f) $SPONGE_{sym}$ (c,g) $BNC$ (d,h)  $\bar{L}_{sym}$. The first row has $k=10$ while the second has $k=30$.
}
\label{fig:SP500_k_10_20}
\end{figure}

\paragraph{Foreign Exchange correlations.}

\iffalse   
    SHORT AND OPTIMIZED PARAGRAPH
The \textsc{SPONGE} algorithms   % proved to work % exceptionally 
also worked
particularly well in the setting of clustering a % Foreign Exchange 
FX currency correlation matrix, derived from % the % correlation matrix of their 
daily  % Special Drawing Rights (SDR) 
exchange % values 
rates \cite{FX_IMF_SDR}. 
Indeed, as shown in Figure \ref{fig:forexClusters}, only \textsc{SPONGE}  and \textsc{SPONGE}$_{sym}$ were able to recover four neat clusters associated respectively to the EURO  \euro, US Dollar \$, UK  Pound Sterling  \textsterling, and Japanese Yen  \yen. 
%%% These are precisely the four currencies that, in certain percentage weights, define the value of the SDR reserve. 
Figure \ref{fig:currencymap} shows that the recovered clusters align well with their geographic locations.
 \fi

The \textsc{SPONGE} algorithms   proved to work  % exceptionally 
particularly well in the setting of clustering a Foreign Exchange matrix
% FX currency correlation matrix, 
derived from % the % correlation matrix of their 
daily  Special Drawing Rights (SDR)  exchange value rates \cite{FX_IMF_SDR}. 
Indeed, as shown in Figure \ref{fig:forexClusters}, only \textsc{SPONGE}  and \textsc{SPONGE}$_{sym}$ were able to recover four neat clusters associated respectively to the EURO  \euro, US Dollar \$, UK  Pound Sterling  \textsterling, and Japanese Yen  \yen. 
%%% 
These are precisely the four currencies that, in certain percentage weights, define the value of the SDR reserve. Figure \ref{fig:currencymap} shows that the recovered clusters align well with their geographic locations.

% \vspace{-2mm}
\begin{figure}
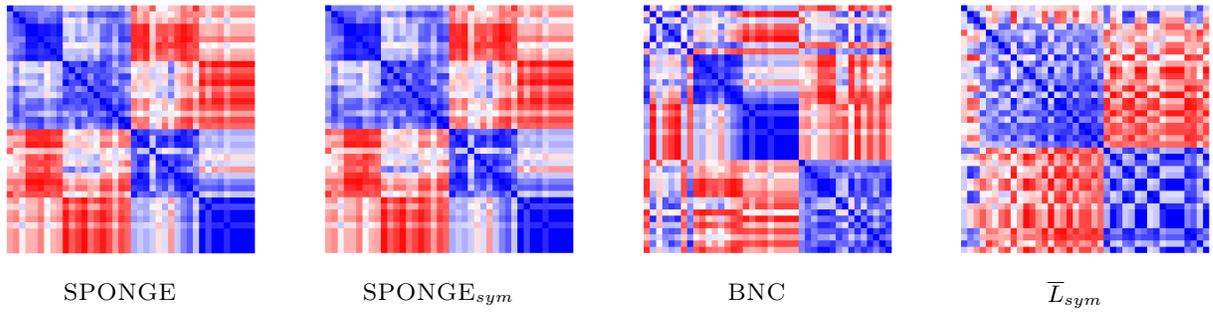

\captionsetup[subfigure]{labelformat=empty}
	\begin{centering}
		\subcaptionbox{\textsc{SPONGE} }{\includegraphics[width=0.24\columnwidth]{{{Figures/real_data_graphs_low_k/FOREXadjGEadd_k4}}}}	
		\subcaptionbox{\textsc{SPONGE}$_{sym}$  }{\includegraphics[width=0.24\columnwidth]{{{Figures/real_data_graphs_low_k/FOREXadjGEmul_k4}}}}	
		\subcaptionbox{\textsc{BNC} }{\includegraphics[width=0.24\columnwidth]{{{Figures/real_data_graphs_low_k/FOREXadjBNC_k4}}}}	
		\subcaptionbox{$\bar{L}_{sym}$ }{\includegraphics[width=0.24\columnwidth]{{{Figures/real_data_graphs_low_k/FOREXadjLap_k4}}}}	
		\par
		\end{centering}
	\captionsetup{width=0.99\linewidth}
	\caption{ % \footnotesize 
	Adjacency matrix of the Forex data set, sorted by cluster membership, for $k = 4$ clusters. We remark that \textsc{BNC} and especially $\bar{L}_{sym}$ return less meaningful results.
	% (a) $SPONGE$ (b) $SPONGE_{sym}$ (c) $BNC$ (d)  $\bar{L}_{sym}$.
}
\label{fig:forexClusters}
\end{figure}
% \vspace{-2mm}

% \vspace{-1mm}
\begin{figure}[h]
	\vspace{.3in}
	\centerline{\includegraphics[width=.75 \linewidth]{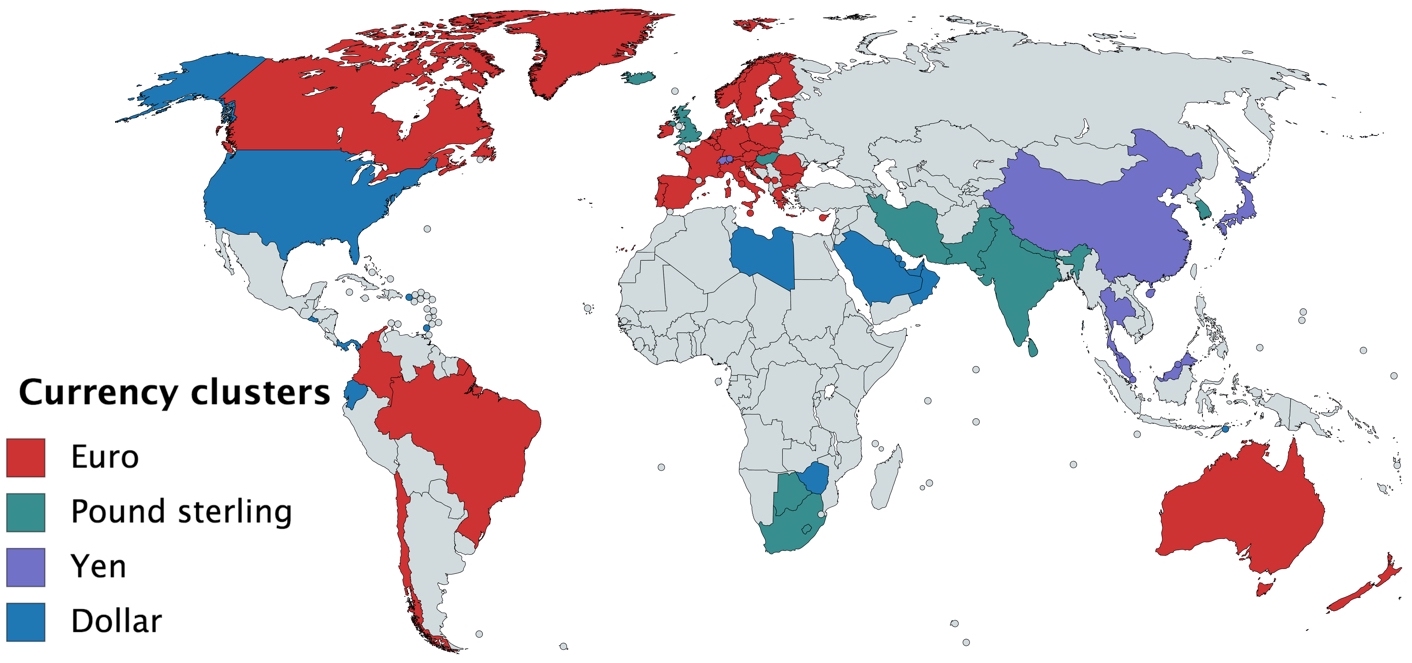}}
	% \vspace{.3in}
	\captionsetup{width=1.05\linewidth}
	\caption{\textsc{SPONGE}$_{sym}$  % clustering 
	clustering of the $50$ currencies with $k = 4$ clusters.}
	\label{fig:currencymap}
\end{figure}
% \vspace{-1mm}

\end{document}